\documentclass[aoas]{imsart}

\RequirePackage{amsthm,amsmath,amsfonts,amssymb}
\RequirePackage[authoryear]{natbib}
\RequirePackage[colorlinks,citecolor=blue,urlcolor=blue]{hyperref}
\RequirePackage{graphicx}
\RequirePackage[authoryear]{natbib}
\RequirePackage{algorithm}
\RequirePackage[noend]{algorithmic}
\RequirePackage{url}
\RequirePackage[utf8]{inputenc}
\RequirePackage{xcolor}
\RequirePackage{wrapfig}
\RequirePackage{alphalph}
\RequirePackage{caption}
\RequirePackage{subcaption}
\RequirePackage{physics}
\RequirePackage{booktabs}

\renewcommand{\hat}{\widehat}
\DeclareMathOperator{\argmax}{argmax}
\DeclareMathOperator{\argmin}{argmin}

\newcommand{\eqdeltastar}{(31)~}
\newcommand{\appendixalgproof}{A~}
\newcommand{\appendixgaussianconnectionproof}{B.1~}
\newcommand{\appendixgaussianconcentrationproof}{B.2~}
\newcommand{\appendixspreadtimeproof}{B.3~}

\newcommand{\algtheorem}{3.1~}

\newcommand{\theorysection}{4~}

\newcommand{\alginitialization}{1~}
\newcommand{\eqmk}{(10)~}
\newcommand{\eqck}{(11)~}
\newcommand{\sstheorygaussian}{4.1~}
\newcommand{\sstheorysbm}{4.2~}
\newcommand{\thmgaussianconnection}{4.1~}
\newcommand{\lemgaussianconcentration}{4.2~}
\newcommand{\thmtimetilnext}{4.3~}
\newcommand{\eqsepconditionone}{(13)~}
\newcommand{\eqsepconditiontwo}{(14)~}
\newcommand{\eqsepconditionthree}{(15)~}

\startlocaldefs
\theoremstyle{plain}

\newtheorem{theorem}{Theorem}[section]
\newtheorem{lemma}[theorem]{Lemma}
\newtheorem{proposition}[theorem]{Proposition}
\newtheorem{corollary}[theorem]{Corollary}

\theoremstyle{remark}

\endlocaldefs

\newcommand{\E}{\mathbf{E}}
\renewcommand{\P}{\mathbf{P}}

\begin{document}
\begin{frontmatter}
\title{Unifying Epidemic Models with Mixtures}
\runtitle{Unifying Epidemic Models with Mixtures}

\begin{aug}
\author[A]{\fnms{Arnab} \snm{Sarker}\ead[label=e1, mark]{arnabs@mit.edu}},
\author[A]{\fnms{Ali} \snm{Jadbabaie}\ead[label=e2, mark]{jadbabai@mit.edu}}
\and
\author[A]{\fnms{Devavrat} \snm{Shah}\ead[label=e3, mark]{devavrat@mit.edu}}
\address[A]{Institute for Data, Systems, and Society, MIT \printead{e1}, \printead{e2}, \printead{e3}}
\end{aug}

\begin{abstract}
    The COVID-19 pandemic has emphasized the need for a robust understanding of epidemic models. 
    Current models of epidemics are classified as either mechanistic or non-mechanistic:
    mechanistic models make explicit assumptions on the dynamics of disease, whereas non-mechanistic models make assumptions on the form of observed time series. 
    Here, we introduce a simple mixture-based model which bridges the two approaches while retaining benefits of both.
    The model represents time series of cases and fatalities as a mixture of Gaussian curves, providing a flexible function class to learn from data compared to traditional mechanistic models.
    Although the model is non-mechanistic, we show that it arises as the natural outcome of a stochastic process based on a networked SIR framework.
    This allows learned parameters to take on a more meaningful interpretation compared to similar non-mechanistic models, and we validate the interpretations using auxiliary mobility data collected during the COVID-19 pandemic.
    We provide a simple learning algorithm to identify model parameters and establish theoretical results which show the model can be efficiently learned from data. 
    Empirically, we find the model to have low prediction error.\footnote{The model is available live at \url{covidpredictions.mit.edu}}
    Ultimately, this allows us to systematically understand the impacts of interventions on COVID-19, which is critical in developing data-driven solutions to controlling epidemics.
\end{abstract}


\end{frontmatter}

\section{Introduction}
\label{s:intro}

The COVID-19 pandemic has reinforced the need for a deep understanding of epidemic processes. 
The initial uncertainty which arose in the beginning of the pandemic led to new questions about infectious disease modeling and estimation, as well as an emphasis on robust control of the pandemic with attention to careful trade-offs between health outcomes and economic costs of interventions \citep{acemoglu2020multi}.

Epidemic processes are often understood through models, which attempt to simplify the complex process by which contagion travels between individuals in a population in order to provide insights and actionable policies \citep{hethcote2000mathematics}. 
Broadly, models in the literature are typically categorized into one of two types: Mechanistic models, and non-mechanistic models, which are also referred to as reduced form models \citep{holmdahl2020wrong}.
Both types of models have benefits and drawbacks to their use, and in this work, we aim to unify the approaches in a single model which bridges the two classes of models while retaining benefits of both.

Mechanistic models are well established and have been used by epidemiologists since the early 20th century to understand the dynamics of the spread of infectious disease \citep{kermack1927contribution}.
Broadly, this class of models make assumptions on the underlying process by which disease spreads, through the use of differential equations or more complex agent-based approaches \citep{brauer2019mathematical}. 
Such models have a vast literature and are well understood, which has made them an obvious choice for tasks such as epidemic forecasting throughout the COVID-19 pandemic \citep{li2020forecasting}.

However, mechanistic models do have some drawbacks when applied to observed data, which resulted in the popularity of reduced form models during the beginning stages of the COVID-19 pandemic \citep{jewell2020caution}. 
Mechanistic models, due to the indirect relationship between model parameters and forecasts, can have wide confidence intervals in forecasting and estimation of the pandemic \citep{ hespanha2021forecasting}. 
Classical models also often make the assumption that individuals do not explicitly react to the state of the pandemic, and depletion of the susceptible population results in the eventual end of the epidemic \citep{kermack1927contribution, hethcote2000mathematics}.
Hence, a proper adjustment to such models requires the introduction of time-varying parameters to accurately estimate the pandemic state, increasing the burden on parameter uncertainty \citep{chen2021numerical}.

Due to these limitations, which were particularly exacerbated during the initial stages of uncertainty at the beginning the COVID-19 pandemic, reduced form models were popularized which use statistical methods to hypothesize a form for the observed time series \citep{murray2020forecasting}. 
Namely, as opposed to assuming a prolonged period of exponential growth as is often the case for mechanistic models, the reduced form models often make an implicit assumption that human behavior generates a sub-exponential trend in observed time series.
This approach has been applied since the mid 19th century, and often results in better short term forecasts with tighter confidence intervals \citep{farr1840progress, santillana2018relatedness}. 
Such models rely on statistical assumptions about the data, but are often inflexible in their predictions in the sense that they can not easily account for changes in human behavior \citep{holmdahl2020wrong}.
Moreover, a major critique of such methods is that they are often not interpretable, as the learned parameters do not have direct interpretations related to the spread of disease \citep{jewell2020caution}.

In this work, we focus on a specific reduced form model which bridges the two approaches to epidemic modeling while retaining benefits of both. Namely, we assume that the observed time series of cases has the form
\begin{equation}
\label{eq:mixtures}
    N(t) = \left(\sum_{k = 1}^r e^{-a_k t^2 + b_k t + c_k} \right) (1 + \varepsilon_t)\,.
\end{equation}
Here, $t$ denotes discrete time index representing daily observations, which we assume takes integer values, $r$ is a parameter which denotes the number of mixtures present in the time series, and $\varepsilon_t$ represents independent, zero-mean noise which is bounded in absolute value almost surely by a parameter $0 \leq \delta < 1$. 
The model itself takes a reduced form approach to estimation, which results in the ability to perform a principled statistical analysis of a simple learning algorithm.
Moreover, we show that the model is the outcome of a simple Susceptible-Infected-Recovered (SIR) process on a network, which allows the learned parameters to become interpretable. 
The flexibility in the number of mixtures $r$ also allows the model to account for the possibility of additional peaks in the data, resulting in a benefit over inflexible assumptions of unimodal models used in the literature \citep{farr1840progress, murray2020forecasting}.
Because $r$ can vary, the model provides a non-parametric function class from which the trajectory of the epidemic can be learned, allowing for multimodal observations to be captured.

The idea of modeling cases as a mixture stems from the reality that the disease is spreading to a population which has diverse regional divisions and includes many jurisdictions \citep{chandrasekhar2020interacting}. Since each region has its own features and policies, we expect the observed case counts to take an additive form, and this is indeed the case within the United States (Fig. \ref{fig:empirical_mixtures}). The specific form of the Gaussian time series is chosen in part due to historical prevalence \citep{farr1840progress, santillana2018relatedness}, and the parameterization is rigorously justified in Section \ref{s:theory}. While a much wider variety of function classes can explain sub-exponential growth \citep{dandekar2020quantifying}, we restrict to the parsimonious class in \eqref{eq:mixtures} since the restrictive assumption better justifies applications to out-of-sample prediction.

\begin{figure}
    \centering
    \includegraphics[width=0.8\textwidth]{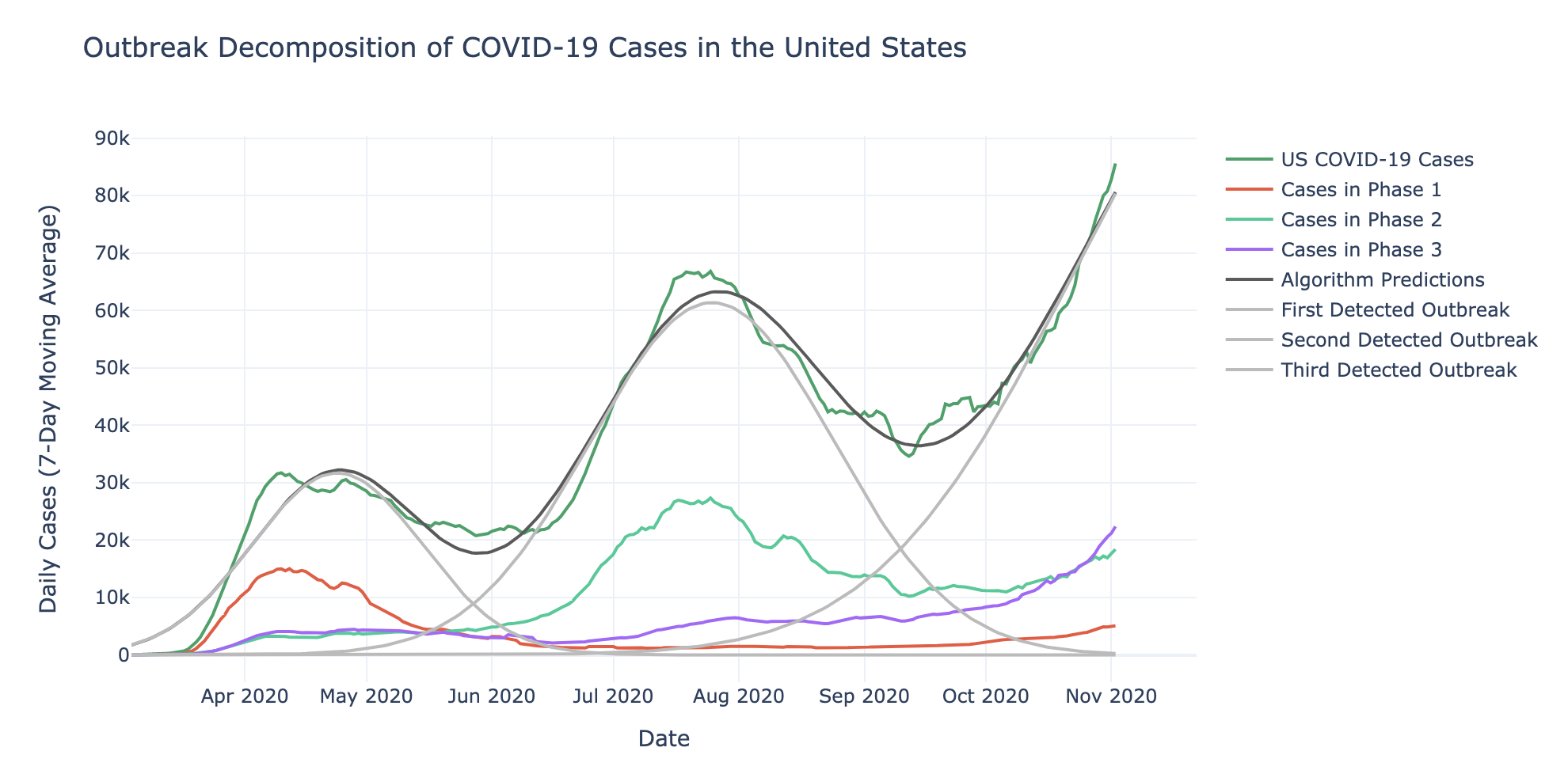}
    \caption{Empirical mixtures in United States COVID-19 Case Data. Different groups of states experience peaks at different times, indicating that a homogenous mixing assumption can not be made across the entire United States. }
    \label{fig:empirical_mixtures}
\end{figure}

After a brief review of the epidemic models and the motivation for the mixture model in Section \ref{s:background}, in Section \ref{s:algorithm} we show that \eqref{eq:mixtures} admits a simple algorithm for learning time series as a mixture of Gaussian curves.
This algorithm can be shown to provably learn the parameters of each component from data, even when the observations are perturbed by bounded noise, which builds confidence in the ability to use the model for data-driven tasks.
In bounding the error of learned parameters based on observations of data, to our knowledge we provide the first explicit statistical guarantee on learning a reduced form model for epidemic forecasting. 
The learning algorithm allows us to perform inference from empirical observations, and we have found that our prediction error is relatively small compared to other models used for epidemic forecasting \citep{ray2020ensemble}. 
We find that the model achieves the best results when our model selection procedure selects an appropriate number of components $r$, and in practice and for model comparisons we use the BIC criterion to select $r$ from in-sample data \citep{ding2018model}.

Moreover, we provide two possible ways to interpret the model from a mechanistic perspective.
First, we show the presented reduced form model can be seen as arising from a simple stochastic process on a graph, allowing for meaningful interpretation of parameters. 
We provide an explicit generative model which results in the function class \eqref{eq:mixtures} and provide a statistical analysis complete with non-asymptotic bounds for the stochastic process in Section \ref{s:theory}. 
Specifically, our generative model suggests that the parameter $a_k$ of each mixture measures the extent to which individuals in the population react to the pandemic by cutting off physical ties with others. 
Although our model is only one of many generative models which can produce observations of the form \eqref{eq:mixtures}, we show that our interpretation of the $a_k$ parameter can be validated with mobility data collected throughout the pandemic, providing evidence that this model provides appropriate insights. 
This validation indicates that policy makers can interpret the simple reduced form model to better understand the progress of the pandemic. 
That is, although the model does not explicitly assume an underlying mechanism, the learned parameters still have an interpretation which will allow for data-driven control methods to be applicable.

Since our network-based interpretation provides one of many possible ways to generate observations of the form \eqref{eq:mixtures}, we also provide a different formulation based on the standard SIR model to recover the Gaussian components of the model. 
Namely, whereas the network-based interpretation does not assume that individuals react as a function of the state of the epidemic, we show a closed-loop result which recovers the Gaussian shape with the assumption that individuals are directly responding to the state of the epidemic.
This allows us to better understand how the Gaussian form can result empirically.

We provide final thoughts and conclusions in Section \ref{s:conclusions}, which summarizes the key contributions of this work and lays out future directions of research.
Ultimately, we find that prediction based on the function class in \eqref{eq:mixtures} appears to unify the disparate approaches to epidemic modeling, striking a desirable balance between mechanistic models and their reduced form counterparts. 
While policy makers in practice may still prefer to use deep-learning based methods to prioritize accuracy \citep{shahid2020predictions}, or agent-based models to allow for refined interpretability \citep{rockett2020revealing}, equation \eqref{eq:mixtures} provides a foundation for models which provide both accurate and interpretable forecasting. 
This allows us to work towards an eventual goal of robust estimation of epidemics which can be actionable towards an eventual goal of simple and effective data-driven control for future epidemics. \\

\noindent \textbf{Summary of Contributions.} \quad Overall, the contributions of this work may be summarized as follows:
\begin{itemize}
    \item We propose a novel mixture model to capture epidemic processes, and argue that it captures the benefit of 
     mechanistic models in terms of being interpretable and the benefit of non-mechanistic models in terms of expressivity.
    \item We provide an efficient algorithm for learning the model from data, and provide theoretical guarantees. 
    \item We show that the proposed mixture model arises naturally as the outcome of a network Susceptible-Infected-Recovered (SIR) model and 
    thus provide an interpretation for the parameters of the model.
    \item We validate this interpretation of the parameters by associating learned parameters with empirical observations,
    and we find significant correlations which support the use of the generative model as well as the specific parameter interpretations.
    \item We provide an additional generative model in which the observations occur due to endogenous behavior, 
    which furthers our understanding of the mixture model and suggests an alternative reason as to why such observations may arise.
\end{itemize}

\section{Modeling Epidemics}
\label{s:background}
Both mechanistic and non-mechanistic epidemic models have been discussed at length in the epidemiology literature, with attention to the benefits and drawbacks of each method \citep{holmdahl2020wrong}.
Here, we review some of the main approaches to epidemic modeling, as to provide context for our model which we claim retains benefits of each type.

\subsection{Mechanistic Models for Epidemics}
\label{ss:mechanistic}
We begin with a discussion of mechanistic models in epidemiology, which will underlie some of the theoretical foundations of the generative network model which produces \eqref{eq:mixtures}. 
The fundamental mechanistic models are originally based on a mean-field approach introduced in \cite{kermack1927contribution}, and are known as compartmental models. 

\begin{figure}
    \centering
    \includegraphics[width = 0.5\textwidth]{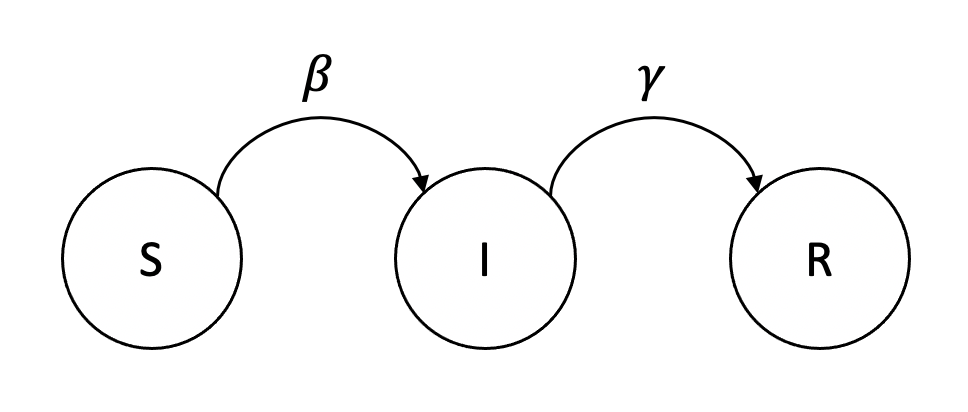}
    \caption{State Diagram of a compartmental model with three states: Susceptible (S), Infected (I), and Recovered (R).}
    \label{fig:sir}
\end{figure}

\subsubsection{Compartmental Models} Compartmental models aim to understand the dynamics of infectious disease by classifying the population into different groups, and using these classifications in order to understand the spread of the disease. 
Labels such as S, E, I, and R are often used to categorize such compartments, which stand for susceptible, exposed, infected, and recovered, respectively. 
Additional labels such as M and V, representing passively immune and vaccinated individuals, respectively, are also commonly considered as compartments in such models \citep{hethcote2000mathematics}. 
The appropriate labels will depend on the characteristics of the disease which is modeled.

Mean-field models in epidemiology model the dynamics of the infectious disease through a series of differential equations which describe the transmission of the disease. 
Figure \ref{fig:sir} shows the commonly used example of a compartmental model using three compartments. 
In this particular model, the dynamics between the compartments are modeled as if every individual has an equal probability of coming into contact with another individual, i.e. the underlying contact network is assumed to be fully connected.
The rate at which the size of the infected population then increases as the number of connections between infected individuals and susceptible individuals increases.
This results in the following set of differential equations to model the spread of disease.
\begin{align}
    \dv{S}{t} &= - \beta S(t) I(t) / N \,,\label{eq:sir-st}\\
    \dv{I}{t} &= \beta S(t) I(t) / N - \gamma I(t) \,, \label{eq:sir-it} \\
    \dv{R}{t} &= \gamma I(t) \label{eq:sir-rt} \,.
\end{align}
Here, $t$ represents continuous time, and the process is assumed to begin at some time $t_0$.
We assume that the number of births and deaths in the population is negligible, so that $S(t) + I(t) + R(t) = N$ for some positive constant $N$ for all $t$. 
Further, the initial conditions for such a system often assume that $S(t_0) \approx N$, $0 < I(t_0) \ll N$, and $R(t_0) = 0$. 
The value $\beta$ in this model represents a contact rate, in the sense that it is a measure of the number of contacts that an infected individual may have with the susceptible population. 
Similarly, the parameter $\gamma$ represents a recovery rate, and represents on average how long an infected individual remains contagious.

Several important observations can be made about the SIR model considered above. 
For example, the ratio $\beta / \gamma$, often denoted as $R_0$, the \textit{basic reproductive number} of the system, provides a simple condition to determine if the number of infections of the population will increase. 
Namely, if $R_0 < 1$, then the number of infected individuals will decay to 0; otherwise, the number of infected individuals will increase until it reaches a maximum before decaying to 0 \cite[Theorem 2.1]{hethcote2000mathematics}. 
It is also common to rewrite the SIR model with time-varying parameters, by considering how the parameters $\beta$ and $\gamma$ may change over time due, for example, to social distancing or increased testing \citep{holmdahl2020wrong}.
Such approaches are necessary when the observed data is multimodal, as the traditional SIR dynamics yield a unimodal time series of infections.

While the time-varying approach provides more flexibility to model observed dynamics, it also results in a difficult statistical procedure to estimate parameters with confidence \citep{chen2021numerical}.
For example, if the class of functions used to model time-varying parameters $\beta(t)$ has no constraints, then one can easily overfit to the data with a time-varying parameters by adapting the $\beta(t)$ parameter to each new observed data point.
Without any appropriate assumptions on the time-varying parameters, it becomes difficult to generate valid out of sample predictions.

Another approach which is often taken with the compartmental model above is to consider the case where the susceptible population is large compared to the number of infections, i.e. $S(t) \approx N$ for some range $t \in [t_0, T]$ (see, e.g. \cite{ma2020estimating}).
In such cases, because the dynamics of the susceptible population the dynamics of the number of infected individuals is the main point of interest, and can be approximated as
\begin{equation}
    \dv{I}{t} \approx (\beta - \gamma) I(t) \,,
\end{equation}
resulting in a simple differential equation with a closed form solution reflecting exponential growth or decay, depending on the sign of $\beta - \gamma$.
This model is useful for estimating epidemiological parameters at the early stages of an epidemic.
Furthermore, by allowing $\beta - \gamma$ to vary with time and take an affine form, the connection to the Gaussian curves in the mixture model \eqref{eq:mixtures} and Farr's law becomes apparent.
As we will show in Sections \ref{s:theory} and \ref{ss:closed-loop}, the idea that $\beta - \gamma$ can have an affine form is not unreasonable, as the affine form can be attained by making particular assumptions on a population's reaction to the progress of the epidemic.

Ultimately, although the mean-field compartmental models have been applied to various problems in epidemiology, they make a restrictive assumption in that they do not take into account potential heterogeneity in the physical contact networks of individuals who spread the infectious disease \citep{chakrabarti2008epidemic, easley2010networks, ganesh2005effect}. 
For example, because super-spreading events have quite common in the spread of COVID-19, it is unlikely that a homegenous contact rate $\beta$ is appropriate to model the dynamics of this infectious disease \citep{frieden2020identifying}. 
For this reason, many authors instead choose to use network-based models.

\subsubsection{Network-Based Models of Epidemics} 
\label{ss:background-network}
By modeling social networks, one may explicitly take into account the interactions between different individuals \citep{easley2010networks, ganesh2005effect, girvan2002simple, keeling2005networks, ruhi2015sirs}. 
In a network-based model, individuals are modeled using a set of nodes $V$ ($|V| = n$), and the interactions between the individuals are modeled using a set of edges $E \subseteq V \times V$. 
At each (discrete) time step $t \in \mathbb{N}$, each node in the network is assigned a label, $\xi_i(t) \in \{ S, I, R \}$. The state of the entire network can then be summarized in the time-varying vector

\begin{equation} \label{eq:3n}
    \xi (t) = (\xi_1(t), \dots, \xi_n(t)) \in \{ S, I, R \}^n \,.
\end{equation}

The model may be fully described by then stating the probability that an individual node transitions from one particular state to another. 
In many network-based models of epidemics, each infected node infects its neighbors independently with some probability $\beta$, and each infected node becomes recovered at each time step with probability $\gamma$. 

Hence, using these transition probabilities, one can study the dynamics of the $3^n$ state Markov chain represented by these underlying probabilistic rules \citep{ruhi2015sirs}. 
However, it is worth noting that such a large Markov chain is not always tractable, and several works consider approximations of this model, such as the linearized version of the dynamics \citep{ahn2013global, ruhi2015sirs}. 

Instead of providing an analysis of the $3^n$ state Markov chain described in \eqref{eq:3n}, some authors instead look for necessary and sufficient conditions for epidemics to either last or dissolve quickly on networks, to create an analog of the $R_0$ measure in compartmental models. 
In particular, they note that if $\beta / \gamma$ is less than the inverse of the spectral radius of the graph, then the epidemic may die out quickly, whereas above this threshold the epidemic will last for a long period of time \citep{ganesh2005effect}. 
While such results provide insight on the dynamics of infectious disease in a networked setting, such works do not specifically attempt to forecast future progress of the infectious disease. 

In epidemic forecasting with network-based mechanistic models, practitioners must instead make assumptions about the underlying social network, for example through agent-based modeling \citep{macal2009agent}. 
By collecting information about how individuals interact with their environment, such models can be made to handle heterogeneity in ways that mean-field models can not. 
These approaches require large amounts of data and resources to generate accurate predictions \citep{aleta2020modelling}. 
Since millions of edges may be included in the mobility network, and observed mobility often changes over time due to community response to the state of the pandemic, the amount of data required to satisfy complex, time-varying network-based models can become prohibitive.

\subsection{Non-Mechanistic Modeling of Epidemics}
In contrast to the mechanistic models of infectious disease considered above, several authors approach prediction of epidemics using a non-mechanistic, or reduced form approach \citep{farr1840progress,le2020neural, santillana2018relatedness}. 
In these approaches, the trajectory of key metrics such as case counts or fatailities is assumed to come from a specific class of functions, and optimization techniques are used in order to determine the parameters of the models and provide forecasts. 
While such approaches often lack interpretable parameters, they do have computational benefits, are data-driven, and are efficient in their use of data as they provide forecasts based only on the available time series.

The most commonly considered class of functions for predicting the number of infected individuals is the class of Gaussian bell curves, which have been considered as early as the mid 19th century with the work of William Farr \citep{farr1840progress}. 
The use of this function class has been considered in predicting infections due to AIDS \citep{bregman1990farr}, smallpox \citep{santillana2018relatedness}, drug mortality \citep{darakjy2014applying}, and COVID-19 \citep{murray2020forecasting}.
The application of such a function class has seen mixed results, particularly in the case of long term forecasting.
In particular, the function class makes the restrictive assumption that the spread of disease is unimodal in nature.

Recently, additional function classes based on advances in machine learning have also been considered in the context of the spread of infectious disease \citep{le2020neural}. 
With such models, the class of functions used for prediction is not necessarily unimodal, and practioners are able to incorporate auxiliary data such as mobility data in order to make predictions with a non-mechanistic approach.

Our work presents a non-mechanistic approach of fitting to the function class \eqref{eq:mixtures}. 
Since the approach is non-mechanistic, it has the benefit of being data-driven and having an efficient implementation.
Further, as we will show in the following section, this function class may actually be thought of as arising from a mechanistic SIR-process. 
In particular, the number of peaks $r$ in \eqref{eq:mixtures} provides an indication of the number of communities in the network as opposed to a more granular assumption on network topology; in this sense, the approach strikes a balance between the unstructured mean-field SIR models and the overly explicit network-based SIR models, while retaining the benefits of non-mechanistic approaches.

\section{Learning Mixtures from Data}
\label{s:algorithm}
Learning a mixture of Gaussians as a time series requires novel algorithmic insights due to the setting of the problem.
The input to the problem are a series of observations $\{N (t)\}_{t \in \{1, \dots, T \}}$ which are drawn according to the model \eqref{eq:mixtures}.
The goal of the learning procedure is then to identify a number of mixtures $r$ and the parameters $(a_k, b_k, c_k)$ for $k = 1$ through $r$ which provide a best fit to the data.
To operationalize a best fit to the data, we will use the $\ell_2$ loss, i.e. we wish to find parameters which minimize
\begin{equation} \label{eq:l2-loss}
L_r(\{a_k, b_k, c_k \}_{k = 1, \dots, r}) = \sum_{t = 1}^T \left( \sum_{k = 1}^r e^{-a_k t^2 + b_k t + c_k} - N(t) \right)^2 \,.
\end{equation} 

This is a difficult non-convex optimization problem, and it is worth noting that it is markedly different than learning a distribution from a mixture of Gaussians. 
Although a method such as expectation-maximization (EM) can be used to learn a mixture of Gaussians from a collection of samples from a distribution, learning from a time series provides a markedly different setting \citep{moon1996expectation}.
Rather than learning from a collection of samples, we must learn directly from the observed points of the time series, which would be comparable to being given a partial density and attempting to estimate the complete density.
The problem of fitting parameters to the Gaussian distribution remains non-convex, and hence we provide a novel algorithm to determine valid parameters and provide confidence intervals for inference.

Our algorithm is based on analyzing a transformation of $N(t)$, which allows us to find a reasonably separated set of mixtures by using a peak finding algorithm. 
We will begin by providing an intuition for the algorithm, which is then made explicit in Algorithms \ref{alg:initialization} and \ref{alg:alternating}. 
We will then show a statistical guarantee showing that, under certain noise assumptions, the algorithm performs well when the number of peaks in the data is small and sufficient data is available.

\subsection{Algorithmic Intuition: Single Peak with No Noise}
To begin our algorithmic intuition, we first consider the simplest possible setting for learning a function class in our setting, which is the noiseless case with $r=1$.
That is, we first assume the observations have the form
\[ N(t) = e^{-at^2 + bt + c} \,.\]
Optimizing \eqref{eq:l2-loss} in this simplified setting still remains non-convex in general, so the goal is to find a reasonable initialization point from which to compute gradient descent.
Ideally, gradient descent from this point of initialization would recover the true parameters of the Gaussian curve.
While the final goal is to learn the parameters $a$, $b$, and $c$, we note that the learning procedure is more intuitive if we re-write the form of the Gaussian as
\[ N(t) = M e^{-a (t-C)^2} \,,\]
where $M = e^{c + b^2 / (4a)}$ and $C = \frac{b}{2 a}$.
This re-writing comes from optimizing the quadratic in the exponent in the Gaussian, and it allows for the parameters to take on specific meanings: $M$ represents the maximum value attained by the Gaussian, $C$ represents the time at which the maximum is reached, and $a$ represents a curvature of the Gaussian.

Hence, in order to learn the parameters a single Gaussian curve in the noiseless case, we can first estimate $\hat{M}$ as the maximum of the observed counts, $\hat{M} = \max_t N(t)$.
$\hat{C}$ would then be estimated as the value of $t$ for which the peak is reached, $\hat{C} = \argmax_t N(t)$.
To learn $a$, we note that the function
\begin{equation}
\label{eq:st}
S(t) = \log \frac{N(t+1)}{N(t)} - \frac{N(t)}{N(t-1)} 
\end{equation}
is precisely equal to $-2a$ in this noiseless, single peak case.
So, we estimate $\hat{a} = - S(t) / 2$ where $t$ is an arbitrary index between $1$ and $T$ for which $S(t)$ is well defined.

To convert from $(\hat{a}, \hat{M}, \hat{C})$ to $(\hat{a}, \hat{b}, \hat{c})$, we note that in the form of the Gaussian, given the definitions of $M$ and $C$, we can solve for $b$ and $c$ as $b = 2 C a$ and $c = \log M - C^2 a$. 
Thus, we can estimate the parameters of the quadratic as
\begin{align*}
    \hat{a} = \hat{a} \,, \qquad
    \hat{b} = 2 \hat{C}\hat{a} \,,\qquad
    \hat{c} = \log \hat{M} - \hat{C}^2 \hat{a} \,.
\end{align*}

This initialization may not be perfect in the case where the true maximizing value $C$ does not lie on an index $t$ which is observed, as for example the maximizing value $C$ could be between two integers.
Hence, we apply gradient descent from this initialization point to find a local minimum which is a best fit for the Gaussian curve.
In practice, this often recovers the true parameters of the Gaussian.

\begin{figure}
     \centering
     \begin{subfigure}[]{0.4\textwidth}
         \centering
         \includegraphics[width=\textwidth]{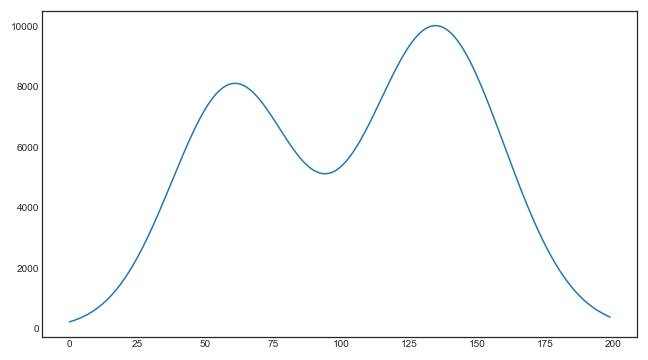}
         \caption{Idealized Number of Cases $N(t)$}
         \label{fig:idealized_cases}
     \end{subfigure}\hspace{0.5in} \begin{subfigure}[]{0.4\textwidth}
         \centering
         \includegraphics[width=\textwidth]{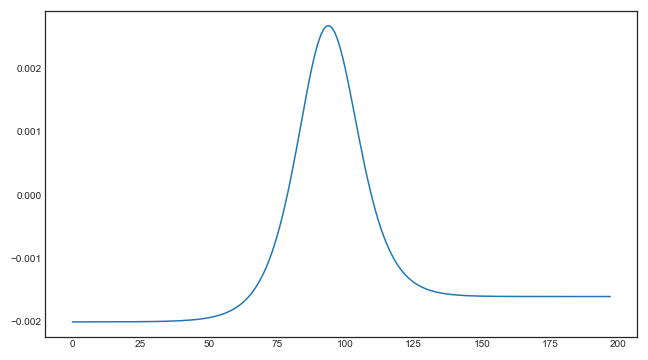}
         \caption{Idealized $S(t)$}
         \label{fig:idealized_st}
     \end{subfigure}
     \caption{Idealized Time Series for Learning a Mixture.}
     \label{fig:idealized_alg}
\end{figure}

\subsection{Algorithmic Intuition: Multiple Peaks with No Noise} \label{ss:intuition-mult}
The above handles the simplest possible case for learning, where $r = 1$.
We next provide intuition to extend to the cases where $r > 1$.
To begin, we consider the noiseless case where the observations come from a mixture of two Gaussian curves,
\[ N(t) = \sum_{k = 1}^2 e^{-a_k t^2 + b_k t + c_k} \,. \]
The intuition for the learning procedure is as follows: We first identify the ``mid-point'' at which the two Gaussians have a similar number of cases, and then use this as a separating point to learn one Gaussian at a time.

To identify this mid-point, we make an observation from the computation of $S(t)$ in this idealized case as shown in Figure \ref{fig:idealized_alg}. 
Namely, we note that in regions where the two clusters have similar counts, $S(t)$ increases to reach a local maximum.

In our algorithm, we exploit this local maximum in order to identify the midpoint between Gaussian components, creating disjoint intervals of time in which each interval corresponds to a single dominant Gaussian curve. 
Once these midpoints are defined, the problem is reduced to identifying the parameters of the dominant Gaussian components in each interval. 
As our theoretical results will indicate, the task of identifying such parameters is simple as long as the Gaussian curves are well-separated. 

After determining the initialization parameters from each Gaussian, done as in the single peak case, we again use gradient descent to optimize parameters due to the discrete nature of the observations.
In the case of more than one Gaussian, we perform an alternating minimization, as the procedure works well empirically to learn separated mixtures.

\subsection{Algorithmic Intuition: Multiple Peaks with Noise}
In order to account for the fact that there may be noise in observations, we utilize averaging in learning the $a_k$ parameters.
That is, rather than let the estimate of $a_k$ depend on one value of $S(t)$, we use averaging so that the noise from multiple observations can be handled.
In practice, this allows for a better estimation of the curvature of each Gaussian component in the mixture.

Moreover, in practice, the last question which remains is to determine the number of mixtures $r$ used in the mixture.
To determine a value of $r$ which fits the data well but can be used in out-of-sample prediction, we use the BIC criterion \citep{ding2018model, schwarz1978estimating}, and select the number of peaks as 
\[ r^* \in \argmin_{r} (3r) \ln T - 2\ln L_r(\{\hat{a}_k, \hat{b}_k, \hat{c}_k \}_{k = 1, \dots, r})\,. \]
Here, $\{\hat{a}_k, \hat{b}_k, \hat{c}_k \}_{k = 1, \dots, r}$ represent the $3r$ learned parameters from the procedure outlined above.
This allows us to justify that the model fits the data well without using too many parameters indicative of overfitting.

Our algorithms are formalized using pseudo-code in Algorithms \ref{alg:initialization} and \ref{alg:alternating}. 
In the next section, we provide theoretical results which indicate the effectiveness of the algorithm subject to bounded noise.

\begin{algorithm}[t]
\begin{algorithmic}[1]
\REQUIRE{Time series $N(t)$ for $t = 1, \dots, T$, number of peaks $r$} 
\ENSURE{$a_k, b_k, c_k$, initial parameters for each mixture} \\
\STATE{Compute $S(t) = \log \frac{N(t+1)}{N(t)} - \log \frac{N(t)}{N(t-1)}$}\\
\STATE{Compute a sorted vector $x[1, \dots, r-1]$, consisting of the locations of $r-1$ peaks of $S(t)$.\protect\footnotemark}
\STATE{Let $x[0] = 0$ and $x[r] = T$.}\\
\FOR{$k=1,2,\ldots, r$}
\STATE{$a_k = - \frac{1}{2 (x[k] - x[k-1])} \sum_{t = x[k-1]}^{x[k]-1} S(t)$} \\
\STATE{$M_k = \textrm{max}_{x[k-1] \leq t \leq x[k]} N(t)$ }\\
\STATE{$C_k = \textrm{argmax}_{x[k-1] \leq t \leq x[k]} N(t)$}
\STATE{$b_k = 2 C_k a_k, c_k =  \log M_k - C_k^2 a_k$}
\ENDFOR
\end{algorithmic}
 \caption{$\texttt{Initialize}(N(t), r)$ -- Initialization for $r$ mixtures}
\label{alg:initialization}
\end{algorithm}
\footnotetext{This can be done using, e.g., a peak-finding algorithm from \texttt{scipy} \citep{virtanen2020scipy}.}

\begin{algorithm}[t]
\begin{algorithmic}[1]
\REQUIRE{Time Series $N(t)$} 
\ENSURE{$a_k, b_k, c_k$, learned parameters for each mixture}
\STATE{$\hat{a}_k, \hat{b}_k, \hat{c}_k = \texttt{Initialize}(N(t), r)$ for $k = 1, \dots, r$}
\STATE{For each $k$, $a_k = \hat{a}_k$, $b_k = \hat{b}_k$, $c_k = \hat{c}_k$}
\WHILE{convergence not yet reached}
\FOR{$k = 1, 2, \ldots, r$}
\STATE{$N_{-k} (t) = \sum_{i \neq k} e^{-a_i t^2 + b_i t + c_i}$}
\STATE{$a_k, b_k, c_k = \textrm{argmin}_{a, b, c} \sum_{t = 1}^T \left(N(t) - N_{-k}(t) - e^{-a t^2 + bt + c} \right)^2$}
\ENDFOR
\ENDWHILE
\end{algorithmic}
\caption{Alternating Minimization}
\label{alg:alternating}
\end{algorithm}


\subsection{Parameter Estimation of the Reduced Form Model}

We are able to provide a provable guarantee for Algorithm \ref{alg:initialization} for the case $r = 2$, and are able to extend to the case $r > 2$ so long as only two peaks are non-negligible at any given time. 
First, we recall that \eqref{eq:mixtures} can be written in the following form for $r = 2$:
\begin{equation} \label{eq:mixtures-rewrite}
    N(t) = \left(\sum_{k = 1}^2 M_k e^{-a_k(t-C_k)^2}\right) (1 + \varepsilon_t) \,,\qquad k = 1, 2 \,.
\end{equation}
Without loss of generality, we will use the convention $C_1 \leq C_2$.

To present our theoretical result, we will make the assumption that $a_1 = a_2$, which will aid in clarity.
However, this assumption can be relaxed, as is discussed in the appendix.
The relaxation of the assumption requires a more intricate bound on the size of the noise $\varepsilon_t$ as well as an assumption that the proportion of cases observed from Community $2$ is increasing on the observed interval, which are both reasonable in this context.
The additional assumptions on underlying parameters are justified after the theorem statement, which is as follows.

\begin{theorem}[Parameter Estimation Bounds] \label{thm:parameter_estimation}
Suppose $a_1 = a_2$, and suppose the parameters of the mixture model and the bound on the noise $\delta$ satisfy the following properties for some $0 < \epsilon  < \min\{M_1, M_2\} / 5$, and $a, M > 0$.
\begin{enumerate}
    \item $M_k \leq M\,, \quad k = 1, 2.$
    \item $a_k \geq a\,, \quad k = 1, 2.$
    \item $|C_1 - C_2| \geq 2 \sqrt{\frac{1}{a} \log \frac{M}{\epsilon}} \,.$
    \item $\delta \leq \delta^*(a_1, C_1, C_2, M_1, M_2)\,,$ where $\delta^*(a_1, C_1, C_2, M_1, M_2)$ is defined as in \eqdeltastar in the Appendix.
    \item $C_1, C_2 \in [0, T]$
\end{enumerate}
Denote $\hat{M}_k$ as the estimate in line 6 and $\hat{C}_k$ as the estimate produced in line 7 of Algorithm \ref{alg:initialization}. Then, for $k = 1, 2$
\begin{align}
    M_k [1 - \delta] &e^{-a_k(\lceil C_k \rceil - C_k)^2} \leq \hat{M}_k \leq [M_k + \epsilon] [1 + \delta] \,, \label{eq:mk} \\
    |\hat{C}_k - C_k| &\leq \sqrt{\frac{1}{a_k} \log \left(\frac{M_k}{M_k e^{-a_k(\lceil C_k \rceil -C_k)^2} \left( \frac{1-\delta}{1+\delta}\right) - \epsilon } \right)} \label{eq:ck} \\
    &\approx \sqrt{\frac{1}{a_k} \left(\frac{1 -  e^{-a_k(\lceil C_k \rceil - C_k)^2} \left( \frac{1-\delta}{1+\delta}\right) }{e^{-a_k(\lceil C_k \rceil -C_k)^2} \left( \frac{1-\delta}{1+\delta}\right)} + \frac{\epsilon}{M_k} \right)} \nonumber\,.
\end{align}
\end{theorem}
\begin{proof}
(Sketch) We provide a sketch of the three major steps of the proof, which are discussed in Section \appendixalgproof of the Supplementary Material. 
The proof of this claim follows in three steps.
First, we show that Assumptions 1-3 result in a condition where, if one component of the mixture model is dominant, then the other is at most $\epsilon$.
Next, we show that because the noise is small due to Assumption 4, the value defined as 
\[ t_m = \argmax_{1 \leq t \leq T-1} S(t) \,,\]
which corresponds to the output of Line 2 in Algorithm \ref{alg:initialization}, identifies an appropriate point which constitutes a midpoint between the two components of the mixture.
That is, the estimate of $\hat{C}_1$ will be in a range where component 1 comprises the majority of cases, and $\hat{C}_2$ will be selected at a point in time where component 2 comprises the majority of cases.
Finally, we combine these results to show that, because the contribution of the non-dominant component is at most $\epsilon$, and the estimates of $\hat{M}_k$ and $\hat{C}_k$ are associated with the correct component $k$, that the claim holds.
\end{proof} 
$ $\newline

\noindent \textbf{Interpretation of Assumptions.\quad} The first condition of this theorem is straight forward, as it states that the number of daily cases in each mixture must be bounded.
The second condition then states that each Gaussian component must have some curvature, i.e. each component can not be too flat.
The third condition is that of \emph{temporal separation}, and requires that the two components of the mixtures are spread apart in time.
Such a condition is common to the study of identification of mixtures, as the separation is key acknowledging that the object of study should be described as distinct components forming a mixture (see, e.g. \cite{daskalakis2017ten}).
In fact, multiplying both sides of this third condition by $\sqrt{a}$, and identifying $a$ with an inverse variance term, we see that it is similar to a ``mean-over-variance'' condition of these works. 
However, the condition here instead reflects a ``mean-over-standard deviation'' condition as the requirement is on $\sqrt{a}$ as opposed to $a$ itself.
In the identification of distributions, separation is usually determined as a function of mean and variance of component distributions, whereas here the main requirement is temporal separation.
The fourth condition simply requires that the observation noise is not particularly large.
Moreover, because the bound on $\delta$ is non-decreasing in $|C_1 - C_2|$, we see that the temporal separation condition allows for the algorithm to become more robust.
Finally, the fifth condition simply requires that for these bounds to hold, the peaks must be observed in the data. \\

\noindent \textbf{Example on US COVID-19 Case Data.\quad} The conditions outlined above are reasonable to assume in practice: For $a = 0.0005$, which roughly corresponds to each ``standard deviation'' of the Gaussian curve lasting approximately 30 days (such that 95\% of total cases in the individual outbreak occur over roughly 100 days), and $M = 300,000$, which is approximately the highest number of daily COVID cases seen in the United States as of September 2021, we find that Assumption 3 of the Theorem requires that $|C_1 - C_2| \geq 105$ for $\epsilon = 13,000$.
That is, under conservative estimates on the parameter bounds, and an $\epsilon$ which is an order of magnitude lower than the maximum number of cases $M$, the distance between peaks must be approximately $3.5$ months, is reasonable as the first two local maxima in daily new cases in the United States were April 9th and July 24th, which have 106 days between them.
For Assumption 4, estimates using the parameters above yields that $\delta$ may not exceed approximately 0.02\%, which is restrictive, but becomes reasonable when case counts are large.
Moreover, as we show in the following section, $\delta$ can be as large as 5\% in synthetic experiments and the combination of Algorithms \ref{alg:initialization} and \ref{alg:alternating} can find parameters which fit the data well, so long as the midpoint between the two components is determined reasonably. \\

Ultimately, our results show that parameter estimation bounds can be achieved which depend on the size of the overlap between the two Gaussian components of the mixture as well as the magnitude of the noise $\varepsilon_t$.
Even with this variability in parameter estimates from the initial results of Algorithm \ref{alg:initialization}, we find empirically that after the alternating minimization process of Algorithm \ref{alg:alternating} is performed, that the empirical results of the algorithm are surprisingly close to observed data, both in and out of sample.
These results are discussed in the following section.

\begin{figure}
    \centering
    \includegraphics[width=\textwidth]{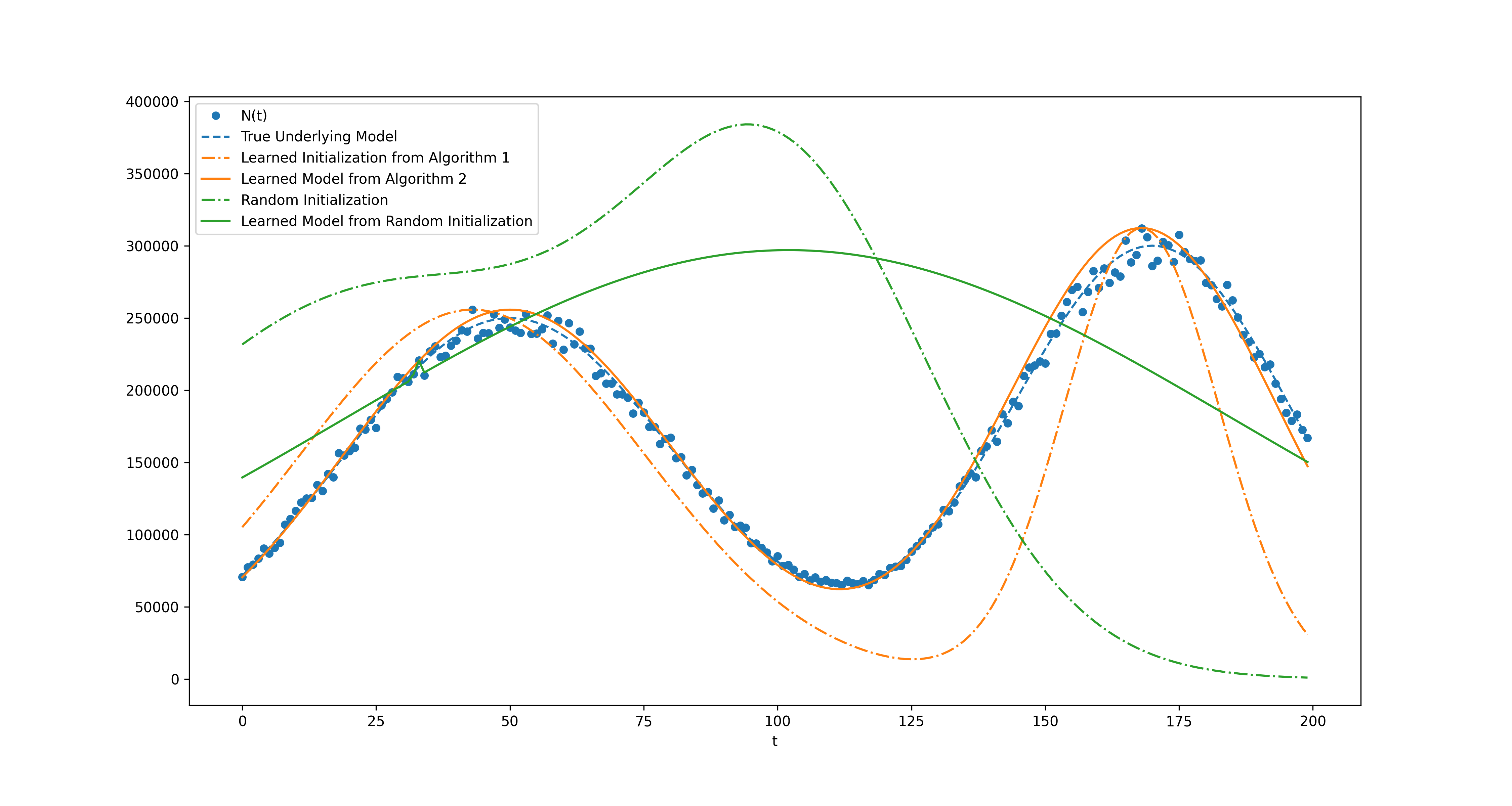}
    \caption{Performance of the Algorithm Compared to Random Initialization.}
    \label{fig:syn_results}
\end{figure}

\subsection{Synthetic Performance of the Mixture Approach}

We first show the performance of the algorithm on a synthetic dataset.
We provide a synthetic example of a time series with the following parameters:
\begin{align*}
    (a_1, M_1, C_1) &= (0.0005, 250000, 50) \,, \\
    (a_2, M_2, C_2) &= (0.0007, 300000, 170) \,,
\end{align*}
with 200 observations from $t = 0$ to $199$, and noise $\varepsilon_t$ which is uniform over the interval $[-0.05, 0.05]$.
Figure \ref{fig:syn_results} highlights the value of the initialization from Algorithm \ref{alg:initialization} over an arbitrary initialization selected at random. 
For the random baseline, we select $a_1$ and $a_2$ uniformly at random from $[0, 0.001]$, $M_1$ and $M_2$ uniformly at random from $[0, 300000]$, and $C_1$ and $C_2$ uniformly at random from $[0, 100]$ and $[100, 200]$, respectively.
For both types of initialization, we then apply Algorithm \ref{alg:alternating}, which performs gradient descent in an alternating fashion to minimize the cost function.
These results hold over multiple random initializations.
The $r^2$ score from using the initialization from Algorithm \ref{alg:initialization} results in $r^2 = 0.984$. 
In contrast, against the baseline of a random initialization run over $1,000$ different initializations drawn as above, the median $r^2$ from this baseline model is $-2.43\times 10^{5}$, with none of the $1,000$ initializations resulting in a fitted model which provides better performance than the initialization from Algorithm \ref{alg:initialization}.

\subsection{Empirical Performance of the Mixture Approach}
Figure \ref{fig:mixtures-acc} shows the utility of our forecasting method. 
The median absolute percent error (MAPE) of our approach is 15.9\%, for one week forecasts, compared to a median of 18.7\% for the same metric across all models used by the COVID-19 forecasting hub.
For two week forecasts, the MAPE of our approach is 20.6\% compared to a median MAPE of 25.9\% across other models in the forecasting hub.
Hence, in the short term we find that our model has low median forecasting error.
In general, using the BIC criterion to select $r$, we find that our method is competitive in terms of forecasting COVID-19 cases, while not requiring any auxiliary data such as mobility or an excessive computational burden.
The results of our forecasts are reproducible and can be viewed online at
\begin{center}
    \url{covidpredictions.mit.edu}
\end{center}

\begin{figure}
     \centering
         \includegraphics[width=0.32\textwidth]{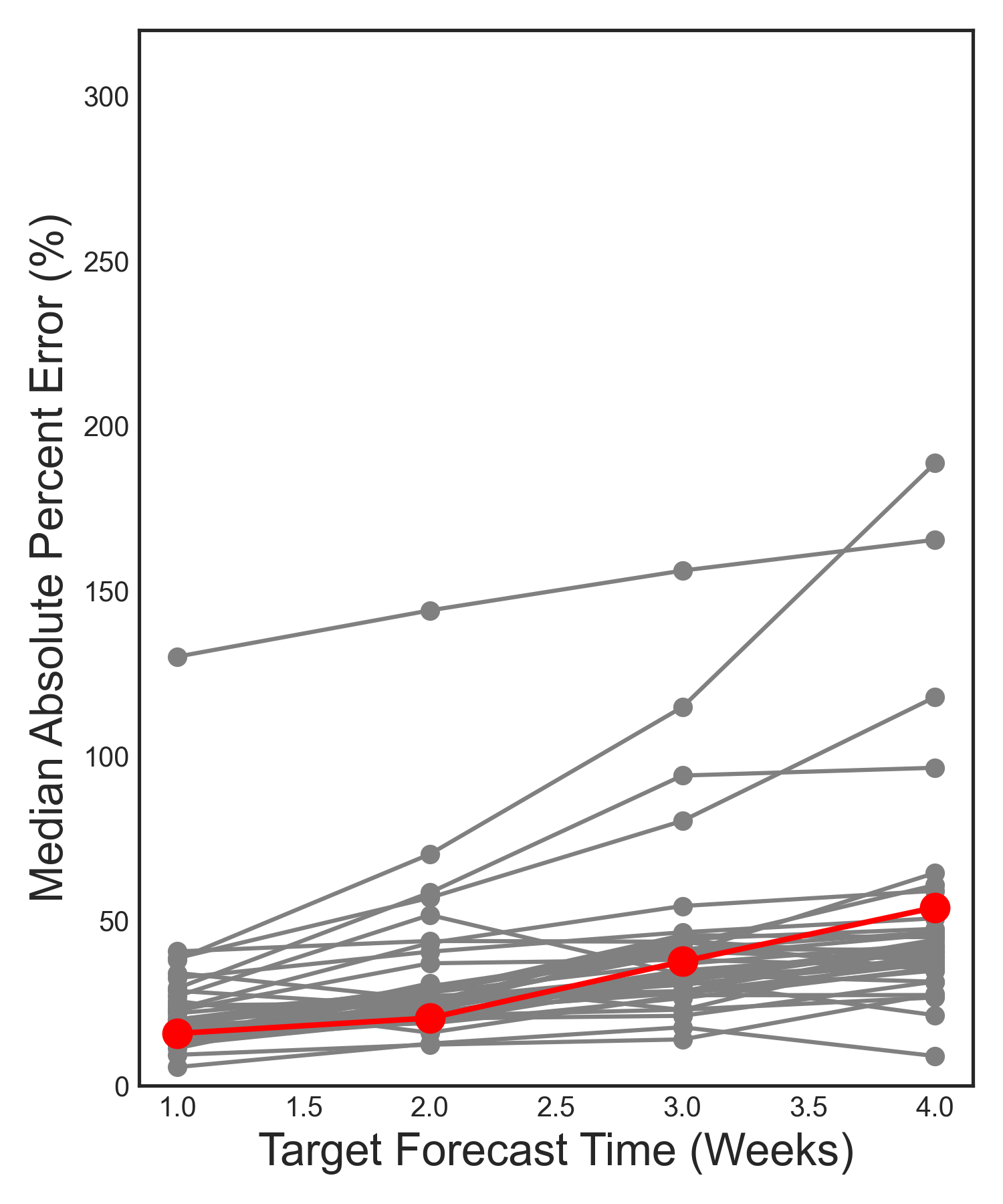}
         \includegraphics[width=0.635\textwidth]{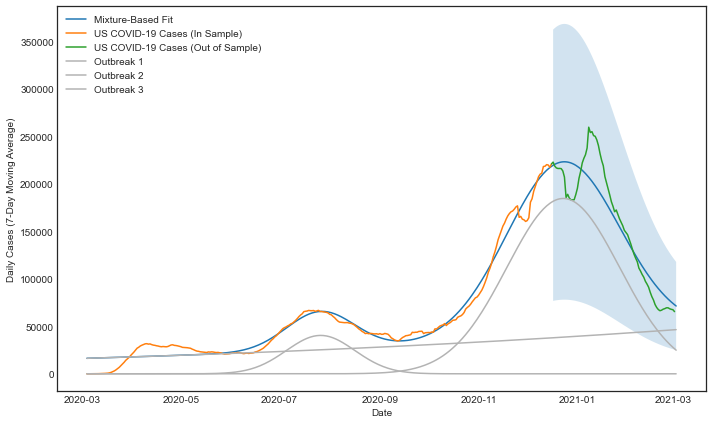}
        \caption{Accuracy of the Mixture Model in Predicting COVID-19 Cases for the United States. (Left) In terms of median accuracy of one week ahead forecasts, the mixture model (red) places 14th out of 34 available models, and 4th among models that do not use data apart from case counts.\protect\footnotemark~ (Right) An example of a remarkable forecast based on \eqref{eq:mixtures}, when the number of components $r$ is selected appropriately.}
        \label{fig:mixtures-acc}
\end{figure}

\footnotetext{{Predictions are collected from \url{https://github.com/reichlab/covid19-forecast-hub} and true case data is provided by Johns Hopkins.}}

\subsection{Limitations} Our methodology relies on a non-mechanistic approach, and implicitly makes assumptions on human response to the pandemic, as discussed in the following section.
When such assumptions are not satisfied, as is the case when there appears to be a new outbreak of cases emerging, our method does particularly poorly. This occurs because there is little evidence of the new cluster within existing data and hence $r$ can not be selected appropriately.
Hence, when new waves of the pandemic occur, our approach tends to deteriorate particularly in long term performance.
Such an issue provides clear avenues for future research, as predicting when a new wave of cases may occur is a critical problem in understanding the way in which epidemics spread.
Such work would increase prediction accuracy and inform policy decisions to optimize the allocation of healthcare resources.

\section{Generative Basis for Mixtures}
\label{s:theory}
The results of the previous section highlight the benefits of using the simple, mixture-based model \eqref{eq:mixtures} in COVID-19 modeling.
The model has a clear basis in the epidemiological literature, as Farr's law has been used to forecast the progress of epidemics since the 19th century, and still remains in use for forecasting drug mortality and the spread of infectious disease \citep{farr1840progress, darakjy2014applying, santillana2018relatedness}.
Moreover, the use of mixtures and the free parameter $r$ selected for the model allows it to be non-parametric, such that it can adapt to the multiple waves of the pandemic which have been observed empirically in the case of COVID-19.
As evidenced by the previous section, this choice of model has clear statistical benefits, as it is flexible, can be provably learned from data, and has reasonable forecasting ability.
Moreover, the simplicity of the learning procedure, which simply relies on the analysis of the time series $N(t)$, lacks the same computational burden associated with complex, network-based models which use auxiliary data to provide forecasts for the spread of infectious disease \citep{aleta2020modelling}.

That being said, while the model \eqref{eq:mixtures} provides an intuitive mixture model, can be learned efficiently from data, and performs well empirically, it does not immediately provide interpretation in the same way as traditional mechanistic
models, as it does not explicitly encode epidemiological parameters such as the infectiousness of disease into its parametric form. 
To this end, in this section we introduce a simple stochastic model that captures population heterogeneity, incorporates traditional epidemiological dynamics, and provides a mechanistic justification for the model \eqref{eq:mixtures}.
Our analysis consists of a tight statistical characterization of this stochastic model which provides the bridge between mechanistic models and the non-mechanistic model of \eqref{eq:mixtures}.
In Section \ref{ss:theory-mobility}, we then provide empirical evidence which lends validity to the generative model.

The model \eqref{eq:mixtures} has two key features which we wish to explain through a generative model: the Gaussian shape of each component and the summation which allows for multiple components to be present in the data.
First, we show how the Gaussian components in each term of the sum can arise when the individuals in the social network partake in a process we term \emph{degree pruning}, which refers to individuals severing ties with one another over time.
As we show, when degree pruning occurs at a constant rate and the population is sufficiently large, the case counts in a single community precisely follow a Gaussian shape.
To show this result for a single component, we restrict to the case where the graph consists of a single community as modeled by an Erd\H{o}s-R\'enyi graph and provide a detailed probabilistic analysis of the generative model in this case.
We also provide a discussion on how Gaussian curves can arise in more general models with a single community, for example when degree distributions are not homogenous.

Next, we show how the generative model can explain the summation which allows for $r$ separate components in the mixture model \eqref{eq:mixtures}.
Namely, by allowing the underlying graph structure to be drawn from a stochastic block model, which explicitly encodes community structure, we are able to show that the observed time series can consist of multiple distinct components which each take the form of Gaussian curves.
Specifically, our model indicates that if there are few connections between communities, then with high probability there will be temporal separation between the first case in the originating community and the first case in another community.

\begin{figure}
    \centering
    \includegraphics[width=0.7\textwidth]{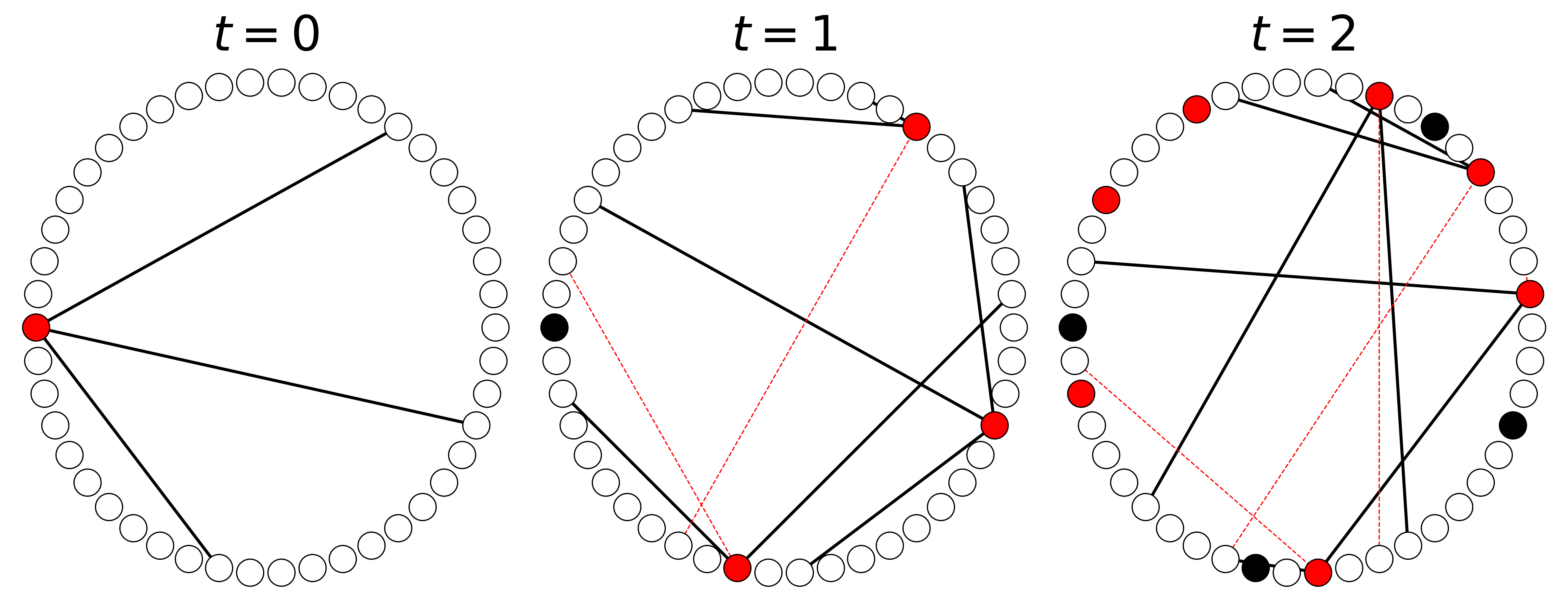}
    \caption{Visualization of Degree Pruning with Deferred Randomness. 
    At each time $t$, the edges associated with infected nodes (red) are revealed, beginning with a single infected node at time $t=0$.
    After a single epoch, each infected node becomes recovered (black).
    As $t$ increases, individuals in the population are also continually severing ties, as indicated by the dashed red edges which represent those edges which would have been present had they not been pruned.
    For simplicity, this visualization assumes $\beta = 1$, i.e. if there is an edge between an infected node and a susceptible node, the susceptible node will become infected at the next time step.}
    \label{fig:degree-pruning}
\end{figure}

\subsection{Justification of Gaussian Components}
\label{ss:theory-gaussian}

In this section, we provide a simple generative theoretical model which results in the Gaussian shape observed in \eqref{eq:mixtures}, and restrict to the case $r = 1$ for exposition. \\

\noindent \textbf{Model Description for Single Community.\quad} The model begins with a set of nodes $V$, where each node $v \in V$ represents an individual within the community, and for notation we let $|V| = n$ represent the number of nodes.
We consider the progression of infection among nodes in $V$ across discrete time steps $t = 0, 1, \dots,$.
At each time $t$, as in traditional network-based models of epidemics, an individual is either susceptible, infected, or recovered.
We denote $S(t) \subseteq V$, $I(t) \subseteq V$, and $R(t) \subseteq V$ as the sets representing susceptible, infected, and recovered individuals, respectively.
At each time $t$, we have that the sets $S(t)$, $I(t)$, and $R(t)$ are disjoint, and $S(t) \cup I(t) \cup R(t) = V$.
That is, the model has an equivalent representation as the network-based models discussed in Section \ref{ss:background-network}.
In our model, at $t = 0$, we assume the initial conditions 
\[ I(0) = \{ v_0 \}, \qquad R(0) = \emptyset, \quad \text{ and } \quad S(0) = V \setminus \{v_0\} \,, \]i.e. that there is a single initially infected node in the graph and that the remaining nodes are susceptible.

Next, we describe the mechanisms by which infection can spread between nodes.
In this model, infection spreads based on edges which exist between pairs of nodes in $V$.
For analytical purposes and clarity of presentation, we consider a model of \emph{deferred randomness}, in the sense that at time $t$, the only edges which are revealed will be those associated with nodes in $I(t)$.
For example, at time $t = 0$, the edges revealed will be those associated with $v_0$, as visualized in the first panel of Figure \ref{fig:degree-pruning}.
Edges are revealed between infected and susceptible nodes based on a probability which is time-varying, and we assume that the presence of each is determined independently of all other edges.
Specifically, for any node $v_i \in I(t)$ and any node $v_s \in S(t)$, the probability that an edge forms between $v_i$ and $v_s$ at time $t$ is equal to 
\[ \frac{d}{n} \gamma^t \,,\]
where we recall $n$ is the number of nodes, $0 < \gamma < 1$ captures the aforementioned notion of degree pruning\footnote{We assume here that $\gamma < 1$. In the case $\gamma = 1$, for finite $n$ the dynamics become that of a traditional mean-field SIR model as discussed in, e.g. \cite{xue2017law}, where the initial stages of the epidemic are a pure exponential growth and the size of the susceptible population governs the disease dynamics.},  and $d$ is a parameter which captures a notion of average degree measure of the graph, in the absence of any degree pruning.
Specifically, the term $\gamma^t$ reflects the notion that at each time step $t$, each edge which would have been revealed is removed from the graph with probability $1 - \gamma$.
This introduction of the parameter $\gamma$ is precisely what will allow for the quadratic term in the Gaussian form of \eqref{eq:mixtures} to appear in the analysis, as otherwise one would expect pure exponential growth in this model in initial stages when $n$ is large.

With the randomness of the graph structure presented, in order to complete the characterization of the model we specify the epidemiological parameters associated with the infectiousness of the disease and the rate at which individuals who are infected become recovered.
We assume there is a parameter $0 < \beta \leq 1$ such that, if there is an edge between an infected individual and a susceptible individual at time $t$, infection spreads with probability $\beta$.
Hence, for any node $v_i \in I(t)$ and any node $v_s \in S(t)$, the probability that $v_i$ infects $v_s$ becomes exactly $\beta \frac{d}{n} \gamma^t$.

In the model, we also assume that if an individual is infected at time $t$, then they become recovered at time $t+1$.
This assumption, which states that each infection lasts for a single time epoch, is no more restrictive than the assumption that nodes cure themselves with a constant probability at each time step, as is often considered in the literature (cf. \cite{chakrabarti2008epidemic, easley2010networks, ruhi2015sirs}). 
In such literature, the assumption of a constant probability of cure at each epoch results in a distribution of infection time which is memoryless, and is often not the case for infectious disease.
Here, by making each epoch last between one and two weeks, one can model the case in which the infectious period lasts for a constant amount of time.
This still results in an approximation, but appears to perform well in practice as was shown in Section \ref{s:algorithm}. \\

With the model defined for a single community, we now move on to describe the results of the generative model, which provides the link between the Gaussian shape of \eqref{eq:mixtures} with the notion of degree pruning. 
Our first result exactly characterizes this connection, as we show that in expectation, the number of infected individuals in the model will follow the shape of the Gaussian curve.

In Theorem \ref{thm:gaussian-connection}, we make use of $O$ notation with respect to $n$, in which we say, for two functions $f$ and $g$, $f(n) = O(g(n))$ if there exists an $m > 0$ and a value $n_0$ such that for all $n \geq n_0$, $f(n) \leq m g(n)$.
That is, we provide results which hold for each $n$, and indicate that our results become apparent as the size of the population becomes large.
Hence, our result relates to previous work which analyzes the SIR model in cases where the size of the susceptible population is large, which is often a paradigm used in understanding why initial stages of an epidemic process can result in exponential growth (cf. \cite{ma2020estimating}).
In standard SIR models, a key assumption is that decay of the susceptible population size results in diminishing rates of infection \citep{kermack1927contribution}.
Here, we instead assume that it is human behavior which ``flattens the curve'' as opposed to the depletion of susceptible individuals, as is made apparent by this assumption that the population size becomes arbitrarily large compared to the number of infections.

\begin{theorem} \label{thm:gaussian-connection}
Define the quantity $N(t) := |I(t)|$, which represents the number of infected individuals at time $t$.
Then,
\begin{equation} \label{eq:gaussian-expectation}
    \E[ N(t) ] = e^{\left(\frac{1}{2} \log \gamma\right) t^2 + \left( \log \frac{d\beta}{\sqrt{\gamma}}\right) t  } - O\left(\frac{t^2 e^{\log^2 (d \beta / \sqrt{\gamma} )/ \log (1/\gamma)}}{n}\right) \,,
\end{equation}
where we recall that $\gamma \in (0, 1)$ reflects the rate of degree pruning in the network, the parameter $d$ reflects the average degree of the graph in the absence of any degree pruning, and $\beta$ is the probability that infection travels across an edge in the graph.
\end{theorem}

Proof details can be found in Section \appendixgaussianconnectionproof of the Supplementary Material.
Theorem \ref{thm:gaussian-connection} is critical to the connection between mechanistic and non-mechanistic models of epidemics.
Farr's law has been used since the mid-19th century to estimate epidemics, and here we see that there is a generative explanation which gives the parameters of the Gaussian form a mechanistic interpretation.
Namely, the quadratic term in the exponent of the Gaussian represents the extent to which individuals are reacting to the progress of the virus, and the linear term in the exponent of the Gaussian represents a reproduction rate.
Such relationships will allow us to validate this interpretation using mobility data in Section \ref{ss:theory-mobility}.
This Theorem also represents a bridge between Gaussian parametric form and compartmental models, as the notion of degree pruning in discrete time corresponds to the affine form of $\beta - \gamma$ previously introduced in Section \ref{s:background}.
That is, by introducing a simple mechanism by which individuals can respond to the progress of an epidemic, the Gaussian form becomes apparent, and the parameters of the Gaussian can be interpreted as well.

Our statistical analysis also details concentration results regarding the behavior of the number of cases for each time $t$.

\begin{lemma}\label{lem:gaussian-concentration}
For the model above, again define $N(t) := |I(t)|$ as the number of cases in the model at time $t$. 
For any $\epsilon > 0$, and any $t$,
\begin{align*} 
& \P\left(N(t) \leq e^{\left(\frac{1}{2} \log \gamma\right) t^2 + \left( \log \frac{d\beta}{\sqrt{\gamma}} (1+\epsilon) \right) t} \right) \geq 1 - \sum_{s = 0}^{t-1} e^{-\min \{ \epsilon^2, \epsilon\} d \beta \gamma^s/4} \,.
\end{align*}
\end{lemma}

The proof of this claim can be found in Section \appendixgaussianconcentrationproof of the Supplementary Material.
Therefore, it is likely that the number of cases differs from its expectation by what is effectively a change in contact rate, since only the linear term of the exponential is changed.
For this particular claim, when $t$ tends towards infinity, the bound on the probability becomes vacuous as $\gamma^t$ tends towards 0 and hence each term in the summation tends towards 1.
However, for large enough $t$, the expectation of $N(t)$ tends towards 0 regardless due to the Gaussian shape, which is to say that a simple Markov inequality can be utilized instead for sufficiently large $t$.
This regime, in which expected case counts are low, is not the focus of this work as we are primarily concerned with regimes in which number of cases are sufficiently large.
In the following section, we show that the result of Lemma \ref{lem:gaussian-concentration} is indeed non-trivial, by utilizing the result to show that distinct Gaussian components are likely to be observed when the model is generalized to handle multiple communities.

Ultimately, these statistical results show that the equation of the form \eqref{eq:mixtures}, in the special case where $r=1$, tightly characterizes the outcome of this natural stochastic process which encodes human behavior into the traditional networked SIR framework.

\subsection{Justification of Mixture Structure}
\label{ss:theory-sbm}

The previous section provided a theoretical model which justified the form of the equation \eqref{eq:mixtures} in the case $r = 1$.
In this section, we provide the full generative model which justifies the general equation \eqref{eq:mixtures} for the case $r > 1$.
The primary difference between this model and the model of the previous section is the set of connections in the underlying graph.
Whereas previously the graph was assumed to be drawn from a model based on an Erd\H{o}s-R\'enyi graph with deferred randomness, here we assume that the underlying graph of connections is drawn from a stochastic block model, which allows for community structure to be encoded into the model (see Figure \ref{fig:sbm_viz}).

\begin{figure}
    \centering
    \includegraphics[width=0.7\textwidth]{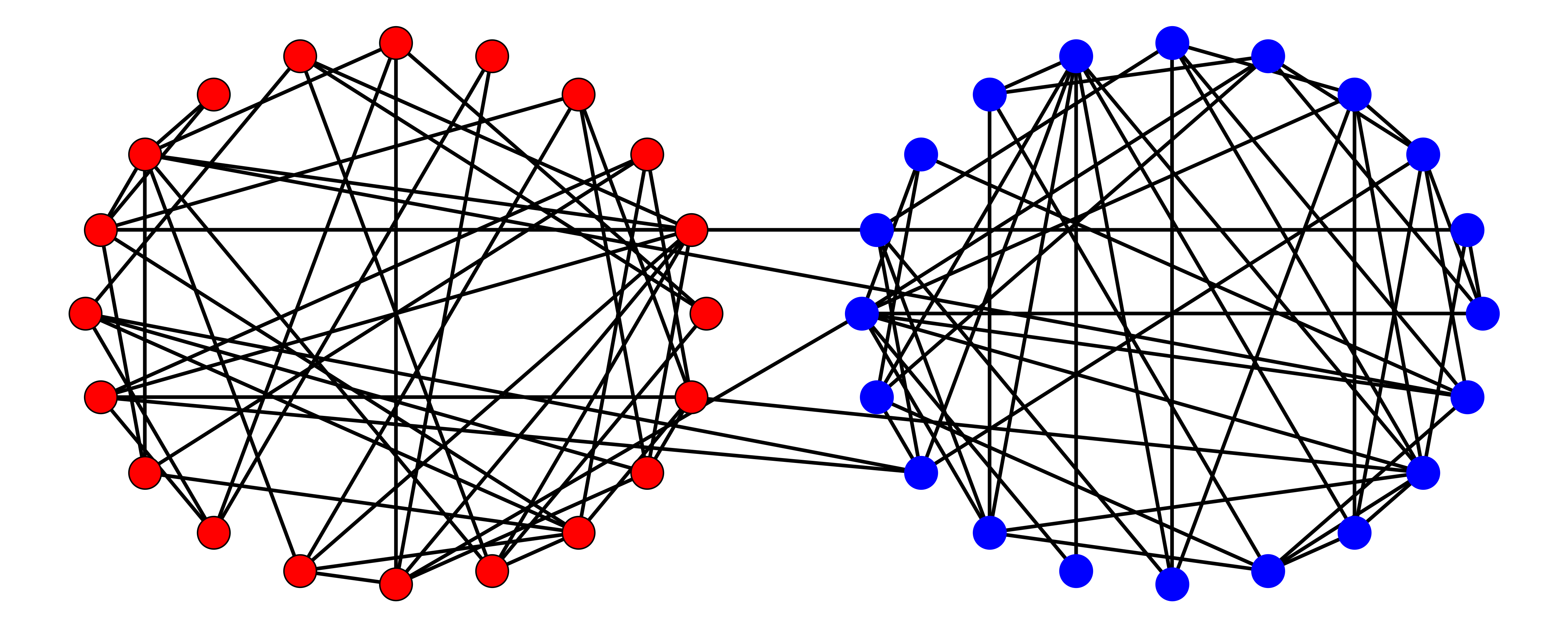}
    \caption{Example of a Stochastic Block Model with Two Communities}
    \label{fig:sbm_viz}
\end{figure}

With the introduction of this community structure, we see that each individual community will experience a Gaussian curve in expectation, stemming from the initial infected individual within each community.
The primary result of this section is to show that if the number of connections between each community is sufficiently small, then it is likely that the Gaussian curves from each community will be well-separated, resulting in observations according to \eqref{eq:mixtures} for $r > 1$.
We first present the details of the model and modification from the previous section, and then present our theoretical results. \\

\noindent \textbf{Description of Connections for Multiple Communities.\quad} 
In the model which accounts for community structure, we first decompose the vertex set $V$ into two disjoint sets $V_1$ and $V_2$, such that $V_1 \cup V_2 = V$, $V_1 \cap V_2 = \emptyset$.
To simplify notation, we will assume $|V_1| = |V_2| = n$, such that the total number of nodes in the graph is now $2n$.
We refer to each set $V_1$ and $V_2$ as communities, as they will represent groups of nodes with a high likelihood of knowing one another.
We will assume an initial condition that $I(0) = \{ v_0 \}$, where $v_0 \in V_1$ without loss of generality, $R(0) = \emptyset$ and $S(0) = V \setminus \{ v_0 \}$.

For this model, as before we assume a model of deferred randomness, where the edges of the graph are revealed according to the individuals who are infected.
However, for an infected individual in $V_1$, the probability of an edge forming with a susceptible individual in $V_1$ will be different than the probability of forming a connection with an individual in $V_2$.
Specifically, for given $v_i \in I(t)$ such that $v_i \in V_1$, and a node $v_s \in S(t)$, then the probability there is an edge between $v_i$ and $v_s$ at time $t$ is equal to $\frac{d_{in}}{n} \gamma^t$ if $v_s$ is also a member of $V_1$, and $\frac{d_{out}}{n} \gamma^t$ if $v_s$ is a member of $V_2$.
In the case where $d_{out} \ll d_{in}$, we see that there is community structure in the underlying graph.
For infection within $V_2$ itself, we will assume that degree pruning only begins in $V_2$ after a first node is infected in $V_2$.
That is, if we let $T$ represent the random variable denoting the time at which the first node in $V_2$ is infected, then for nodes $v_i \in I(t) \cap V_2$ and $v_s \cap V_2$, the probability $v_i$ would infect $v_s$ would be equal to $\frac{d_{in}}{n} \gamma^{t-T}$.
This feature of the model reflects that communities are not expected to react to the virus until it becomes an immediate threat to the individuals within the community.
The remaining epidemiological dynamics, including the infection rate $\beta$ and the recovery rate, remain the same as in the previous model. \\

The primary result of this section shows that, when the stochastic block model is parameterized such that the number of edges between communities is sufficiently small, the Gaussian curves from each community will become temporally well-separated, resulting in a mixture for the observed time series of infection.
The theorem is stated as follows.
\begin{theorem}\label{thm:time_til_next}
Define the following quantity, which represents the time at which the expected number of cases in community 1 is maximized:
\[ C_1 (d_{in},\beta, \gamma) = - \log \left(\frac{d_{in} \beta}{\sqrt{\gamma}} \right) / \log \gamma \,. \]
Suppose that, for some parameter $\delta \in (0, 1/2)$, the following conditions hold:
\begin{enumerate}
    \item The time at which expected number of cases in Community 1 is maximized has the bound
    \begin{equation}
        C_1(d_{in},\beta, \gamma)  < \delta e^{5\sqrt{\gamma}} + \log 20 / \log \frac{1}{\gamma}  \label{eq:sep-condition-1} \,.
    \end{equation}
    \item The expected number of connections between communities are not so large, specifically
    \begin{equation}
        d_{out}  \leq \frac{1}{2 \beta \sum_{s = 0}^{\lfloor - \left( \log \left( \frac{d_{in}\beta}{20\sqrt{\gamma}}\right) \right) / \log \gamma \rfloor - 1} (2d_{in}\beta)^s \gamma^{s(s+1)/2}}\log \left(\frac{1}{1 - \delta}\right) \label{eq:sep-condition-2}\,.
    \end{equation}
    \item The number of individuals in each community is sufficiently large,
    \begin{equation} \label{eq:sep-condition-3}
        n \geq 2 \beta d_{out} \,.
    \end{equation}
\end{enumerate}
If the above conditions hold, then the random variable $T$ which represents the time at which the first infection occurs in community 2 satisfies
\begin{equation}
    \P\left(T > C_1(d_{in},\beta, \gamma) - \log 20 / \log \frac{1}{\gamma} \right) \geq 1 - 2 \delta \,.
\end{equation}
\end{theorem}
The proof of this claim can be found in Section \appendixspreadtimeproof of the Supplementary Material.
Notably, the requirements of the theorem which result in temporal well-separation do not depend as much on the difference between $d_{out}$ and $d_{in}$ as much as it depends on the magnitude of each.
Often in the literature, the identification of a stochastic block model depends on the difference $d_{in} - d_{out}$.

Theorem \ref{thm:time_til_next} indicates that, under certain model assumptions, with high probability the components of the mixture will be temporally well-separated, which is similar to the condition of identification required in Theorem \ref{thm:parameter_estimation}.
From this theorem, we see that we can expect the observed components to be temporally well separated if there are few connections between communities and the outbreak in community 1 is sufficiently small.
This observation highlights the different ways in which policy makers can prevent an outbreak from spreading between communities, as they can either focus on mitigating spread within their own community or ask members of a community with many infections to reduce their ties to other communities. \\

\noindent \textbf{Example.} \quad Suppose that in the context of Theorem \ref{thm:time_til_next}, we set $\delta = 0.05$, and choose the following parameters: $\gamma = 0.9$, which represents that individuals remove approximately $10\%$ of their contacts each epoch, $d_{in} = 55.2$, which represents expected contacts within the community in each epoch, $\beta = 0.5$, which represents the probability of infection given that a contact occurs, and $d_{out} = 2\times 10^{-5}$, which represents the expected number of contacts outside of the community.
Then, $C_1(d_{in}, \beta, \gamma) \approx 32$, and Theorem \ref{thm:time_til_next} shows that for any $n$, with probability at least 90\%, the time at which the first infection occurs in community 2 will be after at least $3.5$ epochs, which occurs after about 10\% of the time that it takes for cases in community 1 to reach their peak.
In such a case, it would be likely to observe distinct mixtures. \\

The connection between the temporal well-separation in the generative model and the temporal separation condition required for learning can in fact be made precise, as shown in the following Corollary.

\begin{corollary} \label{corr:separation-learning}
Recall the definition of $C_1(d_{in}, \beta, \gamma)$, as
\[ C_1 (d_{in},\beta, \gamma) = - \log \left(\frac{d_{in} \beta}{\sqrt{\gamma}} \right) / \log \gamma \,, \]
and define the following random variable
\[ C_2^{d_{in}, \beta, \gamma} = T - \log \left(\frac{d_{in} \beta}{\sqrt{\gamma}} \right) / \log \gamma = T + C_1 (d_{in},\beta, \gamma) \,, \]
where $T$ represents the time at which the first infection in Community 2 occurs. $C_2^{d_{in}, \beta, \gamma}$ then represents the time at which the expected time series of infection in community 2 would be maximized.
Then, under the conditions of Theorem \ref{thm:time_til_next}, with probability at least $1 - 2 \delta$, the following statement holds:
\[ |C_1(d_{in},\beta, \gamma) - C_2^{d_{in},\beta, \gamma}| \geq 2 \sqrt{\frac{1}{a} \log \frac{M}{\epsilon} } \,,\]
for parameters
\begin{align*}
    a &= \frac{1}{2} \log \frac{1}{\gamma} \,, \\
    M &= e^{\frac{1}{2} \left(\log \left( \frac{d_{in}\beta}{\sqrt{\gamma}}\right)^2 /  \log \frac{1}{\gamma} \right)} \,, \\
    \epsilon &\geq \frac{3}{8} M - \frac{\log 20}{4} C_1(d_{in},\beta, \gamma) - \frac{(\log 20)^2}{8} \,.
\end{align*}
\end{corollary}
\begin{proof}
The Corollary follows directly from Theorem \ref{thm:time_til_next}, as we see that $|C_1(d_{in},\beta, \gamma) - C_2^{d_{in},\beta, \gamma}| = T$ since $T \geq 0$.
Hence, with probability at least $1 - 2 \delta$,
\[ |C_1(d_{in},\beta, \gamma) - C_2^{d_{in},\beta, \gamma}| \geq C_1(d_{in},\beta, \gamma) - \log 20 / \log \frac{1}{\gamma} \,. \]
Solving for $\epsilon$ in the inequality
\[ C_1(d_{in},\beta, \gamma) - \log 20 / \log \frac{1}{\gamma} \geq 2 \sqrt{\frac{1}{a} \log \frac{M}{\epsilon} } \,, \]
with $a$ and $M$ defined above then yields the result.
\end{proof}

Corollary \ref{corr:separation-learning} highlights a temporal separation between the peaks of the expected time series, and makes explicit the relationship between the generative model presented here and the learning algorithm of Section \ref{s:algorithm}.
Notably, the $\epsilon$ obtained in the Corollary is nearly an order of magnitude lower than the value $M$, which is desirable for the learning algorithm.
We note that in Corollary \ref{corr:separation-learning}, $a$ is defined to be the quadratic coefficient in the expected time series and that the quantity $M$ represents a maximum number of cases in the expected time series of infection for each community, specifically representing the maximum number of infections in $V_1$ due only to other infections which originated in $V_1$.
We choose these particular expressions for clarity of presentation, as it could be the case that the maximum number of infections in $V_1$ can become larger due to the stochastic nature of the process within $V_1$ or due to additional cases in $V_1$ which can be traced to infections from individuals in $V_2$, after individuals in $V_2$ receive infection.
Rather, the purpose of Corollary \ref{corr:separation-learning} is to highlight the connection between temporal well-separation in this model and the previous condition required for learning.

Taken together, these results provide a statistical characterization of a reasonable stochastic model, and motivate the use of the function class \eqref{eq:mixtures}.
Of course, this is not the only possible model which results in observations which resemble a mixture of Gaussian curves.
Rather, the above model provides but one formalization by which case counts of the form in \eqref{eq:mixtures} can arise, and we prefer this model for its tractability in analysis. 
In particular, the \eqref{eq:mixtures} may also arise from an SIR-model with a particular time-varying reproductive rate as we will show in Section \ref{ss:closed-loop}, or in other network-based models with a particular degree distribution which emulates the degree pruning parameter above. 
Even situations in which the spread has spatial heterogeneity can be captured, so long as the simultaneous outbreaks have similar features resulting in global observations of the Gaussian curve, and the temporal separation of Gaussian curves may reflect different waves of the pandemic \citep{epstein2008coupled}.

\subsection{Empirical Validation of the Networked SIR Interpretation}
\label{ss:theory-mobility}
While our model is not the only one which explains the function class \eqref{eq:mixtures}, we do provide empirical evidence that the data observed during the COVID-19 pandemic is consistent with the simple interpretation provided above. 
Namely, we are able to show that ``degree-pruning'' correlates with observed mobility data.
To do so, we compare degree pruning parameters learned from the data to mobility data taken from Google and SafeGraph. 
Our analysis indicates that the intepretation of the quadratic coefficients in each county is consistent with the notion of degree pruning.

\begin{figure}
    \centering
    \includegraphics[width=0.8\textwidth]{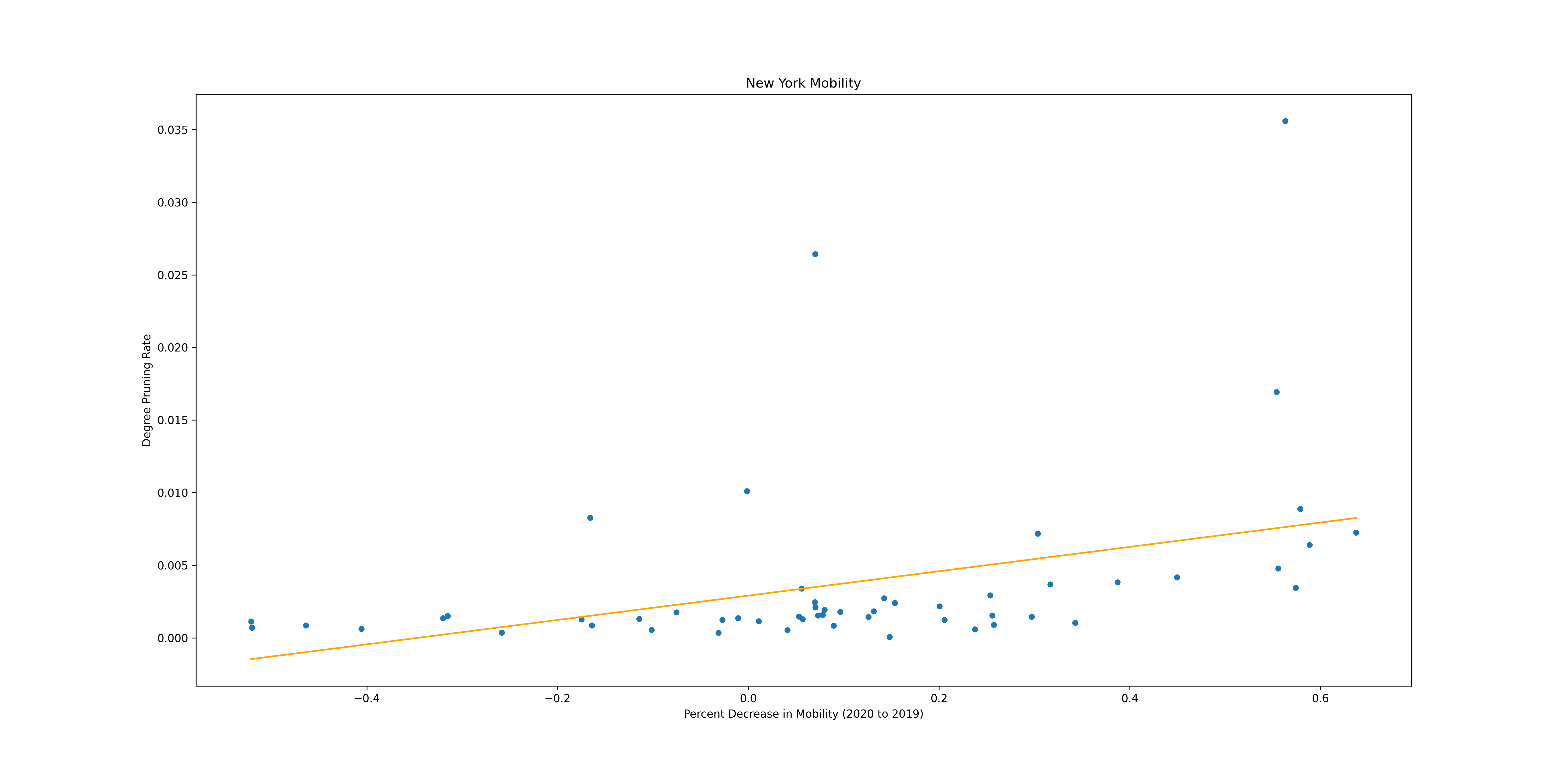}
    \caption{Degree pruning parameters compared to SafeGraph mobility data for counties in New York.
    There is a clear positive relationship between the degree pruning rate and percent decrease in mobility, indicating that for larger decreases in mobility, higher degree pruning rates are observed as expected.}
    \label{fig:ny}
\end{figure}

We use two sets of mobility data in our analysis. 
The first, from Google, provides a time series showing percent reduction in six different types of mobility compared to a baseline measured in February, before the United States government implemented any lockdown policies. 
The six types of mobility are \textsc{retail and recreation}, \textsc{groceries and pharmacy}, \textsc{parks}, \textsc{transit stations}, \textsc{workplaces}, and \textsc{residential}. The data is provided at county, state, and national granularity.

The second dataset, provided by SafeGraph, tracks the census block of particular devices, and results are aggregated daily. 
We process the data to measure the amount of time, on average, that a device spends outside of its census block. 
Since census blocks are typically small compared to the size of counties, the data provides another metric by which we may measure the extent to which individuals limited their mobility in response to the pandemic.

For validation of the interpretation of parameters, we use an estimate of degree pruning parameters which are the output of Algorithm \ref{alg:initialization}.
This also corresponds to an direct estimate of the degree pruning parameters where degree pruning parameters are time-varying. 
That is, if we assume time series have the form
\begin{equation}
N(t) = \left( (d \beta)^t \prod_{s = 1}^t \prod_{\tau = 1}^s \gamma_\tau \right) [1 + \eta_t]\,,
\end{equation}
where $t$ is again a discrete time index and $\eta_t$ represents bounded noise, then $S(t)$ defined in \eqref{eq:st} provides an estimate of $\log \gamma_{t+1}$. 
Note that, if $\gamma_t$ is a constant with respect to $t$, then the above equation is equivalent to \eqref{eq:mixtures} with $r = 1$.
Hence, in this time-varying model, $\frac{1}{t_1 - t_0 + 1} \sum_{\tau = t_0}^{t_1} S(\tau) $ represents an estimate of the log of the geometric mean of the degree pruning parameters of the region from time $t_0$ to $t_1$.

We compute this value for $t_0$ corresponding to March 1st, 2020, and $t_1$ corresponding to May 20th, 2020, when the national mask mandate in the United States was implemented, as in this period of time we would expect mobility to be the best proxy for degree pruning. 
The comparison between SafeGraph data and mobility for counties in New York is shown in Figure \ref{fig:ny}, and a comparison of state-level mobility data to degree pruning parameters in the same time frame for both SafeGraph and Google Mobility Data is shown in Table \ref{tab:google_state}.
In Figure \ref{fig:ny}, there is a clear positive relationship between degree pruning and mobility reduction that can be observed at the county level.

Moreover, among each state in the US, at a national level there are clear expected relationships between mobility and degree pruning parameters.
Change in mobility in terms of time spent in \textsc{retail and recreation}, \textsc{transit stations}, and \textsc{workplaces}, as well as mobility as measured by the average amount of time a device spends outside of its census block, appear to correlate negatively with degree pruning parameters, which is to say that as mobility decreases in these measures, degree pruning rates increase, as expected.
This is validated with correlation coefficients ranging from $-0.248$ to $-0.430$, each of which is statistically significant.
Moreover, mobility as measured by change in time spent in \textsc{Residential} places is positively correlated, implying that as individuals increase their time spent at home, degree pruning rates increase, as expected.

\begin{table}[t]
    \centering
    \begin{tabular}{rrl} \toprule
    Mobility Data & Correlation & $p$-value\\ \midrule
    \textsc{retail and recreation} & -0.275 &  0.054* \\
\textsc{groceries and pharmacy}  & -0.219 &  0.126 \\
\textsc{parks}                 &  0.015 &  0.916 \\
\textsc{transit stations}      & -0.248 &  0.083* \\
\textsc{workplaces}            & -0.299 &  0.035** \\
\textsc{residential}            &  0.345 &  0.014** \\
\textsc{safegraph}                                          & -0.430 &  0.002*** \\\bottomrule
    \end{tabular}
    \caption{Comparison of learned degree pruning parameters to types of mobility data. In a linear regression, these mobility features predict the degree pruning rate with an $r^2$ value of $0.220$.
    For most there is a statistically significant negative relationship between amount of mobility reduction and degree pruning.
    For \textsc{Residential} mobility, we see a positive relationship, indicating that areas in which individuals stayed at home saw more degree pruning in case counts, as expected.\\ ${}^{*}p < 0.10, {}^{**}p < 0.05, {}^{***}p < 0.01$}
    \label{tab:google_state}
\end{table}

\section{Closed-Loop Interpretation of Gaussian Curves}
\label{ss:closed-loop}
While the model above provides clear insight into the structure of the mixture model \eqref{eq:mixtures}, which can be validated empirically, a drawback of the approach is that it is open-loop.
That is, the parameter $\gamma$ which modulates the extent to which individuals remove ties is assumed to be constant to generate the Gaussian form, and the model does not explicitly allow for individuals to react to the state of the disease.
In this section, we show that the Gaussian shape can in fact arise as the outcome of a closed-loop system, in which individuals explicitly react to the spread of disease.
Specifically, by using an approximation of an SIR model, the Gaussian curve can be justified by assuming individuals react in response to the observed number of infections in an epidemic.

Two key features of the network SIR model above are that the number of individuals who are susceptible remains large compared to the infected population, i.e. that $n$ is large, and that the assumption of degree pruning results in contact rates and removal rates are time-varying. 
With these assumptions, we can also rewrite the number of infections in the continuous time Susceptible-Infected-Recovered compartmental from Figure \ref{fig:sir} of Section \ref{ss:mechanistic} as
\begin{align}
\label{eq:infected_dynamics}
    \dv{I}{t} &= \bigg(\beta (t) - \gamma (t) \bigg) I(t) \,. 
\end{align}

Denoting $\alpha(t) := \beta (t) - \gamma (t)$, we then see that, in order to ensure that $I(t)$ follows the shape of a bell curve, $\beta (t) - \gamma (t)$ must follow a specific form, formalized in the following Theorem.

\begin{theorem}[Gaussian Curve as Closed Loop Control]
\label{thm:closed_loop}
In the setting of \eqref{eq:infected_dynamics}, suppose a population can guarantee the following form for $\alpha(t) :=  \beta (t) - \gamma (t)$
\[ \alpha(t) = \begin{cases}
\sqrt{b^2 - 4a(c - \log I(t))} & \text{if $\max_{t' \leq t} \log I(t') < c + \frac{b^2}{4a}$} \\
- \sqrt{b^2 - 4a(c - \log I(t))}  & \text{otherwise}
\end{cases} \,. \]
Then, if $I(0) = e^c$, \[ I(t) = e^{- at^2 + bt + c} \,.\]
\end{theorem}

\begin{proof}
The proof follows from first plugging in the form of $I(t)$ into the equation of $\alpha(t)$,

\[ \alpha(t) = \begin{cases}
\sqrt{b^2 - 4a(at^2 - bt)} & \text{if $\max_{t' \leq t} \log I(t') < c + \frac{b^2}{4a}$} \\
- \sqrt{b^2 - 4a(at^2 - bt)}  & \text{otherwise}
\end{cases} \,. \]
Identifying the square and performing casework yields
\[ \alpha(t) = - 2 at + b \,,\]
in both cases, which directly implies the claim by substitution into the Equation \eqref{eq:infected_dynamics}. Hence, since the solution to the differential equation is unique given the initial condition, it must be the case that
\[ I(t) = e^{- at^2 + bt + c} \,,\]
as desired.
\end{proof}

Theorem \ref{thm:closed_loop} shows that the evolution of the number of infections in a particular community can be written in terms of the number of infections itself, implying that community reaction to the progress of an epidemic may be stated in terms of the infection prevalence. 
Further, as a policy maker, one may try to adjust the value of $\alpha(t)$ to ensure that the maximum number of cases does not exceed some value $I_{max}$ by ensuring that the rate $\alpha(t)$ evolves according to a rule for which the constants $a, b, $ and $c$ satisfy $c + \frac{b^2}{4a} \leq \log I_{max}$.

This indicates another possible way in which the Gaussian form of \eqref{eq:mixtures} can be interpreted through the lens of policy decisions.
Namely, if observed case counts follow the Gaussian trend, then a possible explanation for this is that policy decisions are being formed as a function of the state of the epidemic, as opposed to through a constant pruning of edges.



\section{Conclusions}
\label{s:conclusions}
We provide a simple, non-mechanistic model for forecasting cases and fatalities in an epidemic which, upon further inspection, bridges the two approaches to epidemic forecasting while retaining benefits of both.
By assuming that observed case counts follow a functional form represented by a sum of Gaussian curves, we benefit from the reduced form structure of the model because we are able to perform statistical inference to measure the parameters of each Gaussian.
The approach fits the data surprisingly well.
Moreover, we show that a generative model can yield observed case counts of the form \eqref{eq:mixtures}, which provides a benefit from mechanistic models in that the learned parameters have an interpretation.
The interpretation is validated using mobility data from Google and SafeGraph and suggests the a path to unify the two broad approaches to epidemic modeling.

A key observation from the mechanistic perspective of the model is that, in order to attain the quadratic term of the Gaussian time series, individuals in a community must be continually severing ties with neighbors.
There are several possible strategies by which a policy maker or general population may achieve a specific level of ``degree pruning'' in order to attain a smaller level of infection over time. 
These strategies include, but are not limited to, the use of social distancing, masks, testing, and vaccinations in order to remove links between individuals in the community faced with an epidemic.

Social distancing has been one of the primary ways in which communities have implemented degree pruning throughout COVID-19 pandemic. 
By initiating stay-at-home orders and posting signs in public areas that enforce individuals to stay apart from one another, there are fewer links between individuals in the communities by which the virus may spread. Masks have had a similar effect on the extent of degree pruning. As the pandemic has progressed, mask mandates have become mandatory in many areas, and have reduced the total number of COVID-19 cases \citep{chernozhukov2020causal}. Testing provides another means of degree pruning, by limiting the number of contacts an individual has once they are infectious. By allowing individuals to know that they are infected with the virus, the individual may then self isolate and limit the number of contacts they will have with susceptible individuals. 
In particular, by \emph{increasing} testing rates over time, the rate of degree pruning can be changed. Finally, vaccinations provide another simple mechanism by which degree pruning may be achieved. 
For example, vaccinating a constant fraction of the susceptible population at each time step has the exact effect of degree pruning, as edges in the graph will be deleted proportionally to the rate of vaccination. Ultimately, a combination of the above strategies, with increased utilization over time, may allow policy makers to better understand and combat the spread of infectious disease, when considering the evolution of an epidemic through the lens of degree pruning.

The overarching goal of this work is to provide progress towards robust, data-driven control of epidemics for general outbreaks of infectious disease.
This will require refined statistical algorithms for estimating the state of an epidemic subject to noisy observations, as well as an understanding of the mechanisms that policy makers can utilize to inhibit spread of the disease.
Hence, further research can focus on determining optimal ways to estimate epidemic state from noisy data due to delays and testing variance, as well as understanding the impact that policy levers such as lockdowns and masks have on the spread of infectious disease.
By unifying the disparate approaches to epidemic forecasting, we hope to take a step towards reaping the benefits of both approaches in the design of public health policy for infectious disease.




\begin{appendix}

\section{Proof of Theorem \ref{thm:parameter_estimation}}
\label{a:alg-proof}
We first provide a proof of Theorem \algtheorem as stated, and then discuss the extension to the case where $a_1 \neq a_2$.

The proof of the theorem has three parts. 
First, we show the following Lemma, which shows that Assumptions 1-3 of the Theorem result in a low overlap condition such that when one component of the mixture comprises a majority of cases, the other component must have a small size.

\begin{lemma} \label{lem:sufficient-assumption-1}
Fix an $\epsilon > 0$, and suppose that the parameters of the underlying model satisfy the following three conditions.
\begin{enumerate}
    \item $M_k \leq M\,, \quad k = 1, 2$
    \item $a_k \geq a\,, \quad k = 1, 2$
    \item $|C_1 - C_2| \geq 2 \sqrt{\frac{1}{a} \log \frac{M}{\epsilon}} \,.$
\end{enumerate}
Then, the parameters satisfy the following conditions.
\begin{enumerate}
    \item[{1. [Dominance of Component 1]}] If $t$ satisfies 
\[ M_1 e^{-a_1(t-C_1)^2} \geq M_2 e^{-a_2(t-C_2)^2} \,,\]
then 
\[ M_2 e^{-a_2(t-C_2)^2} \leq \epsilon \,. \]
\item[{2. [Dominance of Component 2]}] Otherwise, if $t$ satisfies 
\[ M_1 e^{-a_1(t-C_1)^2} \leq M_2 e^{-a_2(t-C_2)^2} \,,\]
then 
\[ M_1 e^{-a_1(t-C_1)^2} \leq \epsilon \,. \]
\end{enumerate}
\end{lemma}
\begin{proof}
We prove that the three conditions of the lemma imply the Dominance of Component 1, and note that the proof for the Dominance of Component 2 follows symmetrically. 
We proceed by contradiction. 
Suppose that the three conditions hold, and that $t$ satisfies
\begin{equation} \label{eq:t-cond}
M_1 e^{-a_1(t-C_1)^2} \geq M_2 e^{-a_2(t-C_2)^2} \,,
\end{equation}
but
\begin{equation} \label{eq:lemma-assumption-1-contradiction}
M_2 e^{-a_2(t-C_2)^2} > \epsilon \,. 
\end{equation}
Since $M_2 e^{-a_2(t-C_2)^2} > \epsilon$, $t$ must satisfy
\begin{equation} \label{eq:t-c2}
    |t-C_2| < \sqrt{\frac{1}{a} \log \frac{M}{\epsilon}} \,,
\end{equation}
as conditions 1 and 2 of the lemma would imply $M e^{-a(t-C_2)^2} > \epsilon$, and \eqref{eq:t-c2} follows from rearranging this inequality.
Moreover, since we assume both \eqref{eq:t-cond} and \eqref{eq:lemma-assumption-1-contradiction}, we must have $M_1 e^{-a_1(t-C_1)^2} > \epsilon$, which similarly implies that $t$ satisfies
\begin{equation} \label{eq:t-c1}
    |t-C_1| < \sqrt{\frac{1}{a} \log \frac{M}{\epsilon}} \,.
\end{equation}
However, equations \eqref{eq:t-c2} and \eqref{eq:t-c1} imply that for any $t$ which satisfies \eqref{eq:t-cond},
\[ |C_1 - C_2| \leq |t - C_1| + |t - C_2| < 2 \sqrt{\frac{1}{a} \log \frac{M}{\epsilon}} \,,\]
by the triangle inequality. This is a contradiction to condition 3 of the Lemma.
Hence, it must be the case that whenever $t$ satisfies \eqref{eq:t-cond}, 
\[ M_2 e^{-a_2(t-C_2)^2} \leq \epsilon \,,\]
proving the claim.
\end{proof}

The next claim shows that, because the noise is sufficiently small, each $\hat{C}_k$ occurs in a location in which component $k$ is dominant, for $k = 1, 2$. 

\begin{lemma} \label{lem:midpoint}
Suppose $a_1 = a_2$ and the conditions of Theorem \algtheorem hold, i.e. the following properties are satisfied for some $0 < \epsilon  < \min\{M_1, M_2\} / 5$, and $a, M > 0$.
\begin{enumerate}
        \item $M_k \leq M\,, \quad k = 1, 2$
    \item $a_k \geq a\,, \quad k = 1, 2$
    \item $|C_1 - C_2| \geq 2 \sqrt{\frac{1}{a} \log \frac{M}{\epsilon}} \,.$
    \item $\delta \leq \delta^*(a_1, C_1, C_2, M_1, M_2)$ as defined as in \eqref{eq:delta-star}.
    \item $C_1, C_2 \in [0, T]$
\end{enumerate}
Let
\begin{align*}
    t_m &= \argmax_{1 \leq t \leq T} S(t) \,, \\
    \hat{C}_1 &= \argmax_{0 \leq t \leq t_m} N(t) \,,\\
    \hat{C}_2 &= \argmax_{t_m \leq t \leq T} N(t) \,,
\end{align*}
which represent the estimates of $\hat{C}_k$ in Line 2 of Algorithm \alginitialization. 
These values will satisfy the following three properties:
\begin{align}
    &C_1 \leq t_m \leq C_2 \label{eq:tm_midpoint}\\
    &M_1 e^{-a_1(\hat{C}_1 - C_1)^2} \geq M_2 e^{-a_2(\hat{C}_1 - C_2)^2} \label{eq:c1-hat-ok} \\
    &M_1 e^{-a_1(\hat{C}_2 - C_1)^2} \leq M_2 e^{-a_2(\hat{C}_2 - C_2)^2} \label{eq:c2-hat-ok}
\end{align}
\end{lemma}
\begin{proof}
To show the proof of this Lemma, we introduce the following object which measures the ratio of cases between the two communities:
\begin{equation} \label{eq:rt}
    r(t) = \frac{M_2 e^{-a_2(t-C_2)^2}}{M_1 e^{-a_1(t-C_1)^2}} \,.
\end{equation}
Since we assume $a_1 = a_2$, this simplifies to
\begin{equation}
    r(t) = \frac{M_2}{M_1} e^{2a_1 (C_2 - C_1) t + a_1 (C_1^2 - C_2^2)} \,.
\end{equation}
Since we assume without loss of generality that $C_1 < C_2$, we can observe that $r(t)$ is increasing in $t$, and starts at $r(0) < 1$ and ends at $r(T) > 1$.
This holds, as $r(C_1) \leq \frac{\epsilon}{M_1 - \epsilon} < \frac{1}{4}$, and similarly, $r(C_2) \geq \frac{M_2 - \epsilon}{\epsilon} > 4$.
Noting $C_1, C_2 \in [0, T]$ allows us to conclude $r(0) < 1$ and $r(T) > 1$.
Since $r(t)$ is continuous, by the intermediate value theorem there must be some point $t^*$ for which $r(t^*) = 1$.
The primary work of this lemma is to show that if the observation noise is small, then $t_m$ as defined above is close to $t^*$ in the sense that the case ratio of $t_m$ is close to 1.

We first note the following relationship between $S(t)$ and $r(t)$:
\begin{equation} \label{eq:st-rt}
    S(t) = - 2 a_1 + \log \left( 1 + \frac{r(t)}{(1 + r(t))^2} \left(e^{2a_1(C_1 - C_2)} + e^{-2a_1(C_1 - C_2)}  - 2 \right) \right) + \log \frac{(1 + \eta_{t+1})(1+\eta_{t-1})}{(1 + \eta_t)^2} \,.
\end{equation}
The proof of Equation \eqref{eq:st-rt} is as follows: First, by noting 
\begin{align*}
    \log N(t) &= \log \left( M_1 e^{-a_1(t-C_1)^2} + M_2 e^{-a_2(t-C_2)^2} \right) \\
    &= \log M_1 e^{-a_1(t-C_1)^2}  + \log \left(1 + r(t) \right) \,,
\end{align*}
we can show algebraically that
\begin{align}
    S(t) &= \log \frac{N(t+1)}{N(t)} - \log \frac{N(t)}{N(t-1)} \nonumber \\
    &= \log M_1 e^{-a_1(t + 1 -C_1)^2}  + \log \left(1 + r(t + 1) \right) \nonumber\\
    &\quad - 2 \left(\log M_1 e^{-a_1(t-C_1)^2}  + \log \left(1 + r(t) \right) \right) \nonumber\\
    &\quad + \log M_1 e^{-a_1(t - 1 -C_1)^2}  + \log \left(1 + r(t - 1) \right) \nonumber\\
    &\quad + \log \frac{(1 + \eta_{t+1})(1+\eta_{t-1})}{(1 + \eta_t)^2}\nonumber\\
    &= \log M_1 e^{-a_1(t -C_1)^2 - 2a_1(t-C_1) - a_1} - 2 \log M_1 e^{-a_1(t-C_1)^2} + \log M_1 e^{-a_1(t -C_1)^2 + 2a_1 (t-C_1) - a_1}   \nonumber\\
    &\quad+ \log \left(1 + r(t + 1) \right) - 2 \log \left(1 + r(t) \right) + \log \left(1 + r(t-1) \right) + \log \frac{(1 + \eta_{t+1})(1+\eta_{t-1})}{(1 + \eta_t)^2} \nonumber\\
    &=  - 2 a_1 +  \log \left(1 + r(t + 1) \right) - 2 \log \left(1 + r(t) \right) + \log \left(1 + r(t-1) \right) + \log \frac{(1 + \eta_{t+1})(1+\eta_{t-1})}{(1 + \eta_t)^2} \nonumber\\
    &=  - 2 a_1 +  \log \left(1 + r(t) \frac{e^{-2a_2(t-C_2) - a_2}}{e^{-2a_1(t-C_1) - a_1}} \right) - 2 \log \left(1 + r(t) \right) + \log \left(1 + r(t) \frac{e^{2a_2(t-C_2) - a_2}}{e^{2a_1(t-C_1) - a_1}}\right)\nonumber\\
    &\quad + \log \frac{(1 + \eta_{t+1})(1+\eta_{t-1})}{(1 + \eta_t)^2} \nonumber\\
    &=  - 2 a_1 +  \log \left(1 + r(t) \frac{e^{-2a_2(t-C_2) - a_2}}{e^{-2a_1(t-C_1) - a_1}} + r(t) \frac{e^{2a_2(t-C_2) - a_2}}{e^{2a_1(t-C_1) - a_1}} + r(t)^2 \frac{e^{-2a_2}}{e^{-2a_1}}\right) - 2 \log \left(1 + r(t) \right)\nonumber\\
    &\quad + \log \frac{(1 + \eta_{t+1})(1+\eta_{t-1})}{(1 + \eta_t)^2} \nonumber\\
    &= - 2 a_1 + \log \left( \left( \frac{1 + (r(t)e^{a_1 - a_2})^2}{(1 + r(t))^2}\right) + \frac{r(t)}{(1 + r(t))^2} \left(\frac{e^{-2a_1(t-C_1)}}{e^{-2a_2(t-C_2)}} + \frac{e^{-2a_2(t-C_2)}}{e^{-2a_1(t-C_1)}} \right) e^{a_1 - a_2}\right) \nonumber\\
    &\quad + \log \frac{(1 + \eta_{t+1})(1+\eta_{t-1})}{(1 + \eta_t)^2} \label{eq:st-rt-general}\,.
\end{align}
In the case where $a_1 = a_2$, this simplifies further yielding the result \eqref{eq:st-rt}.

Since at $t = t^*$, we have $r(t^*) = 1$, and because the noise $\eta_t \in [-\delta, \delta]$ for all $t$, we then can place an lower bound on $S(t^*)$ as follows.
\begin{align*}
    S(t^*) &= - 2 a_1 + \log \left( 1 + \frac{r(t^*)}{(1 + r(t^*))^2} \left(e^{2a_1(C_1 - C_2)} + e^{-2a_1(C_1 - C_2)}  - 2 \right) \right) + \log \frac{(1 + \eta_{t^*+1})(1+\eta_{t^*-1})}{(1 + \eta_t^*)^2} && \eqref{eq:st-rt} \\
    &= - 2 a_1 + \log \left( 1 + \frac{1}{(1 + 1)^2} \left(e^{2a_1(C_1 - C_2)} + e^{-2a_1(C_1 - C_2)}  - 2 \right) \right) + \log \frac{(1 + \eta_{t^*+1})(1+\eta_{t^*-1})}{(1 + \eta_t^*)^2} \\
    & \geq - 2 a_1 + \log \left( 1 + \frac{1}{4} \left(e^{2a_1(C_1 - C_2)} + e^{-2a_1(C_1 - C_2)}  - 2 \right) \right) - 2 \log \frac{1 + \delta}{1 - \delta}
\end{align*}

We next wish to show that for any $t$ such that either $r(t) \leq \frac{\epsilon}{\min \{ M_1, M_2 \} \left(\frac{1-\delta}{1+\delta}\right) - \epsilon}$ or $r(t) \geq \frac{\min \{ M_1, M_2 \}\left(\frac{1-\delta}{1+\delta}\right) - \epsilon}{\epsilon}$, there is a non-trivial lower bound on $S(t)$.
To do so, we note the following fact: For any value $\tau < 1$, if either  $r(t) \leq \tau $ either or $r(t) \geq \frac{1}{\tau}$, then $\frac{r(t)}{(1+r(t))^2} \leq \tau$.

This then allows us to show that, if either $r(t) \leq \frac{\epsilon}{\min \{ M_1, M_2 \} \left(\frac{1-\delta}{1+\delta}\right) - \epsilon}$ or $r(t) \geq \frac{\min \{ M_1, M_2 \}\left(\frac{1-\delta}{1+\delta}\right) - \epsilon}{\epsilon}$, then
\begin{align*}
    S(t) &= - 2 a_1 + \log \left( 1 + \frac{r(t)}{(1 + r(t))^2} \left(e^{2a_1(C_1 - C_2)} + e^{-2a_1(C_1 - C_2)}  - 2 \right) \right) + \log \frac{(1 + \eta_{t+1})(1+\eta_{t-1})}{(1 + \eta_t)^2} && \eqref{eq:st-rt}\\
    &\leq - 2 a_1 + \log \left( 1 + \frac{\epsilon}{\min \{ M_1, M_2 \} \left(\frac{1-\delta}{1+\delta}\right) - \epsilon} \left(e^{2a_1(C_1 - C_2)} + e^{-2a_1(C_1 - C_2)}  - 2 \right) \right) + \log \frac{(1 + \eta_{t+1})(1+\eta_{t-1})}{(1 + \eta_t)^2} \\
    &\leq - 2 a_1 + \log \left( 1 + \frac{\epsilon}{\min \{ M_1, M_2 \} \left(\frac{1-\delta}{1+\delta}\right) - \epsilon} \left(e^{2a_1(C_1 - C_2)} + e^{-2a_1(C_1 - C_2)}  - 2 \right) \right) + 2 \log \frac{1 + \delta}{1 - \delta} \,,
\end{align*}
where the first inequality is due to the bound on $\frac{r(t)}{(1+r(t))^2}$ mentioned above and the second is due to the boundedness of noise.

Importantly, if we show that $S(t^*) \geq S(t)$ for all $t$ such that $r(t) \leq \frac{\epsilon}{\min \{ M_1, M_2 \} \left(\frac{1-\delta}{1+\delta}\right) - \epsilon}$ or $r(t) \geq \frac{\min \{ M_1, M_2 \}\left(\frac{1-\delta}{1+\delta}\right) - \epsilon}{\epsilon}$ almost surely, then we must have
\begin{equation} \label{eq:tm-works}
    \frac{\epsilon}{\min \{ M_1, M_2 \}\left(\frac{1-\delta}{1+\delta} \right) - \epsilon} \leq r(t_m) \leq \frac{\min \{ M_1, M_2 \}\left(\frac{1-\delta}{1+\delta} \right) - \epsilon }{\epsilon} \,,
\end{equation}
as $S(t_m) \geq S(t^*)$ due to the definition of $t_m$ as an argmax.

Due to the condition on $\delta$, it is in fact the case that $S(t^*) \geq S(t)$ for $t$ which satisfy the aforementioned conditions.
Namely, when $\delta \leq \delta^*(a_1, C_1, C_2, M_1, M_2)$ for
\begin{align}
&\delta^*(a_1, C_1, C_2, M_1, M_2) = \nonumber \\
&\quad \min \left\lbrace \left(\left( \frac{1 + \frac{1}{4}  \left(e^{2a_1(C_1 - C_2)} + e^{-2a_1(C_1 - C_2)}  - 2 \right)}{1 + \frac{\epsilon}{\min\{M_1, M_2\} - \epsilon } \left(e^{2a_1(C_1 - C_2)} + e^{-2a_1(C_1 - C_2)}  - 2 \right)}\right)^{1/5} - 1\right) / 4, \frac{1}{4}\right\rbrace \label{eq:delta-star}\,,
\end{align}
we then see that because for any $\delta \leq \frac{1}{2}$, $\frac{1+\delta}{1 - \delta} \leq 1 + 4 \delta$.
Hence, $\delta$ also satisfies
\begin{align*}
\frac{1+\delta}{1-\delta} \leq 1 + 4\delta \leq \left( \frac{1 + \frac{1}{4}  \left(e^{2a_1(C_1 - C_2)} + e^{-2a_1(C_1 - C_2)}  - 2 \right)}{1 + \frac{\epsilon}{\min\{M_1, M_2\} - \epsilon } \left(e^{2a_1(C_1 - C_2)} + e^{-2a_1(C_1 - C_2)}  - 2 \right)}\right)^{1/5} \,.
\end{align*}
Given the assumption that $\epsilon < \min \{M_1, M_2\}$, we also note
\begin{equation} \label{eq:fix-m-bar}
    \frac{\min\{M_1, M_2 \} - \epsilon}{\min\{M_1, M_2 \}\left(\frac{1-\delta}{1+\delta}\right) - \epsilon} \leq 1 + 4 \delta \,.
\end{equation}
This results, as
\begin{align*}
    \min\{M_1, M_2 \} \geq 5\epsilon
    \implies &\min\{M_1, M_2 \} \geq \epsilon \frac{2 [1+\delta]}{(1 - 2\delta)} \qquad (\delta \leq \frac{1}{4})\\
    \implies &\min\{M_1, M_2 \} \geq \epsilon \frac{[2 \delta][1+\delta]}{\delta (1 - 2\delta)}\\
    \implies &\min\{M_1, M_2 \}\left(\frac{\delta - 2 \delta^2}{1+\delta}\right) \geq \epsilon [2 \delta]\\
    \implies &\min\{M_1, M_2 \}\left(\frac{- 2 \delta + 4 \delta^2)}{1+\delta}\right) \leq - \epsilon [4 \delta]\\
    \implies &\min\{M_1, M_2 \}\left(\frac{1 + \delta - (1 + 3 \delta - 4 \delta^2)}{1+\delta}\right) \leq - \epsilon [4 \delta]\\
    \implies &\min\{M_1, M_2 \}\left(1 - \frac{1+3 \delta - 4 \delta^2}{1+\delta}\right) \leq - \epsilon [4 \delta]\\
    \implies &\min\{M_1, M_2 \}\left(1 - \frac{1-\delta}{1+\delta}[1 + 4 \delta]\right) \leq - \epsilon [4 \delta]\\
    \implies &\min\{M_1, M_2 \} - \left[\min\{M_1, M_2 \}\left(\frac{1-\delta}{1+\delta}[1 + 4 \delta]\right) \right] \leq - \epsilon [4 \delta]\\
    \implies &\min\{M_1, M_2 \} - \left[\min\{M_1, M_2 \}\left(\frac{1-\delta}{1+\delta}[1 + 4 \delta]\right) \right] - \epsilon \leq - \epsilon [1 + 4 \delta]\\
    \implies &\min\{M_1, M_2 \} - \epsilon \leq \left[\min\{M_1, M_2 \}\left(\frac{1-\delta}{1+\delta}\right) - \epsilon \right][1 + 4 \delta] \\
    \implies &\frac{\min\{M_1, M_2 \} - \epsilon}{\min\{M_1, M_2 \}\left(\frac{1-\delta}{1+\delta}\right) - \epsilon} \leq 1 + 4 \delta \,.
\end{align*}
This then implies
\begin{align*}
    &4 \log \frac{1+\delta}{1-\delta} + \log  \frac{\min\{M_1, M_2 \} - \epsilon}{\min\{M_1, M_2 \}\left(\frac{1-\delta}{1+\delta}\right) - \epsilon} \\
    &\qquad\leq \log \left( 1 + \frac{1}{4}  \left(e^{2a_1(C_1 - C_2)} + e^{-2a_1(C_1 - C_2)}  - 2 \right)\right) \\
    &\qquad\quad - \log \left(1 + \frac{\epsilon}{\min\{M_1, M_2\} - \epsilon } \left(e^{2a_1(C_1 - C_2)} + e^{-2a_1(C_1 - C_2)}  - 2 \right) \right) \,,
\end{align*}
which implies
\begin{align*}
    4 \log \frac{1+\delta}{1-\delta} &\leq \log \left( 1 + \frac{1}{4}  \left(e^{2a_1(C_1 - C_2)} + e^{-2a_1(C_1 - C_2)}  - 2 \right)\right) \\
    &\quad - \log \left(\frac{\min\{M_1, M_2 \} - \epsilon}{\min\{M_1, M_2 \}\left(\frac{1-\delta}{1+\delta}\right) - \epsilon}  + \frac{\epsilon \left( \frac{\min\{M_1, M_2 \} - \epsilon}{\min\{M_1, M_2 \}\left(\frac{1-\delta}{1+\delta}\right) - \epsilon} \right)}{\min\{M_1, M_2\} - \epsilon } \left(e^{2a_1(C_1 - C_2)} + e^{-2a_1(C_1 - C_2)}  - 2 \right) \right) \\
    &= \log \left( 1 + \frac{1}{4}  \left(e^{2a_1(C_1 - C_2)} + e^{-2a_1(C_1 - C_2)}  - 2 \right)\right) \\
    &\quad - \log \left(1 + \frac{\epsilon}{\min\{M_1, M_2\}\frac{1-\delta}{1+\delta}  - \epsilon } \left(e^{2a_1(C_1 - C_2)} + e^{-2a_1(C_1 - C_2)}  - 2 \right) \right)
\end{align*}
From here, we can conclude that the condition on $\delta^*(a_1, C_1, C_2, M_1, M_2)$ implies that the lower bound on $S(t^*)$ is at least as large as the upper bound on $S(t)$ for $t$ such that either $r(t) \leq \frac{\epsilon}{\min \{ M_1, M_2 \} \left(\frac{1-\delta}{1+\delta}\right) - \epsilon}$ or $r(t) \geq \frac{\min \{ M_1, M_2 \}\left(\frac{1-\delta}{1+\delta}\right) - \epsilon}{\epsilon}$, as the final inequality implies
\begin{align*}
    &- 2 a_1 + \log \left( 1 + \frac{1}{4} \left(e^{2a_1(C_1 - C_2)} + e^{-2a_1(C_1 - C_2)}  - 2 \right) \right) - 2 \log \frac{1 + \delta}{1 - \delta} \geq \\
    &\quad - 2 a_1 + \log \left( 1 + \frac{\epsilon}{\min \{ M_1, M_2 \} \left(\frac{1-\delta}{1+\delta}\right) - \epsilon} \left(e^{2a_1(C_1 - C_2)} + e^{-2a_1(C_1 - C_2)}  - 2 \right) \right) + 2 \log \frac{1 + \delta}{1 - \delta} \,.
\end{align*}

Therefore, as mentioned above, because $S(t_m) \geq S(t^*)$ due to the definition of $t_m$ as an argmax, we can conclude that equation \eqref{eq:tm-works} holds.

\noindent \textbf{Justification of \eqref{eq:tm_midpoint}.} Equation \eqref{eq:tm-works} can then be used to show that \eqref{eq:tm_midpoint} holds.
Specifically, recall that in the case where $a_1 = a_2$, the case ratio $r(t)$ can be written as in \eqref{eq:rt}, which is an increasing function in $t$. Moreover, we have the following set of inequalities:
\begin{align*}
    r(C_1) &\leq \frac{\epsilon}{M_1 - \epsilon} \\
    &\leq\frac{\epsilon}{\min \{ M_1, M_2 \} \left(\frac{1-\delta}{1+\delta}\right) - \epsilon}\\
    &\leq r(t_m) \\
    &\leq \frac{\min \{ M_1, M_2 \} \left(\frac{1-\delta}{1+\delta}\right) - \epsilon}{\epsilon} \\
    &\leq \frac{M_2 - \epsilon}{\epsilon} \\
    &\leq r(C_2) \,,
\end{align*}
which by the increasing nature of $r(t)$ on the interval then implies directly that 
\[ C_1 \leq t_m \leq C_2 \,,\]
as desired.

\noindent \textbf{Justification of \eqref{eq:c1-hat-ok} and \eqref{eq:c2-hat-ok}.} Equation \eqref{eq:tm-works} can also be used to show that both \eqref{eq:c1-hat-ok} and \eqref{eq:c2-hat-ok} hold.
We proceed by showing \eqref{eq:c1-hat-ok} holds, and note that the proof for \eqref{eq:c2-hat-ok} follows symmetrically.

We first note that \eqref{eq:c1-hat-ok} is equivalent to the condition $r(\hat{C}_1) \leq 1$, and that we must have $\hat{C}_1 \leq t_m$ by the definition of $\hat{C}_1$.
Next, suppose towards contradiction that $r(\hat{C}_1) > 1$.
Because $r(t)$ is an increasing function of $t$, and from \eqref{eq:tm-works}, it would have to be the case that
\[ 1 < r(\hat{C}_1) \leq r(t_m) \leq \frac{\min \{ M_1, M_2 \} \left(\frac{1-\delta}{1+\delta}\right) - \epsilon}{\epsilon} \,.\]
From Lemma \ref{lem:sufficient-assumption-1}, we know that because $r(\hat{C}_1) > 1$, at the time index $\hat{C_1}$ the number of cases from the first outbreak is at most $\epsilon$, i.e. $M_1 e^{-a_1(\hat{C_1} - C_1)^2} \leq \epsilon$. 
Further, because 
\[ \frac{M_2 e^{-a_2(\hat{C}_1-C_2)^2}}{M_1 e^{-a_1(\hat{C}_1-C_1)^2}} = r(\hat{C}_1) \leq \frac{\min \{ M_1, M_2 \} \left(\frac{1-\delta}{1+\delta}\right) - \epsilon}{\epsilon}\,,\]
we must also have $M_2 e^{-a_1(\hat{C_1} - C_2)^2} < \min \{ M_1, M_2 \} \left(\frac{1-\delta}{1+\delta}\right) - \epsilon$, or equivalently,
\begin{align*}
N(\hat{C}_1) &= \left(\sum_{k = 1}^2 M_k e^{-a_k (\hat{C_1} - C_k)^2}\right)[1 + \eta_{\hat{C}_1}] \\
&< \left(M_2 e^{-a_1(\hat{C_1} - C_2)^2} + \epsilon \right) [1 + \eta_{\hat{C}_1}]\\
&\leq  \min \{ M_1, M_2 \} \left(\frac{1-\delta}{1+\delta}\right)[1 + \eta_{\hat{C}_1}]  \\
&\leq  \min \{ M_1, M_2 \} [1-\delta] &&(|\eta_{\hat{C}_1}| \leq \delta) \,.
\end{align*} 
By the definition of $\hat{C_1}$ as an argmax, this then implies that for all $t < t_m$,
\[ N(t)  \leq \min \{ M_1, M_2 \}[1-\delta] \,.\]
However, at the point $C_1$, which is the time at which cases in the first outbreak are maximized, we must have $N(C_1) > M_1 [1-\delta]$, a contradiction. Hence, it must be the case that $ r(\hat{C}_1) \leq 1$, which is equivalent to \eqref{eq:c1-hat-ok}, as desired.
\end{proof}

Finally, we put the above two lemmas together for the final proof of the claim.

\noindent \textbf{Proof of Theorem \algtheorem. }
We prove for the case where $k =1$, as the $k = 2$ case holds symmetrically. 
To show \eqmk, we first show $\hat{M}_k \geq M_k [1 - \delta]$.
We have
\begin{align*}
    \hat{M}_1 &= \max_{0 \leq t \leq t_m} N(t) \\
    &\geq \max_{0 \leq t \leq t_m} 
    \left(\sum_{k = 1}^2 M_k e^{-a_k(t-C_k)^2}\right) [1 - \delta] &&(\eta_t \in [-\delta, \delta]) \\
    &\geq \max_{0 \leq t \leq t_m} M_1 e^{-a_1(t-C_1)^2}[1 - \delta] &&(M_2 e^{-a_2(t-C_2)^2} \geq 0)\\
    &\geq M_1 e^{-a_1(\lceil C_1 \rceil -C_1)^2}[1 - \delta] &&(\text{From Lemma \ref{lem:midpoint}}, t_m \geq C_1) \,.
\end{align*}
The upper bound holds similarly, as
\begin{align*}
    \hat{M}_1 &= \max_{0 \leq t \leq t_m} N(t) \\
    &\leq \max_{0 \leq t \leq t_m} 
    \left(\sum_{k = 1}^2 M_k e^{-a_k(t-C_k)^2}\right) [1 + \delta] &&(\eta_t \in [-\delta, \delta]) \\
    &\leq \max_{0 \leq t \leq t_m} (M_1 e^{-a_1(t-C_1)^2} + \epsilon) [1 + \delta] &&(*)\\
    &\leq (M_1 + \epsilon)[1+\delta] &&(e^{-x} \leq 1, ~\forall x \geq 0) \,.
\end{align*}
Where the inequality in $(*)$ holds because Lemma \ref{lem:midpoint} ensures that at $\hat{C}_1$, $M_1 e^{-a_1(\hat{C}_1-C_1)^2} \geq M_2 e^{-a_2(\hat{C}_1-C_2)^2}$ and Lemma \ref{lem:sufficient-assumption-1} then ensures that $M_2 e^{-a_2(\hat{C}_1-C_2)^2} < \epsilon$.
We next show the proof of \eqck, and again show the case where $k = 1$.
To begin, we note that by the definition of $\hat{M}_1$
\begin{align}
    \hat{M}_1 = \left(\sum_{k = 1}^2 M_k e^{-a_k(\hat{C}_1-C_k)^2}\right)[1 + \eta_{\hat{C}_1}] \,.
\end{align}
We further note the following lower and upper bound on $\hat{M}_1$,
\begin{align}
    \hat{M}_1 = \left(\sum_{k = 1}^2 M_k e^{-a_k(\hat{C}_1-C_k)^2}\right)[1 + \eta_{\hat{C}_1}] &\geq M_1 e^{-a_1(\lceil C_1 \rceil -C_1)^2}[1 -\delta] \label{eq:ck-proof-lower}\\
    \hat{M}_1 = \left(\sum_{k = 1}^2 M_k e^{-a_k(\hat{C}_1-C_k)^2}\right)[1 + \eta_{\hat{C}_1}] &\leq (M_1 e^{-a_1(\hat{C}_1-C_1)^2} + \epsilon)[1 + \delta] \label{eq:ck-proof-upper} \,,
\end{align}
where both inequalities follow directly from the computations above used to prove \eqmk.
Hence, from the two inequalities above,
\begin{align*}
    M_1 e^{-a_1(\lceil C_1 \rceil -C_1)^2}[1 -\delta] \leq (M_1 e^{-a_1(\hat{C}_1-C_1)^2} + \epsilon) \,.
\end{align*}
Solving for $|\hat{C}_1-C_1|$, we get
\begin{align*}
    M_1 e^{-a_1(\hat{C}_1-C_1)^2} &\geq  M_1 e^{-a_1(\lceil C_1 \rceil -C_1)^2} \left( \frac{1-\delta}{1+\delta}\right) - \epsilon \implies \\
    -a_1(\hat{C}_1-C_1)^2 &\geq \log \left( e^{-a_1(\lceil C_1 \rceil -C_1)^2} \left( \frac{1-\delta}{1+\delta}\right) - \frac{\epsilon}{M_1} \right) \implies\\
    |\hat{C}_1-C_1| &\leq \sqrt{\frac{1}{a_1} \log \left(\frac{1}{ e^{-a_1(\lceil C_1 \rceil -C_1)^2} \left( \frac{1-\delta}{1+\delta}\right) - \frac{\epsilon}{M_1}} \right)} \,.
\end{align*}
Rewriting the final equation yields
\begin{align*}
    |\hat{C}_1-C_1| &\leq \sqrt{\frac{1}{a_1} \log \left(\frac{M_1}{M_1 e^{-a_1(\lceil C_1 \rceil -C_1)^2} \left( \frac{1-\delta}{1+\delta}\right) - \epsilon} \right)} \\
    &=\sqrt{\frac{1}{a_1} \log \left(1 + \frac{M_1 - M_1 e^{-a_1(\lceil C_1 \rceil -C_1)^2} \left( \frac{1-\delta}{1+\delta}\right) + \epsilon}{M_1 e^{-a_1(\lceil C_1 \rceil -C_1)^2} \left( \frac{1-\delta}{1+\delta}\right) - \epsilon} \right)} \,.
\end{align*}
To get an approximation, we can apply a Taylor expansion as long as $\frac{M_1 - M_1 e^{-a_1/4} \left( \frac{1-\delta}{1+\delta}\right) + \epsilon}{M_1 e^{-a_1(\lceil C_1 \rceil -C_1)^2} \left( \frac{1-\delta}{1+\delta}\right) - \epsilon} \approx 0$, which holds when $\delta$ is close to 0 and $M_1 e^{-a_1(\lceil C_1 \rceil -C_1)^2} \left( \frac{1-\delta}{1+\delta}\right) \gg \epsilon$.
Roughly, we would then have
\begin{align*}
    |\hat{C}_1-C_1| &\lesssim \sqrt{\frac{1}{a_1} \left( \frac{M_1 - M_1 e^{-a_1(\lceil C_1 \rceil -C_1)^2} \left( \frac{1-\delta}{1+\delta}\right) + \epsilon}{M_1 e^{-a_1(\lceil C_1 \rceil -C_1)^2} \left( \frac{1-\delta}{1+\delta}\right) - \epsilon} \right)} \\
    &\approx \sqrt{\frac{1}{a_1} \left( \frac{M_1 - M_1 e^{-a_1(\lceil C_1 \rceil -C_1)^2} \left( \frac{1-\delta}{1+\delta}\right) }{M_1 e^{-a_1(\lceil C_1 \rceil -C_1)^2} \left( \frac{1-\delta}{1+\delta}\right) - \epsilon } \right) + \frac{\epsilon}{M_1 e^{-a_1(\lceil C_1 \rceil -C_1)^2} \left( \frac{1-\delta}{1+\delta}\right) - \epsilon }} \\
    &\approx \sqrt{\frac{1}{a_1} \left(\frac{1 -  e^{-a_1(\lceil C_1 \rceil -C_1)^2} \left( \frac{1-\delta}{1+\delta}\right) }{e^{-a_1(\lceil C_1 \rceil -C_1)^2} \left( \frac{1-\delta}{1+\delta}\right)} + \frac{\epsilon}{M_1} \right)}\,,
\end{align*}
completing the claim of the theorem. \hfill $\qedsymbol$

\subsection{Relaxation of the Assumption $a_1 = a_2$}
The assumption $a_1 = a_2$ was made for clarity of presentation, but is not required in general.
Namely, we note that Lemma \ref{lem:midpoint} can be re-written for the case $a_1 \neq a_2$ with the following modifications.

\begin{lemma}\label{lem:midpoint-a1-neq-a2}
Suppose the parameters of the underlying mixture satisfy the following properties:
\begin{enumerate}
        \item $M_k \leq M\,, \quad k = 1, 2$
    \item $a_k \geq a\,, \quad k = 1, 2$
    \item $|C_1 - C_2| \geq \max \bigg\lbrace 2 \sqrt{\frac{1}{a} \log \frac{M}{\epsilon}}, \\
    \sqrt{ \frac{1}{a_1 a_2}\left(\frac{1}{4}\log^2 \left(\left(\max \{1, e^{2a_1 - 2a_2}\} - \left(\frac{1 + e^{a_1 - a_2}}{2} \right)^2  \right) \frac{4}{e^{a_1 - a_2}} + 2 \right) + (a_2 - a_1)\log \frac{M_1}{M_2} \right)} \bigg\rbrace \,.$
    \item $\delta$ and $\epsilon$ satisfy \eqref{eq:delta-a1-neq-a2}
    \item $C_1, C_2 \in [0, T]$
    \item $r(t) = \frac{M_2e^{-a_2 (t-C_2)^2}}{M_1e^{-a_1 (t-C_1)^2}}$ is increasing in $t$ on the interval $[0, T]$.
\end{enumerate}
Then, as in Lemma \ref{lem:midpoint}, for
\[ t_m = \argmax_{1 \leq t \leq T} S(t) \,,\]
the quantities
\begin{align*}
    \hat{C}_1 &= \argmax_{0 \leq t \leq t_m} N(t) \\
    \hat{C}_2 &= \argmax_{t_m \leq t \leq T} N(t) \,,
\end{align*}
which represent the estimates of $\hat{C}_k$ in Line 2 of Algorithm \alginitialization, and $t_m$ itself will satisfy \eqref{eq:tm_midpoint}-\eqref{eq:c2-hat-ok}.
\end{lemma}
\begin{proof}
The proof again begins by noting the relationship between $S(t)$ and the case ratio $r(t) = \frac{M_2e^{-a_2(t-C_2)^2}}{M_1e^{-a_1(t-C_1)^2}}$, which we claim can be written as 
\begin{align}
    S(t) &= -2a_1 + \log \bigg( \left(\frac{1 + r(t) e^{a_1 - a_2}}{1 + r(t)} \right)^2 \label{eq:st-rt-a1-neq-a2}\\
    &\quad + \frac{r(t)e^{a_1 - a_2}}{(1+r(t))^2} \left(e^{2\sqrt{a_1 a_2 (C_1 - C_2)^2 - (a_2 - a_1)\log \frac{r(t)M_1}{M_2}}} + e^{-2\sqrt{a_1 a_2 (C_1 - C_2)^2 - (a_2 - a_1)\log \frac{r(t)M_1}{M_2}}} - 2 \right)\bigg) \nonumber\\
    &\quad + \log \frac{(1 + \eta_{t+1})(1+\eta_{t-1})}{(1 + \eta_t)^2} \nonumber\,. 
\end{align}
The proof of this observation is as follows.
First, we note that from the manipulations of Lemma \ref{lem:midpoint}, 
\begin{align*}
    S(t) &= - 2 a_1 + \log \left( \left( \frac{1 + (r(t)e^{a_1 - a_2})^2}{(1 + r(t))^2}\right) + \frac{r(t)}{(1 + r(t))^2} \left(\frac{e^{-2a_1(t-C_1)}}{e^{-2a_2(t-C_2)}} + \frac{e^{-2a_2(t-C_2)}}{e^{-2a_1(t-C_1)}} \right) e^{a_1 - a_2}\right)\\
    &\quad + \log \frac{(1 + \eta_{t+1})(1+\eta_{t-1})}{(1 + \eta_t)^2}\,.
\end{align*}
Next, we note that since $r(t) = \frac{M_2e^{-a_2(t-C_2)^2}}{M_1e^{-a_1(t-C_1)^2}}$,
\[ -a_2(t-C_2)^2 + a_1(t-C_1)^2 = \log \frac{M_1 r(t)}{M_2} \,,\]
which implies
\begin{align*}
    (a_1 - a_2)t^2 + 2(a_2 C_2 - a_1 C_1)t + a_1 C_1^2 - a_2 C_2^2 - \log \frac{M_1 r(t)}{M_2} \,.
\end{align*}
Then, applying the quadratic formula, we see:
\begin{align*}
    t &= \frac{-2(a_2 C_2 - a_1 C_1) \pm \sqrt{4(a_2 C_2 - a_1 C_1)^2 - 4 (a_1 - a_2) \left(a_1 C_1^2 - a_2 C_2^2 - \log \frac{M_1 r(t)}{M_2}\right)}}{2(a_1 - a_2)} \\
    &= \frac{-2(a_2 C_2 - a_1 C_1) \pm 2\sqrt{a_1 a_2 (C_1 - C_2)^2 - (a_2 - a_1) \left(\log \frac{M_1 r(t)}{M_2}\right)}}{2(a_1 - a_2)} \,,
\end{align*}
where the second equality follows from noting
\begin{align*}
    4(a_2 C_2 - a_1 C_1)^2 &- 4 (a_1 - a_2) \left(a_1 C_1^2 - a_2 C_2^2 - \log \frac{M_1 r(t)}{M_2}\right) \\
    &= 4a_2^2 C_2^2 + 4a_1^2 C_1^2 - 8 a_1 a_2 C_1 C_2 - 4 a_1^2 C_1^2 - 4 a_2^2 C_2^2 \\
    &\quad + 4 a_1 a_2 C_1^2 + 4 a_1 a_2 C_2^2 - 4(a_2 - a_1) \log \frac{M_1 r(t)}{M_2} \\
    &= 4 a_1 a_2 (C_1^2 - 2C_1 C_2 + C_2^2) - 4 (a_2 - a_1) \log \frac{M_1 r(t)}{M_2}\\
    &= 4 a_1 a_2 (C_1 - C_2)^2 - 4 (a_2 - a_1) \log \frac{M_1 r(t)}{M_2} \,.
\end{align*}
Given that $t$ can be written in terms of $r(t)$, we can then plug back into the definition of $S(t)$, and find that since the solution to the quadratic formula shows
\[ 2(a_1 - a_2)t + 2(a_2 C_2 - a_1 C_1) = \pm 2\sqrt{a_1 a_2 (C_1 - C_2)^2 - (a_2 - a_1) \left(\log \frac{M_1 r(t)}{M_2}\right)} \,,  \]
it must be the case
\begin{align*}
&\frac{e^{-2a_1(t-C_1)}}{e^{-2a_2(t-C_2)}} + \frac{e^{-2a_2(t-C_2)}}{e^{-2a_1(t-C_1)}} =
&e^{2\sqrt{a_1 a_2 (C_1 - C_2)^2 - (a_2 - a_1)\log \frac{r(t)M_1}{M_2}}} + e^{-2\sqrt{a_1 a_2 (C_1 - C_2)^2 - (a_2 - a_1)\log \frac{r(t)M_1}{M_2}}} \,,
\end{align*}
as both positive and negative terms will appear in the sum regardless of which selection of $t$ is taken.
Plugging this expression into \eqref{eq:st-rt-general} proves the equality in \eqref{eq:st-rt-a1-neq-a2}.

At the index $t^*$ for which $r(t^*) = 1$, which is unique because we assume $r(t)$ is increasing on the observed interval, we see that the following statement holds:
\begin{align*}
    S(t^*) &\geq -2a_1 + \log \bigg( \left(\frac{1 + e^{a_1 - a_2}}{2} \right)^2 \\
    &\quad + \frac{e^{a_1 - a_2}}{4} \left(e^{2\sqrt{a_1 a_2 (C_1 - C_2)^2 - (a_2 - a_1)\log \frac{M_1}{M_2}}} + e^{-2\sqrt{a_1 a_2 (C_1 - C_2)^2 - (a_2 - a_1)\log \frac{M_1}{M_2}}} - 2 \right)\bigg) \\
    &\quad - 2 \log \left(\frac{1+\delta}{1-\delta} \right) \,.
\end{align*}
The inequality here comes from noting $|\eta_t| \leq \delta$.

Now, suppose $t$ is such that either $r(t) \leq \frac{\epsilon}{\min \{ M_1, M_2 \}\left(\frac{1-\delta}{1+\delta}\right) - \epsilon}$ or $r(t) \geq \frac{\min \{ M_1, M_2 \}\left(\frac{1-\delta}{1+\delta}\right) - \epsilon}{\epsilon}$. 
Since we assume $r(t)$ is increasing in $t$ on the observed interval, there are precisely two continuous regions of time intervals which are considered here.
We then see, for $t$ in these regions, since for any $\tau < 1$ we know that $r(t) \leq \tau \implies \frac{r(t)}{(1+r(t))^2} \leq \tau$ and $r(t) \geq \frac{1}{\tau} \implies \frac{r(t)}{(1+r(t))^2} \leq \tau, $
\begin{align*}
    S(t) &= -2a_1 + \log \bigg( \left(\frac{1 + r(t) e^{a_1 - a_2}}{1 + r(t)} \right)^2 \\
    &\quad + \frac{r(t)e^{a_1 - a_2}}{(1+r(t))^2} \left(e^{2\sqrt{a_1 a_2 (C_1 - C_2)^2 - (a_2 - a_1)\log \frac{r(t)M_1}{M_2}}} + e^{-2\sqrt{a_1 a_2 (C_1 - C_2)^2 - (a_2 - a_1)\log \frac{r(t)M_1}{M_2}}} - 2 \right)\bigg)\\
    &\quad + \log \frac{(1 + \eta_{t+1})(1+\eta_{t-1})}{(1 + \eta_t)^2} \\
    &\leq -2a_1 + \log \bigg( \left(\frac{1 + r(t) e^{a_1 - a_2}}{1 + r(t)} \right)^2 \\
    &\quad + \frac{\epsilon}{\min \{ M_1, M_2 \}\left(\frac{1-\delta}{1+\delta}\right) - \epsilon} e^{a_1 - a_2} \bigg(e^{2\sqrt{a_1 a_2 (C_1 - C_2)^2 - (a_2 - a_1)\log \frac{r(t)M_1}{M_2}}} \\
    &\qquad\qquad\qquad\qquad\qquad + e^{-2\sqrt{a_1 a_2 (C_1 - C_2)^2 - (a_2 - a_1)\log \frac{r(t)M_1}{M_2}}} - 2 \bigg)\bigg)\\
    &\quad + \log \frac{(1 + \eta_{t+1})(1+\eta_{t-1})}{(1 + \eta_t)^2} 
\end{align*}
Defining
\begin{align*}
&Q(a_1, a_2, C_1, C_2, M_1, M_2)= \\
&\sup_{t \in [0, T] ~|~ \frac{r(t)}{(1+r(t))^2} \leq \frac{\epsilon}{\min \{ M_1, M_2 \}\left(\frac{1-\delta}{1+\delta}\right) - \epsilon}} e^{2\sqrt{a_1 a_2 (C_1 - C_2)^2 - (a_2 - a_1)\log \frac{r(t)M_1}{M_2}}} + e^{-2\sqrt{a_1 a_2 (C_1 - C_2)^2 - (a_2 - a_1)\log \frac{r(t)M_1}{M_2}}} - 2 \,,
\end{align*} 

it is then clear that for such values of $t$,
\begin{align*}
    S(t) &\leq - 2 a_1 + \log \bigg( \left(\frac{1 + r(t) e^{a_1 - a_2}}{1 + r(t)} \right)^2  + \frac{\epsilon}{\min \{ M_1, M_2 \}\left(\frac{1-\delta}{1+\delta}\right) - \epsilon} e^{a_1 - a_2} Q(a_1, a_2, C_1, C_2, M_1, M_2)\bigg)\\
    &\quad + \log \frac{(1 + \eta_{t+1})(1+\eta_{t-1})}{(1 + \eta_t)^2}  \\
    &\leq- 2 a_1 + \log \bigg( \left(\frac{1 + r(t) e^{a_1 - a_2}}{1 + r(t)} \right)^2  + \frac{\epsilon}{\min \{ M_1, M_2 \}\left(\frac{1-\delta}{1+\delta}\right) - \epsilon} e^{a_1 - a_2} Q(a_1, a_2, C_1, C_2, M_1, M_2)\bigg)\\
    &\quad + 2 \log \frac{1+\delta}{1-\delta} \qquad (|\eta_t| \leq \delta)\\
    &\leq -2 a_1 + \log \bigg( \max \{1, e^{2a_1 - 2a_2}\}  + \frac{\epsilon}{\min \{ M_1, M_2 \}\left(\frac{1-\delta}{1+\delta}\right) - \epsilon} e^{a_1 - a_2} Q(a_1, a_2, C_1, C_2, M_1, M_2)\bigg)\\
    &\quad + 2 \log \frac{1+\delta}{1-\delta} \,. \qquad \left(\left(\frac{1 + r(t) e^{a_1 - a_2}}{1 + r(t)} \right)^2 \leq \max\{1, e^{2a_1 - 2a_2}\} \right) \,.
\end{align*}
Importantly, $Q(a_1, a_2, C_1, C_2, M_1, M_2)$ is bounded, as $t \in [0, T]$ is bounded, and is not a function of $\epsilon$.

Hence, if $\delta$ and $\epsilon$ are small enough as to ensure
\begin{align} 
   &-2a_1 + \log \bigg( \left(\frac{1 + e^{a_1 - a_2}}{2} \right)^2 \label{eq:delta-a1-neq-a2}\\
    &\quad + \frac{e^{a_1 - a_2}}{4} \left(e^{2\sqrt{a_1 a_2 (C_1 - C_2)^2 - (a_2 - a_1)\log \frac{M_1}{M_2}}} + e^{-2\sqrt{a_1 a_2 (C_1 - C_2)^2 - (a_2 - a_1)\log \frac{M_1}{M_2}}} - 2 \right)\bigg) \nonumber\\
    &\quad - 2 \log \left(\frac{1+\delta}{1-\delta} \right) \geq \nonumber \\
    & - 2 a_1 + \log \bigg( \max \{1, e^{2a_1 - 2a_2}\}  + \frac{\epsilon}{\min \{ M_1, M_2 \}\left(\frac{1-\delta}{1+\delta}\right) - \epsilon} e^{a_1 - a_2} Q(a_1, a_2, C_1, C_2, M_2)\bigg)\nonumber \\
    &\quad + 2 \log \frac{1+\delta}{1-\delta} \nonumber
\end{align}
then the conclusion of Lemma \ref{lem:midpoint} will hold and the proof will follow as before.
We note that such a selection of $\delta$ and $\epsilon$ is possible, as when both are equal to 0, the inequality holds strictly so long as:
\begin{align*}
    &\left(\frac{1 + e^{a_1 - a_2}}{2} \right)^2\\
    &\quad + \frac{e^{a_1 - a_2}}{4} \left(e^{2\sqrt{a_1 a_2 (C_1 - C_2)^2 - (a_2 - a_1)\log \frac{M_1}{M_2}}} + e^{-2\sqrt{a_1 a_2 (C_1 - C_2)^2 - (a_2 - a_1)\log \frac{M_1}{M_2}}} - 2 \right) \\
    &> \max \{1, e^{2a_1 - 2a_2}\} \,,
\end{align*}
which is guaranteed by the temporal well-separation condition in Assumption 3 of the Lemma, since
\begin{align*}
    & |C_1 - C_2|  \\
    &> \sqrt{ \frac{1}{a_1 a_2}\left(\frac{1}{4}\log^2 \left(\left(\max \{1, e^{2a_1 - 2a_2}\} - \left(\frac{1 + e^{a_1 - a_2}}{2} \right)^2  \right) \frac{4}{e^{a_1 - a_2}} + 2 \right) + (a_2 - a_1)\log \frac{M_1}{M_2} \right)}\\
    \implies & a_1 a_2 (C_1 - C_2)^2 - (a_2 - a_1)\log \frac{M_1}{M_2}  \\
    &> \frac{1}{4}\log^2 \left(\left(\max \{1, e^{2a_1 - 2a_2}\} - \left(\frac{1 + e^{a_1 - a_2}}{2} \right)^2  \right) \frac{4}{e^{a_1 - a_2}} + 2 \right) \\
    \implies & 2\sqrt{a_1 a_2 (C_1 - C_2)^2 - (a_2 - a_1)\log \frac{M_1}{M_2}} \\
    &> \log \left(\left(\max \{1, e^{2a_1 - 2a_2}\} - \left(\frac{1 + e^{a_1 - a_2}}{2} \right)^2  \right) \frac{4}{e^{a_1 - a_2}} + 2 \right) \\
    \implies & e^{2\sqrt{a_1 a_2 (C_1 - C_2)^2 - (a_2 - a_1)\log \frac{M_1}{M_2}}} \\
    &> \left(\max \{1, e^{2a_1 - 2a_2}\} - \left(\frac{1 + e^{a_1 - a_2}}{2} \right)^2  \right) \frac{4}{e^{a_1 - a_2}} + 2\\
    \implies & e^{2\sqrt{a_1 a_2 (C_1 - C_2)^2 - (a_2 - a_1)\log \frac{M_1}{M_2}}} + e^{-2\sqrt{a_1 a_2 (C_1 - C_2)^2 - (a_2 - a_1)\log \frac{M_1}{M_2}}} - 2  \\
    &> \left(\max \{1, e^{2a_1 - 2a_2}\} - \left(\frac{1 + e^{a_1 - a_2}}{2} \right)^2  \right) \frac{4}{e^{a_1 - a_2}}\\
    \implies &\frac{e^{a_1 - a_2}}{4} \left(e^{2\sqrt{a_1 a_2 (C_1 - C_2)^2 - (a_2 - a_1)\log \frac{M_1}{M_2}}} + e^{-2\sqrt{a_1 a_2 (C_1 - C_2)^2 - (a_2 - a_1)\log \frac{M_1}{M_2}}} - 2 \right) \\
    &> \max \{1, e^{2a_1 - 2a_2}\} - \left(\frac{1 + e^{a_1 - a_2}}{2} \right)^2 \\
    \implies &\left(\frac{1 + e^{a_1 - a_2}}{2} \right)^2 \\
    &\quad + \frac{e^{a_1 - a_2}}{4} \left(e^{2\sqrt{a_1 a_2 (C_1 - C_2)^2 - (a_2 - a_1)\log \frac{M_1}{M_2}}} + e^{-2\sqrt{a_1 a_2 (C_1 - C_2)^2 - (a_2 - a_1)\log \frac{M_1}{M_2}}} - 2 \right) \\
    &> \max \{1, e^{2a_1 - 2a_2}\} \,,
\end{align*}
By continuity of each term on the right-hand side and left hand side of \eqref{eq:delta-a1-neq-a2}, we see that for fixed parameters $a_1, a_2, C_1, C_2, M_1,$ and $M_2$, there exist $\delta, \epsilon > 0$ which satisfy the above inequality as well.
\end{proof}

\section{Proof Details for Section \ref{s:theory}}
\label{a:theory}
In this section, we provide the details of the proofs for Section \theorysection, beginning with those related to the single community model of Section \sstheorygaussian and then those related to Section \sstheorysbm.

We begin with a review of Bernoulli's inequality, which is used throughout the proofs:
\begin{proposition}[Bernoulli's inequality]
For any $x > -1$, and every integer $r \geq 0$,
\[ (1 + x)^r \geq 1 + rx \,.\]
\end{proposition}
\begin{proof}
The claim follows from  induction on $r$. 
As a first base case, when $r = 0$, we see
\[ (1+x)^0 = 1 \geq 1 + 0x \,. \]
Next, we assume the induction hypothesis that for an arbitrary integer $k \geq 0$, and any $x > -1$, 
\[ (1 + x)^k \geq 1 + kx \,.\]
We wish to show the claim holds true when $r = k + 2$. Indeed,
\begin{align*}
    (1 + x)^{k+1} &= (1 + x)^k (1+x) \\
    &\geq (1+kx) (1 +x) \qquad \text{(Induction Hypothesis,} ~(1+x) \geq 0)\\
    &= (1+kx)(1 + x) \\
    &= 1 + (k+1)x + kx^2\\
    &\geq 1 + (k+1) x \,,
\end{align*}
where the final step follows from the fact $x^2 \geq 0$ and $k \geq 0$.
Hence, by induction, for any integer $r$ the claim holds.
\end{proof}

\subsection{Proof of Theorem \thmgaussianconnection} \label{aa:gaussian-connection-proof}
Recall the definition of $N(t)$ as the number of infected individuals at time $t$, and that in the model, $N(0) = 1$, as we begin with a single initially infected node.
We will define $\mathcal{F}_t = \sigma(I(0), I(1), \dots, I(t))$ as the $\sigma-$algebra generated by previous observations of infected individuals.

We first show upper bounds on the quantities on the conditional expectations of $N(t)$, $N(t)^2$, and $N(s) \times N(t)$ for any $s, t \geq 0$.
These bounds will then allow us to fully characterize the expectation of $N(t)$ on its own.
These upper bounds are summarized in Lemma \ref{lem:gaussian-upper}.

\begin{lemma} \label{lem:gaussian-upper}
In the model of Section \sstheorygaussian, recall $N(t) = |I(t)|$ is the number of cases at time $t$.
The following upper bounds hold.
\begin{align}
    \E[N(t)] &\leq (d\beta)^t \gamma^{\frac{t(t-1)}{2}} \,, \label{eq:lin-upper} \\
    \E[N(t)^2] &\leq (d\beta)^t \gamma^{\frac{t(t-1)}{2}} + (d\beta)^{2t} \gamma^{t(t-1)} + \sum_{t' = 1}^{t-1} (d\beta)^{t'} \gamma^{\frac{t'(t'-1)}{2}} \prod_{\tau = t'}^{t-1} (d\beta\gamma^\tau)^{2}  \,.\label{eq:quad-upper}
\end{align}
Notably, each upper bound above is constant with respect to $n$, the number of individuals in the graph.
\end{lemma}
\begin{proof}
We begin by noting that, conditioned on previous observations in $\mathcal{F}_t$, $N(t+1)$ is distributed as a binomial random variable with $n- \sum_{s = 0}^t N(s)$ trials and success probability $1 - (1 - \frac{d \beta \gamma^t}{n})^{N(t)}$.
That is, the size of the susceptible population is  $n- \sum_{s = 0}^t N(s)$ as this reflects the entire population that has not yet been infected or recovered, and the success probability reflects that a new infection occurs if at least one of the $N(t)$ infected individuals infects a particular susceptible node.

Therefore, the following bound can be shown for $\E[N(t+1)|\mathcal{F}_t]$.
\begin{align*}
    \E[N(t+1)|\mathcal{F}_t] &= \left( n- \sum_{s = 0}^t N(s)\right) \left(1 - \left(1 - \frac{d \beta \gamma^t}{n}\right)^{N(t)} \right) \\
    &\leq n \left(1 - \left(1 - \frac{d \beta \gamma^t}{n}\right)^{N(t)} \right) \\
    &\leq n\left(1 - 1 + N(t) \frac{d\beta\gamma^t}{n} \right) \qquad \text{(Bernoulli's inequality)} \,.
\end{align*}
Hence, simplifying, we see
\begin{equation}\label{eq:upper-cond-bound}
    \E[N(t+1)|\mathcal{F}_t] \leq N(t) d \beta \gamma^t \,.
\end{equation}
From this, the upper bound on $N(t)$ becomes apparent and follows from induction.
Specifically, since $N(0) = 1$, the bound holds for the case $t = 0$.
Then, if we assume the induction hypothesis $\E[N(t)] \leq (d\beta)^t \gamma^{\frac{t(t-1)}{2}}$ for arbitrary $t \geq 0$, we see
\begin{align*}
    \E[N(t+1)] &= \E[\E[N(t+1)|\mathcal{F}_t]] \\
    &\leq \E[N(t) d\beta\gamma^t] \qquad \eqref{eq:upper-cond-bound} \\
    &\leq (d\beta)^t \gamma^{\frac{t(t-1)}{2}} \times d\beta \gamma^t  \qquad \text{(Induction Hypothesis)}\\
    &= (d\beta)^{t+1} \gamma^{\frac{t(t-1) + 2t}{2}} = (d\beta)^{t+1} \gamma^{\frac{(t+1)t}{2}} \,,
\end{align*}
proving the bound \eqref{eq:lin-upper}.

Similarly, the following bound holds for $\E[N(t+1)^2|\mathcal{F}_t]$.
Recall that for a binomial random variable with $n$ trials and probability $p$ of success, the second moment is computed as $np(1-p) + (np)^2$.
Hence,
\begin{align*}
    \E[N(t+1)^2|\mathcal{F}_t] &= \left( n- \sum_{s = 0}^t N(s)\right) \left(1 - \left(1 - \frac{d \beta \gamma^t}{n}\right)^{N(t)} \right) \left(\left(1 - \frac{d \beta \gamma^t}{n}\right)^{N(t)} \right) +\\
    &\qquad \left(\left( n- \sum_{s = 0}^t N(s)\right) \left(1 - \left(1 - \frac{d \beta \gamma^t}{n}\right)^{N(t)} \right) \right)^2\\
    &\leq n \left(1 - \left(1 - \frac{d \beta \gamma^t}{n}\right)^{N(t)} \right) \left(\left(1 - \frac{d \beta \gamma^t}{n}\right)^{N(t)} \right) +\\
    &\qquad \left(n \left(1 - \left(1 - \frac{d \beta \gamma^t}{n}\right)^{N(t)} \right) \right)^2 \\
    &\leq n \left(1 - \left(1 - \frac{d \beta \gamma^t}{n}\right)^{N(t)} \right)  +  \left(n \left(1 - \left(1 - \frac{d \beta \gamma^t}{n}\right)^{N(t)} \right) \right)^2 \\
    &\leq n\left(1 - 1 + N(t) \frac{d\beta\gamma^t}{n} \right) + \left(n\left(1 - 1 + N(t) \frac{d\beta\gamma^t}{n} \right) \right)^2 \qquad \text{(Bernoulli's inequality)} \\
    &=N(t) d\beta\gamma^t + N(t)^2 (d\beta\gamma^t)^2 \,,
\end{align*}
As before, the final claim for the upper bound follows from induction.
As a base case, we note again that $\E[N(0)^2] = 1$ as the quantity is deterministic.
Next, we assume the induction hypothesis that $\E[N(t)^2] \leq (d\beta)^{2t} \gamma^{t(t-1)} + \sum_{t' = 0}^{t-1} (d\beta)^{t'} \gamma^{\frac{t'(t'+1)}{2}} \prod_{\tau = t'}^{t-1} (d\beta\gamma^\tau)^{2}$ for an arbitrary $t \geq 0$.
We then see
\begin{align*}
    \E[N(t+1)^2] &= \E[\E[N(t+1)^2|\mathcal{F}_t]] \\
    &\leq \E\left[N(t) d\beta\gamma^t + N(t)^2 (d\beta\gamma^t)^2 \right] \\
    &\leq (d\beta)^{t+1} \gamma^{\frac{t(t+1)}{2}} + (d\beta\gamma^t)^2 \E\left[N(t)^2 \right] \\
    &\leq (d\beta)^{t+1} \gamma^{\frac{t(t+1)}{2}} + (d\beta\gamma^t)^2 \left((d\beta)^t \gamma^{\frac{t(t-1)}{2}} + (d\beta)^{2t} \gamma^{t(t-1)} + \sum_{t' = 1}^{t-1} (d\beta)^{t'} \gamma^{\frac{t'(t'+1)}{2}} \prod_{\tau = t'}^{t-1} (d\beta\gamma^\tau)^{2} \right) \\
    &= (d\beta)^{t+1} \gamma^{\frac{t(t+1)}{2}} + (d\beta)^{2(t+1)} \gamma^{t(t+1)} + \sum_{t' = 0}^{t} (d\beta)^{t'} \gamma^{\frac{t'(t'+1)}{2}} \prod_{\tau = t'}^{t} (d\beta\gamma^\tau)^{2} \,,
\end{align*}
as desired, proving \eqref{eq:quad-upper}.
\end{proof}

As a Corollary, we note the following upper bound on the second moment of $N(t)$, which helps to simplify the presentation of the results:
\begin{corollary}\label{corr:quad-upper-o-notation}
\begin{equation}
\E[N(t)^2] \leq O\left( t e^{\log^2 (d \beta / \sqrt{\gamma} )/ \log (1/\gamma)}\right) \,.
\end{equation}
\end{corollary}
\begin{proof}
The proof of this claim follows by first noting that 
\[ (d\beta)^{2t} \gamma^{t(t-1)}  \leq e^{(\log^2 d \beta / \sqrt{\gamma} )/ \log (1/\gamma)} \,.\]
This follows, as
\begin{align*}
    (d\beta)^{2t} \gamma^{t(t-1)} &\leq \max_{t}(d\beta)^{2t} \gamma^{t(t-1)} \\
    &= \max_t e^{- t^2 \log (1/\gamma) + 2t \log (d\beta/\sqrt{\gamma})} \,.
\end{align*}
Solving the maximization problem, we find that the right hand side maximized at $t^* = \frac{ \log (d\beta/\sqrt{\gamma})}{\log 1/\gamma}$, which yields
\begin{align*}
    (d\beta)^{2t} \gamma^{t(t-1)} &\leq e^{\log^2 (d \beta / \sqrt{\gamma})/ \log (1/\gamma)} \,.
\end{align*}
We also note that $(d\beta)^{2t} \gamma^{t(t-1)}$ is the dominant term in the inequality \eqref{eq:quad-upper}, and that as long as $d\beta > 1$, which corresponds to the non-trivial initial condition for spread, then each term in the summation of \eqref{eq:quad-upper} can be bounded by $e^{\log^2 (d \beta / \sqrt{\gamma})/ \log (1/\gamma)}$. 
Hence, from Lemma \ref{lem:gaussian-upper}, we can then write
\begin{align*}
    \E[N(t)^2] &\leq (d\beta)^t \gamma^{\frac{t(t-1)}{2}} + (d\beta)^{2t} \gamma^{t(t-1)} + \sum_{t' = 1}^{t-1} (d\beta)^{t'} \gamma^{\frac{t'(t'-1)}{2}} \prod_{\tau = t'}^{t-1} (d\beta\gamma^\tau)^{2} \\
    &\leq t e^{(\log^2 d \beta / \sqrt{\gamma} )/ \log (1/\gamma)} \,,
\end{align*}
proving the claim.
\end{proof}

As an additional Corollary, the cross-products $\E[N(s) N(t)]$ are also bounded by a constant with respect to $n$:
\begin{corollary} \label{corr:gaussian-cross-upper}
\begin{equation}
    \E[N(s)N(t)] \leq  \max\{s, t\} e^{\log^2 (d \beta / \sqrt{\gamma} )/ \log (1/\gamma)}\label{eq:cross-upper} \,,
\end{equation}
which notably is constant with respect to $n$.
\end{corollary}
\begin{proof}
This follows from the observation that for any two random variables $X$ and $Y$,
\begin{align*}
    \E[XY] \leq \frac{1}{2}\E[X^2] + \frac{1}{2}\E[Y^2]\,,
\end{align*}
which is due to the fact that for any values $x, y \in \mathbb{R}$, because $(x - y)^2 \geq 0$, we must have $x^2 + y^2 \geq 2 xy$, from which the claim follows.
Hence,
\begin{align*}
    \E[N(s) N(t)] &\leq \frac{1}{2}\E[N(t)^2] + \frac{1}{2} \E[N(s)^2] \\
    &\leq \frac{1}{2} \max \{\E[N(t)^2] , \E[N(s)^2] \} + \frac{1}{2} \max \{\E[N(t)^2] , \E[N(s)^2] \} \\
    &=  \max \{\E[N(t)^2] , \E[N(s)^2] \} \,.
\end{align*} 
From Corollary \ref{corr:quad-upper-o-notation}, the claim then follows.
\end{proof}

\noindent \textbf{Proof of Theorem \thmgaussianconnection}. \quad We prove the claim of the proof in two parts, by noting an upper bound and lower bound on $\E[N(t)]$ which implies the claim of the Theorem.

The upper bound on $\E[N(t)]$ is immediate from Lemma \ref{lem:gaussian-upper}, as it shows
\[ \E[N(t)] \leq (d\beta)^t \gamma^{\frac{t(t-1)}{2}} =  e^{\left(\frac{1}{2} \log \gamma\right) t^2 + \left( \log \frac{d\beta}{\sqrt{\gamma}}\right) t  } \,.\]

We next show a lower bound on $\E[N(t)]$, and proceed by induction.
We claim:
\begin{equation}
    \label{eq:nt-lower-ind-hyp}
    \E[N(t)] \geq e^{\left(\frac{1}{2} \log \gamma\right) t^2 + \left( \log \frac{d\beta}{\sqrt{\gamma}}\right) t  } - O\left( \frac{t^2 e^{\log^2 (d \beta / \sqrt{\gamma} )/ \log (1/\gamma)}}{n}\right) \,,
\end{equation}
where we recall the $O$ notation is taken with respect to the community size $n$.
As a base case, we note that $\E[N(0)] = 1$ due to the initial condition that a single individual is infected. This clearly satisfies the induction hypothesis, as $1 = e^0 - 0$, and $0 = O(d\beta/n)$ trivially.

Next, we assume the induction hypothesis \eqref{eq:nt-lower-ind-hyp}. 
We again recall that given $\mathcal{F}_t = \sigma(I(0), \dots, I(t))$ as the $\sigma-$algebra generated by observations of infections up to time $t$, the conditional distribution of $N(t+1)$ is a binomial random variable with $n - \sum_{s = 0}^t N(s)$ trials and probability of success $\left(1 - \left(1 - \frac{d \beta \gamma^t}{n}\right)^{N(t)} \right)$.
Hence,
\begin{align*}
     \E[N(t+1)] &=  \E[\E[N(t+1)|\mathcal{F}_t] ]\\
     &= \E\left[\left( n- \sum_{s = 0}^t N(s)\right) \left(1 - \left(1 - \frac{d \beta \gamma^t}{n}\right)^{N(t)} \right)\right] \\
     &= \E\left[n\left(1 - \left(1 - \frac{d \beta \gamma^t}{n}\right)^{N(t)} \right) - \left(1 - \left(1 - \frac{d \beta \gamma^t}{n}\right)^{N(t)} \right) \sum_{s = 0}^t N(s)\right] \\
     &\geq \E\left[n \left(1 - 1 + N(t)\frac{d\beta\gamma^t}{n} - \frac{N(t)^2 (d\beta\gamma^t)^2}{2 n^2} \right) - \left(1 - 1 +  N(t)\frac{d\beta\gamma^t}{n}\right) \sum_{s = 0}^t N(s)\right] \,,
\end{align*}
where the last step here holds by noting 
\[ 1 - \frac{d \beta \gamma^t}{n} N(t)  \left(1 - \frac{d \beta \gamma^t}{n}\right)^{N(t)} \leq
\left(1 - \frac{d \beta \gamma^t}{n}\right)^{N(t)}
\leq 1 - N(t)\frac{d\beta\gamma^t}{n} + \frac{N(t)^2 (d\beta\gamma^t)^2}{2 n^2} \,,\]
where the first inequality is due to Bernoulli's inequality and the second follows from noting that for any $x, y > 0$, $(1-x)^y \leq e^{-xy} \leq 1 - xy + \frac{ x^2y^2}{2}$ and applying this inequality with $x = \frac{d \beta \gamma^t}{n}$ and $y = N(t)$.
We then continue, and see from the inequalities on $\E[N(t)],$
\begin{align*}
    \E[N(t+1)] &\geq \E\left[n \left( N(t)\frac{d\beta\gamma^t}{n} - \frac{N(t)^2 (d\beta\gamma^t)^2}{2 n^2} \right) - \left(  N(t)\frac{d\beta\gamma^t}{n}\right) \sum_{s = 0}^t N(s)\right] \\
    &= \E\left[ N(t) d\beta\gamma^t - \frac{N(t)^2 (d\beta\gamma^t)^2}{2 n} - \left(\frac{d\beta\gamma^t}{n}\right) \sum_{s = 0}^t N(t) N(s)\right]\\
    &= \E\left[ N(t) d\beta\gamma^t \right] - \left(\frac{1}{n} \E\left[\frac{N(t)^2 (d\beta\gamma^t)^2}{2} + d\beta\gamma^t \sum_{s = 0}^t N(t) N(s)\right]\right) \,.
\end{align*}
Therefore, from Corollaries  \ref{corr:quad-upper-o-notation} and \ref{corr:gaussian-cross-upper}, which show that there is an upper bound on $\E[N(t)^2]$ and $\E[N(s)N(t)]$ which only depends on $t e^{\log^2 (d \beta / \sqrt{\gamma} )/ \log (1/\gamma)}$, and not $n$ itself, we see that the $t$ cross-terms yield
\begin{equation} \label{eq:lower-cond-bound}
    \E[N(t+1)] \geq d\beta\gamma^t\E\left[ N(t) \right] - O\left(\frac{t^2 e^{\log^2 (d \beta / \sqrt{\gamma} )/ \log (1/\gamma)}}{n}\right) \,.
\end{equation}
Finally, applying the induction hypothesis, we find
\begin{align*}
    \E[N(t+1)] &\geq d\beta\gamma^t\E\left[ N(t) \right] - O\left(\frac{t^2 e^{\log^2 (d \beta / \sqrt{\gamma} )/ \log (1/\gamma)}}{n}\right)\\
    &\geq d\beta\gamma^t \times \left( e^{\left(\frac{1}{2} \log \gamma\right) t^2 + \left( \log \frac{d\beta}{\sqrt{\gamma}}\right) t  } - O\left(\frac{(t-1)^2 e^{\log^2 (d \beta / \sqrt{\gamma} )/ \log (1/\gamma)}}{n}\right) \right) \\
    &\qquad - O\left(\frac{t^2 e^{\log^2 (d \beta / \sqrt{\gamma} )/ \log (1/\gamma)}}{n}\right)\\
    &= e^{\left(\frac{1}{2} \log \gamma\right) (t+1)^2 + \left( \log \frac{d\beta}{\sqrt{\gamma}}\right) (t+1)  } - O\left(\frac{t^2 e^{\log^2 (d \beta / \sqrt{\gamma} )/ \log (1/\gamma)}}{n}\right) \,,
\end{align*}
as desired.

Hence, combining the lower bound in \eqref{eq:nt-lower-ind-hyp} with the upper bound provided in Lemma \ref{lem:gaussian-upper}, we see that $\E[N(t)]$ is tightly characterized up to negative additive term which scales as $O(d\beta/n)$, proving the claim of Theorem \thmgaussianconnection.

\subsection{Proof of Lemma \lemgaussianconcentration}
\label{aa:gaussian-concentration-proof}
To show this claim, we first show the following lemma, which establishes a concentration result on the conditional distribution of $N(t+1)$ given previous observations.
\begin{lemma}
\label{lem:poisson-limit}
Let $\mathcal{F}_t = \sigma(I(0), I(1), \dots, I(t))$ again represent the $\sigma$-algebra of observations generated by previous observations of infected individuals.
Then, for any $\theta > 0$,
\begin{equation}
    \label{eq:limit-mgf}
    \E[e^{\theta N(t+1)}|\mathcal{F}_t] \leq e^{(d\beta \gamma^t N(t))(e^\theta - 1)} \,,
\end{equation}
where we identify the right hand side as the moment generating function of a centered Poisson distribution with parameter $d\beta \gamma^t N(t)$.
\end{lemma}
\begin{proof}
As before, we note that given $\mathcal{F}_t$, the distribution of $N(t+1)$ is a binomial random variable with $n - \sum_{s = 0}^t N(s)$ trials and success probability $1 - (1 - \frac{d \beta \gamma^t}{n})^{N(t)}$.
Hence,
\begin{align*}
    \E[e^{\theta N(t+1)}|\mathcal{F}_t] &= \bigg( \left(1 - \frac{d \beta \gamma^t}{n}\right)^{N(t)} + \Big(1 -  \left(1 - \frac{d \beta \gamma^t}{n}\right)^{N(t)}\Big)e^\theta \bigg)^{n - \sum_{s = 0}^t N(s)} \,.
\end{align*}
This holds for any arbitrary choice of $n$, and any particular history $\mathcal{F}_t$.
Next, we see
\begin{align*}
    \E[e^{\theta N(t+1)}|\mathcal{F}_t] &= \bigg( \left(1 - \frac{d \beta \gamma^t}{n}\right)^{N(t)} + \Big(1 -  \left(1 - \frac{d \beta \gamma^t}{n}\right)^{N(t)}\Big)e^\theta \bigg)^{n}  \\
    &\quad \times \bigg( \left(1 - \frac{d \beta \gamma^t}{n}\right)^{N(t)} + \Big(1 -  \left(1 - \frac{d \beta \gamma^t}{n}\right)^{N(t)}\Big)e^\theta \bigg)^{-\sum_{s = 0}^t N(s)} \\
    &\leq  \bigg( \left(1 - \frac{d \beta \gamma^t}{n}\right)^{N(t)} + \Big(1 -  \left(1 - \frac{d \beta \gamma^t}{n}\right)^{N(t)}\Big)e^\theta \bigg)^{n}\,.
\end{align*}
To show the final inequality above, we first note that $\theta > 0$ implies 
\begin{align*}
&\left(1 - \frac{d \beta \gamma^t}{n}\right)^{N(t)} + \Big(1 -  \left(1 - \frac{d \beta \gamma^t}{n}\right)^{N(t)}\Big)e^\theta \\
&\qquad= 1 - \left( 1 - \left(1 - \frac{d \beta \gamma^t}{n}\right)^{N(t)} \right) + \Big(1 -  \left(1 - \frac{d \beta \gamma^t}{n}\right)^{N(t)}\Big)e^\theta \\
&\qquad = 1+ \Big(1 -  \left(1 - \frac{d \beta \gamma^t}{n}\right)^{N(t)}\Big)(e^\theta - 1) \\
&\qquad \geq 1\,,
\end{align*}
and hence the inequality follows because $-\sum_{s = 0}^t N(s) < 0$.
Next, we see that algebraically manipulating the upper bound yields
\begin{align*}
    \E[e^{\theta N(t+1)}|\mathcal{F}_t] &\leq \bigg(1 + \Big(1 -  \left(1 - \frac{d \beta \gamma^t}{n}\right)^{N(t)}\Big)(e^\theta - 1) \bigg)^{n} \\
    &\leq \bigg(1 + \Big(1  - 1 + \frac{d \beta \gamma^t}{n} N(t) \Big)(e^\theta - 1) \bigg)^{n} \quad \text{(Bernoulli's inequality)} \\
    &= \bigg(1 + \Big(\frac{d \beta \gamma^t}{n} N(t) \Big)(e^\theta - 1) \bigg)^{n} \\
    &\leq e^{(d\beta \gamma^t N(t))(e^\theta - 1)} \,,
\end{align*}
proving the Lemma as desired.
\end{proof}
Hence, using the tail bounds of the Poisson distribution, the following Corollary can be shown:
\begin{corollary} \label{corr:ind-concentration}
For any $x > 0$,
\[  \P \left( N(t+1) > (d \beta \gamma^t + x) N(t) ~|~ \mathcal{F}_t \right) \leq e^{-\min \{\frac{x^2}{4d \beta \gamma^t}, \frac{x}{4} \}} \,,\]
where we recall $\mathcal{F}_t = \sigma(I(0), \dots I(t))$ is the $\sigma$-algebra generated by observations of infections up to time $t$ and $N(t) = |I(t)|$ is the number of infected individuals at time $t$.
\end{corollary}
\begin{proof}
From e.g. \cite{zhang2020non} on the results for tail bounds of the Poisson distribution, we immediately see that because Lemma \ref{lem:poisson-limit} shows that the conditional moment generating function of $N(t+1)$ is dominated by that of a Possion random variable, then for $k \geq 1$,
\[\P \left( |N(t+1) - \E[N(t+1)|\mathcal{F}_t, N(t) = k| > x ~|~ \mathcal{F}_t, N(t) = k\right) \leq e^{-\frac{x^2}{2d \beta \gamma^t k} h\left(\frac{x}{2d \beta \gamma^t k}\right)} \,,\]
where, $h(u) = 2\frac{(1+u) \ln(1+u) - u}{u^2}$.
Notably, for $u \geq 0$, $h(u) \geq \frac{1}{1+u}$ as can be verified by considering the function $(1+u)h(u)$ (see \cite{cannone2017poisson} for details).

Hence,
\begin{align*}
    \P \left( \bigg\vert N(t+1) - \E[N(t+1)|\mathcal{F}_t, N(t) = k]\bigg\vert > x ~|~ \mathcal{F}_t, N(t)= k \right) &\leq e^{-\frac{x^2}{2(d \beta \gamma^t k + x)}} \\
    &\leq e^{-\min \{\frac{x^2}{4d \beta \gamma^t k}, \frac{x}{4} \}} \,,
\end{align*}
which holds because $a + b \leq 2\max\{a, b\}$ for $a, b \geq 0$.

Further, replacing $x$ with $xk$ and rewriting the event of interest, we see
\begin{align}
    \P \left( \bigg\vert N(t+1) - \E\left[N(t+1)|\mathcal{F}_t, N(t) = k\right] \bigg\vert > xk ~|~ \mathcal{F}_t, N(t)= k \right) &\leq e^{-\min \{\frac{kx^2}{4d \beta \gamma^t}, \frac{kx}{4} \}} \nonumber\\
    &\leq e^{-\min \{\frac{x^2}{4d \beta \gamma^t}, \frac{x}{4} \}} \label{eq:single-prop} \,,
\end{align}
since $k \geq 1$.

From the proof of Theorem \thmgaussianconnection, and in particular equations \eqref{eq:upper-cond-bound} and \eqref{eq:lower-cond-bound} which establish upper and lower bounds on the conditional expectation of $N(t+1)$ given $N(t)$, we see that for $k \geq 1$
\[ \E\left[N(t+1)|\mathcal{F}_t, N(t) = k\right] = k d \beta \gamma^t - c_{t, k}(n) \,, \]
where $c_{t, k} (n) = O\left( \frac{1}{n}\right)$ is a non-negative function. 

Moreover, in the case where $k = 0$, we see that if $N(t) = 0$, then $N(t+1) = 0$ with probability 1, as there are no infections at time $t$ to spread at time $t+1$.
Hence, 
\[\P \left( N(t+1) > (d \beta \gamma^t + x) N(t) ~|~ \mathcal{F}_t, N(t) = 0 \right) = 0 \,. \]
Therefore,
\begin{align*}
    &\P \left( N(t+1) > (d \beta \gamma^t + x) N(t) ~|~ \mathcal{F}_t \right) \\
    &\quad= \sum_{k = 0}^\infty \P \left( N(t+1) > (d \beta \gamma^t + x) N(t) ~|~ \mathcal{F}_t, N(t) = k \right) \P(N(t) = k) \\
    &\quad= 0 + \sum_{k = 1}^\infty \P \left( N(t+1) > (d \beta \gamma^t + x) N(t) ~|~ \mathcal{F}_t, N(t) = k \right) \P(N(t) = k) \\
    &\quad \leq \sum_{k = 1}^\infty \P \left( N(t+1) > (d \beta \gamma^t + x) N(t) - c_{t, k}(n) ~|~ \mathcal{F}_t, N(t) = k \right) \P(N(t) = k) \quad (c_{t,k} (n) \geq 0)\\
    &\quad\leq \sum_{k = 1}^\infty \P \left( \bigg\vert N(t+1) - [d \beta \gamma^t N(t) - c_{t, k}(n)] \bigg\vert > xk ~|~ \mathcal{F}_t, N(t) = k \right) \P(N(t) = k) \\
    &\quad= \sum_{k = 1}^\infty \P \left( \bigg\vert N(t+1) - \E\left[N(t+1)|\mathcal{F}_t, N(t) = k\right] \bigg\vert > xk ~|~ \mathcal{F}_t, N(t)= k \right) \P(N(t) = k) \\
    &\quad\leq \sum_{k = 1}^\infty e^{-\min \{\frac{x^2}{4d \beta \gamma^t}, \frac{x}{4} \}} \P(N(t) = k)\\ 
    &\quad\leq e^{-\min \{\frac{x^2}{4d \beta \gamma^t}, \frac{x}{4} \}} \,,
\end{align*}
where the final step follows from the normalization of probability measure such that $\sum_{k = 0}^\infty \P(N(t) = k) = 1$.
\end{proof}

\noindent \textbf{Proof of Lemma \lemgaussianconcentration. } \quad 
We first note the result of Corollary \ref{corr:ind-concentration} when $x = \epsilon d \beta \gamma^t$ yields
\begin{align*}
\P \big( N(t+1) > (d \beta \gamma^t + \epsilon d\beta \gamma^t) &N(t) ~|~ \mathcal{F}_t \big)\\
&\quad= \P \left( N(t+1) > (d \beta \gamma^t) (1+\epsilon) N(t) ~|~ \mathcal{F}_t \right)\\
&\quad\leq e^{-\min \{\frac{(\epsilon d \beta \gamma^t)^2}{4d \beta \gamma^t}, \frac{\epsilon d \beta \gamma^t}{4} \}} \\
&\quad= e^{-\min \{ \epsilon^2, \epsilon\} d \beta \gamma^t/4} \,.
\end{align*}

Next, we see:
\begin{align*}
    &\P\left( \bigcap_{s = 0}^{t-1} \left\lbrace N(s+1) \leq (d \beta \gamma^t) (1+\epsilon) N(s)  \right\rbrace \right) \\
    &\qquad = \P\left( \bigcap_{s = 0}^{t-1} \left\lbrace N(s+1) \leq (d \beta \gamma^t) (1+\epsilon) N(s)  \right\rbrace \right) \\
    &\qquad = \prod_{s = 0}^{t-1}\P\left( N(s+1) \leq (d \beta \gamma^t) (1+\epsilon) N(s)  ~\big\vert~ \bigcap_{\tau = 0}^{s-1} N(s+1) \leq (d \beta \gamma^t) (1+\epsilon) N(s)  \right)\\
    &\qquad = \prod_{s = 0}^{t-1}\left(1 -  \P\left( N(s+1) > (d \beta \gamma^t) (1+\epsilon) N(s) +  ~\big\vert~ \bigcap_{\tau = 0}^{s-1} N(s+1) \leq (d \beta \gamma^t) (1+\epsilon) N(s)  \right) \right)\\
    &\qquad \geq 1 - \sum_{s = 0}^{t-1} \P\left( N(s+1) > (d \beta \gamma^t) (1+\epsilon) N(s) ~\big\vert~ \bigcap_{\tau = 0}^{s-1} N(s+1) \leq (d \beta \gamma^t) (1+\epsilon) N(s) \right)\\
    &\qquad \geq 1 - \sum_{s = 0}^{t-1} e^{-\min \{ \epsilon^2, \epsilon\} d \beta \gamma^s/4} \,,
\end{align*}

Finally, to prove the claim, we see that when for all $s \leq t-1$, $N(s+1) \leq (d \beta \gamma^s) (1+\epsilon) N(s)$, we also have $N(t) \leq (1+\epsilon)^t (d\beta)^t \gamma^{t(t-1)/2}$. 
This follows similarly to the proof of \eqref{eq:lin-upper} in Lemma \ref{lem:gaussian-upper}.
Hence,
\begin{align*}
    &\P\left( \bigcap_{s = 0}^{t-1} \left\lbrace N(s+1) \leq (d \beta \gamma^t) (1+\epsilon) N(s) \right\rbrace \right) \\
    &\qquad \leq \P\left(N(t) \leq (1+\epsilon)^t (d\beta)^t \gamma^{t(t-1)/2} \right) \,.
\end{align*} 
Since $(1+\epsilon)^t (d\beta)^t \gamma^{t(t-1)/2}  = e^{\left(\frac{1}{2} \log \gamma\right) t^2 + \left( \log \frac{d\beta}{\sqrt{\gamma}} (1+\epsilon) \right) t}$, this implies
\begin{align*}
    \P\left(N(t) \leq e^{\left(\frac{1}{2} \log \gamma\right) t^2 + \left( \log \frac{d\beta}{\sqrt{\gamma}} (1+\epsilon) \right) t}  \right) \geq 1 - \sum_{s = 0}^{t-1} e^{-\min \{ \epsilon^2, \epsilon\} d \beta \gamma^s/4} \,,
\end{align*}
proving the claim of the theorem.

\subsection{Proof of Theorem \thmtimetilnext}
\label{aa:spread-time-proof}
First, we introduce some notation to denote the spread of infection in the first community and characterize the event of interest.
We let $N_1(t)$ represent the number of infections in community 1 at time $t$, and note that due to the initial condition, $N_1(0) = 1$.
Let $A_t$ be the event that there are no cases in community $2$ at time $t$.
Hence, the event of interest $T > t$, which refers to the event that there are no infections in community 2 prior to time $t$, is equivalent to $\cap_{s = 0}^t A_s$.

The proof has three parts: First, we show a general lower bound on $\P(T > t)$ which illustrates that there are two key sufficient conditions to providing the lower bound on the probability: upper bounding the number of cases in Community 1, and ensuring $d_{out}$ is low enough to ensure a case in Community 2 is unlikely.
The second part of the proof is to show that Assumption \eqsepconditionone ensures the upper bound on number of cases in community 1, and the third part of the proof shows that Assumption \eqsepconditiontwo ensures $d_{out}$ is sufficiently small.\

\noindent \textbf{General lower bound for $\P(T > t)$.} \quad 
We begin by noting the following observation, which holds for any $t$:
\begin{align*}
    \P(T > t) &= \P\left( \bigcap_{s = 0}^t A_s \right)\\
    &\geq \P\left( \bigcap_{s = 0}^t A_s \cap \Big(  N_1(s+1) \leq (2 d_{in} \beta \gamma^s ) N_1(s) \Big) \right) \,.
\end{align*}
The first step here follows from noting the definition of $A_s$ as the event that no individuals in community 2 are infected by those in community 1 at time $s$, and hence $T > t$ occurs if and only if no individual from community 2 has been infected by an individual from community 1 for all times $s = 0$ until $s = t$.
The inequality then follows from the monotonicity of probability measure.
Next, we see that the definition of conditional probability yields
\begin{align*}
    &\P(T > t) \\
    &\quad \geq \P\left( \bigcap_{s = 0}^t A_s \cap \left(  N_1(s+1) \leq (2 d_{in} \beta \gamma^s ) N_1(s) \right) \right)\\
    &\quad = \prod_{s = 0}^t \P\left( A_s \cap \Big(  N_1(s+1) \leq (2 d_{in} \beta \gamma^s ) N_1(s) \Big) ~\bigg\vert~  \bigcap_{s' = 0}^{s - 1} A_{s'} \cap \Big(  N_1(s'+1) \leq (2 d_{in} \beta \gamma^{s'} ) N_1(s') \Big)\right) \\
    &\quad = \prod_{s = 0}^t \P\left(N_1(s+1) \leq (2 d_{in} \beta \gamma^s ) N_1(s) ~\bigg\vert~  A_s \cap \bigcap_{s' = 0}^{s - 1} A_{s'} \cap \Big(  N_1(s'+1) \leq (2 d_{in} \beta \gamma^{s'} ) N_1(s') \Big)\right) \\
    &\qquad \qquad \times  \P\left( A_s  ~\bigg\vert~ \bigcap_{s' = 0}^{s-1} A_{s'} \cap \Big(  N_1(s'+1) \leq (2 d_{in} \beta \gamma^{s'} ) N_1(s') \Big)\right) \,.
\end{align*}
Further algebraic manipulation reveals
\begin{align*}
    &\P(T > t) \\
    &\geq \prod_{s = 0}^t \P\left(N_1(s+1) \leq (2 d_{in} \beta \gamma^s ) N_1(s) ~\bigg\vert~  A_s \cap \bigcap_{s' = 0}^{s - 1} A_{s'} \cap \Big(  N_1(s'+1) \leq (2 d_{in} \beta \gamma^{s'} ) N_1(s') \Big)\right) \\
    &\qquad \qquad \times  \prod_{s = 0}^t \P\left( A_s  ~\bigg\vert~ \bigcap_{s' = 0}^{s-1} A_{s'} \cap \Big(  N_1(s'+1) \leq (2 d_{in} \beta \gamma^{s'} ) N_1(s') \Big)\right) \\
    &= \prod_{s = 0}^t \left( 1 - \P\left(N_1(s+1) > (2 d_{in} \beta \gamma^s ) N_1(s) ~\bigg\vert~  A_s \cap \bigcap_{s' = 0}^{s - 1} A_{s'} \cap \Big(  N_1(s'+1) \leq (2 d_{in} \beta \gamma^{s'} ) N_1(s') \Big)\right) \right)\\
    &\qquad \qquad \times  \prod_{s = 0}^t \P\left( A_s  ~\bigg\vert~ \bigcap_{s' = 0}^{s-1} A_{s'} \cap \Big(  N_1(s'+1) \leq (2 d_{in} \beta \gamma^{s'} ) N_1(s') \Big)\right) \\
    &\geq \left(1 - \sum_{s = 0}^t  \P\left(N_1(s+1) > (2 d_{in} \beta \gamma^s ) N_1(s) ~\bigg\vert~  A_s \cap \bigcap_{s' = 0}^{s - 1} A_{s'} \cap \Big(  N_1(s'+1) \leq (2 d_{in} \beta \gamma^{s'} ) N_1(s') \Big)\right) \right)\\
    &\qquad \qquad \times  \prod_{s = 0}^t \P\left( A_s  ~\bigg\vert~ \bigcap_{s' = 0}^{s-1} A_{s'} \cap \Big(  N_1(s'+1) \leq (2 d_{in} \beta \gamma^{s'} ) N_1(s') \Big)\right) \,,
\end{align*}
The final inequality follows from the following fact: for values $x_1, \dots, x_t \in [0, 1)$, $\prod_{s = 1}^t (1 - x_s) = \geq 1 - \sum_{s = 1}^t x_s$.
This fact is a generalization of Bernoulli's inequality and can be proven by induction.
Hence, for any $t$
\begin{align}
    \P(T > t) &\geq \left(1 - \sum_{s = 0}^t  \P\left(N_1(s+1) > (2 d_{in} \beta \gamma^s ) N_1(s) ~\bigg\vert~  A_s \cap \bigcap_{s' = 0}^{s - 1} A_{s'} \cap \Big(  N_1(s'+1) \leq (2 d_{in} \beta \gamma^{s'} ) N_1(s') \Big)\right) \right) \nonumber \\
    &\qquad \times  \prod_{s = 0}^t \P\left( A_s  ~\bigg\vert~ \bigcap_{s' = 0}^{s-1} A_{s'} \cap \Big(  N_1(s'+1) \leq (2 d_{in} \beta \gamma^{s'} ) N_1(s') \Big)\right) \label{eq:wait-lower-bound-general}
\end{align}
\null\\
\noindent\textbf{Bounding Cases in Community 1.} \quad Similar to Corollary \ref{corr:ind-concentration}, and defining $\mathcal{F}_t^1 = \sigma(I_1(0), \dots, I_1(t))$ as the observations of infections within the first community up to time $t$, we can show that
\begin{align*}
    \P \left( N_1(t+1) > (d_{in} \beta \gamma^t + x) N_1(t) ~\big\vert~ \mathcal{F}_t^1, \bigcap_{s = 0}^t A_t \right) 
    &\leq e^{-\min \{\frac{x^2}{4d \beta \gamma^t}, \frac{x}{4} \}} \,.
\end{align*}
The only difference between this claim and that of Corollary \ref{corr:ind-concentration} is that we condition on knowing there are no cases in the second community up to time $t$, in which case the problem reduces to that of the single community model.
Hence, the proof of this claim follows from that of Corollary \ref{corr:ind-concentration}.

We can then set $x = d_{in} \beta \gamma^t$, and find
\begin{align*}
     \P \left( N_1(t+1) > (2 d_{in} \beta \gamma^t) N_1(t) ~\big\vert~ \mathcal{F}_t^1, \bigcap_{s = 0}^t A_t \right) 
    &\leq e^{-d_{in} \beta \gamma^t / 4}
\end{align*}
This claim holds for any choice of $t$, and hence it holds for all $t \leq C_1 - \log 20 / \log \frac{1}{\gamma} = - \left( \log \left( \frac{d_{in}\beta}{20\sqrt{\gamma}}\right) \right) / \log \gamma$.
Hence,
\begin{align*}
    \sum_{t = 0}^{\lfloor - \left( \log \left( \frac{d_{in}\beta}{20\sqrt{\gamma}}\right) \right) / \log \gamma \rfloor} &\P\left( N_1(t+1) > (2d_{in} \beta \gamma^t ) N_1(t)~|~ \mathcal{F}_t^1, \bigcap_{s = 0}^t A_t \right) \\
    &\leq \sum_{t = 0}^{\lfloor - \left( \log \left( \frac{d_{in}\beta}{20\sqrt{\gamma}}\right) \right) / \log \gamma \rfloor} e^{-d_{in} \beta \gamma^t / 4} \\
    &\leq \sum_{t = 0}^{\lfloor - \left( \log \left( \frac{d_{in}\beta}{20\sqrt{\gamma}}\right) \right) / \log \gamma \rfloor}e^{-d_{in} \beta \gamma^{- \left( \log \left( \frac{d_{in}\beta}{20\sqrt{\gamma}}\right) \right) / \log \gamma} / 4} \quad \left(t \leq \lfloor - \left( \log \left( \frac{d_{in}\beta}{20\sqrt{\gamma}}\right) \right) / \log \gamma \rfloor \right)\\
    &= \lfloor - \left( \log \left( \frac{d_{in}\beta}{20\sqrt{\gamma}}\right) \right) / \log \gamma \rfloor \times e^{-d_{in} \beta e^{-\log \left( \frac{d_{in}\beta}{20\sqrt{\gamma}}\right)} / 4} \\
    &= \lfloor - \left( \log \left( \frac{d_{in}\beta}{20\sqrt{\gamma}}\right) \right) / \log \gamma \rfloor \times e^{-d_{in} \beta  \frac{20\sqrt{\gamma}}{d_{in}\beta} / 4} \\
    &= \lfloor - \left( \log \left( \frac{d_{in}\beta}{20\sqrt{\gamma}}\right) \right) / \log \gamma \rfloor \times e^{-5\sqrt{\gamma}} \,.
\end{align*}
We also see, from assumption ~\eqsepconditionone made in the Theorem statement, that
\begin{align*}
    C_1(d_{in},\beta, \gamma)  &< \delta e^{5\sqrt{\gamma}} + \log 20 / \log \frac{1}{\gamma} \qquad \text{(Assumption \eqsepconditionone)} \\
    \implies & - \log \left(\frac{d_{in} \beta}{\sqrt{\gamma}} \right) / \log \gamma < \delta e^{5\sqrt{\gamma}} - \log 20 / \log \gamma \\
    \implies & - \log \left(\frac{d_{in} \beta}{20 \sqrt{\gamma}} \right) / \log \gamma < \delta e^{5\sqrt{\gamma}}  \\
    \implies & \left(- \log \left(\frac{d_{in} \beta}{20 \sqrt{\gamma}} \right) / \log \gamma\right) \times e^{-5\sqrt{\gamma}} < \delta \\
    \implies & \left\lfloor - \left( \log \left( \frac{d_{in}\beta}{20\sqrt{\gamma}}\right) \right) / \log \gamma \right\rfloor \times e^{-5\sqrt{\gamma}} \leq \delta \,.
\end{align*}
Therefore,
\begin{equation}
    \label{eq:high-case-delta}
    \sum_{t = 0}^{\lfloor - \left( \log \left( \frac{d_{in}\beta}{20\sqrt{\gamma}}\right) \right) / \log \gamma \rfloor} \P\left( N_1(t+1) > (2d_{in} \beta \gamma^t ) N_1(t) ~|~  \mathcal{F}_t^1, \bigcap_{s = 0}^t A_t \right) \leq \delta \,.
\end{equation}
That is, we see that with high probability, the number of cases in community 1 is small enough as to not create too many chances for infection to spread to community 2.

Returning to the bound provided by \eqref{eq:wait-lower-bound-general}, we now see that, for $t = \lfloor - \left( \log \left( \frac{d_{in}\beta}{20\sqrt{\gamma}}\right) \right) / \log \gamma \rfloor$ the term 
\[ \sum_{s = 0}^t  \P\left(N_1(s+1) > (2 d_{in} \beta \gamma^s ) N_1(s) ~\bigg\vert~  A_s \cap \bigcap_{s' = 0}^{s - 1} A_{s'} \cap \Big(  N_1(s'+1) \leq (2 d_{in} \beta \gamma^{s'} ) N_1(s') \Big)\right) \]
can be bounded using \eqref{eq:high-case-delta}, and hence what remains is to bound 
\[ \prod_{s = 0}^t \P\left( A_s  ~\bigg\vert~ \bigcap_{s' = 0}^{s-1} A_{s'} \cap \Big(  N_1(s'+1) \leq (2 d_{in} \beta \gamma^{s'} ) N_1(s') \Big)\right) \,.\]
We provide this bound by imposing reasonable requirements on $d_{out}$ and $n$. \\

\noindent\textbf{Requirements on $d_{out}$ and $n$.} \quad 
In order to bound the product of probabilities above, we first note that
\begin{align}
    \P\left(A_t ~|\mathcal{F}_t^1, \bigcap_{s = 1}^{t-1} A_{s} \right)
    &=\left(1 - \frac{d_{out}\beta}{n} \right)^{N_1(t-1) n} \,. \label{eq:time-t-no-infect-obs}
\end{align}
That is, for there to be no cases at time $t$ given that there are no cases prior to time $t$ and given the information about previous observations in community 1, it must be the case that each of the $N_1(t-1)$ infected individuals in community 1 is unable to infect anyone in community 2.
Since there are $N_1(t-1)$ infected individuals in community 1 at time $t-1$, there are $n$ susceptible individuals in community 2, the probability of an infection not occurring between a pair of indivuduals is $1 - \frac{d_{out}\beta}{n}$, and all infection events are independent, we then see that the observation \eqref{eq:time-t-no-infect-obs} must hold.

Therefore, we must also have,
\begin{align*}
    \prod_{s = 0}^t \P\left( A_s  ~\bigg\vert~ \bigcap_{s' = 0}^{s-1} A_{s'} , \mathcal{F}_t^1 \right) &= \prod_{s = 1}^t \left(1 - \frac{d_{out}\beta}{n} \right)^{N_1(s-1) n} && (\P(A_0) = 1) \\
    &= e^{\sum_{s = 1}^t N_1(s-1) n \log \left(1 - \frac{d_{out}\beta}{n} \right)} \\
    &\geq e^{\sum_{s = 1}^t N_1(s-1) n  \frac{ - \frac{d_{out}\beta}{n}}{\left(1 - \frac{d_{out}\beta}{n} \right)}} &&\left(\frac{-x}{1-x} \leq \log (1 - x) \right) \\
    &\geq e^{\sum_{s = 1}^t 2 N_1(s-1) n  (- \frac{d_{out}\beta}{n})} &&(\text{Assumption \eqsepconditionthree},~n \geq 2 \beta d_{out}) \\
    &= e^{-2 d_{out}\beta \sum_{s = 1}^t N_1(s-1) } \\
    &= e^{-2 d_{out}\beta \sum_{s = 0}^{t-1} N_1(s) }
\end{align*}

We now note that, for $t^* = \lfloor - \left( \log \left( \frac{d_{in}\beta}{20\sqrt{\gamma}}\right) \right) / \log \gamma \rfloor$, we have $e^{-2 d_{out}\beta \sum_{s = 0}^{t^*-1} N_1(s) } \geq 1 - \delta$. 
This holds because
\begin{equation} 
    e^{-2 d_{out}  \beta\sum_{s = 0}^{t^*-1} (2d\beta)^s \gamma^{s(s+1)/2}} \geq 1 - \delta \,,
\end{equation}
is equivalent to
\[ d_{out}  \leq \frac{1}{2 \beta \sum_{s = 0}^{\lfloor - \left( \log \left( \frac{d_{in}\beta}{20\sqrt{\gamma}}\right) \right) / \log \gamma \rfloor - 1} (2d_{in}\beta)^s \gamma^{s(s+1)/2}}\log \left(\frac{1}{1 - \delta}\right) \,,\]
which is the exact condition of \eqsepconditionthree. 

Hence, we see
\begin{align}
    \prod_{s = 0}^t \P\left( A_s  ~\bigg\vert~ \bigcap_{s' = 0}^{s-1} A_{s'} \cap \Big(  N_1(s'+1) \leq (2 d_{in} \beta \gamma^{s'} ) N_1(s') \Big)\right) \geq 1 - \delta \,. \label{eq:d-out-small}
\end{align}
\null \\
\noindent\textbf{Proof of Theorem \thmtimetilnext.} \quad To complete the proof, we note:
\begin{align*}
    &\P\left(T > \lfloor - \left( \log \left( \frac{d_{in}\beta}{20\sqrt{\gamma}}\right) \right) / \log \gamma \rfloor \right) \\
    &\quad \geq \Bigg(1 - \sum_{s = 0}^{\lfloor - \left( \log \left( \frac{d_{in}\beta}{20\sqrt{\gamma}}\right) \right) / \log \gamma \rfloor} \P\bigg(N_1(s+1) > (2 d_{in} \beta \gamma^s ) N_1(s) ~\bigg\vert~  \\
    &\hspace{16em} A_s \cap \bigcap_{s' = 0}^{s - 1} A_{s'} \cap \Big(  N_1(s'+1) \leq (2 d_{in} \beta \gamma^{s'} ) N_1(s') \Big)\bigg) \Bigg)  \\
    &\qquad \times  \prod_{s = 0}^{\lfloor - \left( \log \left( \frac{d_{in}\beta}{20\sqrt{\gamma}}\right) \right) / \log \gamma \rfloor} \P\left( A_s  ~\bigg\vert~ \bigcap_{s' = 0}^{s-1} A_{s'} \cap \Big(  N_1(s'+1) \leq (2 d_{in} \beta \gamma^{s'} ) N_1(s') \Big)\right) \qquad \eqref{eq:wait-lower-bound-general} \\
    &\quad\geq (1 - \delta) \times \prod_{s = 0}^{\lfloor - \left( \log \left( \frac{d_{in}\beta}{20\sqrt{\gamma}}\right) \right) / \log \gamma \rfloor} \P\left( A_s  ~\bigg\vert~ \bigcap_{s' = 0}^{s-1} A_{s'} \cap \Big(  N_1(s'+1) \leq (2 d_{in} \beta \gamma^{s'} ) N_1(s') \Big)\right) \qquad \eqref{eq:high-case-delta} \\
    &\quad \geq (1-\delta) (1-\delta) \qquad \eqref{eq:d-out-small} \\
    &\quad = 1 - 2 \delta + \delta^2 \geq 1 - 2 \delta \,,
\end{align*}
proving the claim of Theorem \thmtimetilnext.

\end{appendix}

\bibliographystyle{imsart-nameyear} 
\bibliography{main.bib}       

\begin{thebibliography}{41}

\bibitem[\protect\citeauthoryear{Acemoglu et~al.}{2020}]{acemoglu2020multi}
\begin{btechreport}[author]
\bauthor{\bsnm{Acemoglu},~\bfnm{Daron}\binits{D.}},
  \bauthor{\bsnm{Chernozhukov},~\bfnm{Victor}\binits{V.}},
  \bauthor{\bsnm{Werning},~\bfnm{Iv{\'a}n}\binits{I.}} \AND
  \bauthor{\bsnm{Whinston},~\bfnm{Michael~D}\binits{M.~D.}}
(\byear{2020}).
\btitle{A multi-risk SIR model with optimally targeted lockdown}
\btype{Technical Report},
\bpublisher{National Bureau of Economic Research}.
\end{btechreport}
\endbibitem

\bibitem[\protect\citeauthoryear{Ahn and Hassibi}{2013}]{ahn2013global}
\begin{binproceedings}[author]
\bauthor{\bsnm{Ahn},~\bfnm{Hyoung~Jun}\binits{H.~J.}} \AND
  \bauthor{\bsnm{Hassibi},~\bfnm{Babak}\binits{B.}}
(\byear{2013}).
\btitle{Global dynamics of epidemic spread over complex networks}.
In \bbooktitle{52nd IEEE Conference on Decision and Control}
\bpages{4579--4585}.
\bpublisher{IEEE}.
\end{binproceedings}
\endbibitem

\bibitem[\protect\citeauthoryear{Aleta et~al.}{2020}]{aleta2020modelling}
\begin{barticle}[author]
\bauthor{\bsnm{Aleta},~\bfnm{Alberto}\binits{A.}},
  \bauthor{\bsnm{Martin-Corral},~\bfnm{David}\binits{D.}},
  \bauthor{\bparticle{y} \bsnm{Piontti},~\bfnm{Ana~Pastore}\binits{A.~P.}},
  \bauthor{\bsnm{Ajelli},~\bfnm{Marco}\binits{M.}},
  \bauthor{\bsnm{Litvinova},~\bfnm{Maria}\binits{M.}},
  \bauthor{\bsnm{Chinazzi},~\bfnm{Matteo}\binits{M.}},
  \bauthor{\bsnm{Dean},~\bfnm{Natalie~E}\binits{N.~E.}},
  \bauthor{\bsnm{Halloran},~\bfnm{M~Elizabeth}\binits{M.~E.}},
  \bauthor{\bsnm{Longini~Jr},~\bfnm{Ira~M}\binits{I.~M.}},
  \bauthor{\bsnm{Merler},~\bfnm{Stefano}\binits{S.}} \betal{et~al.}
(\byear{2020}).
\btitle{Modelling the impact of testing, contact tracing and household
  quarantine on second waves of COVID-19}.
\bjournal{Nature Human Behaviour}
\bvolume{4}
\bpages{964--971}.
\end{barticle}
\endbibitem

\bibitem[\protect\citeauthoryear{Brauer, Castillo-Chavez and
  Feng}{2019}]{brauer2019mathematical}
\begin{bbook}[author]
\bauthor{\bsnm{Brauer},~\bfnm{Fred}\binits{F.}},
  \bauthor{\bsnm{Castillo-Chavez},~\bfnm{Carlos}\binits{C.}} \AND
  \bauthor{\bsnm{Feng},~\bfnm{Zhilan}\binits{Z.}}
(\byear{2019}).
\btitle{Mathematical models in epidemiology}.
\bpublisher{Springer}.
\end{bbook}
\endbibitem

\bibitem[\protect\citeauthoryear{Bregman and Langmuir}{1990}]{bregman1990farr}
\begin{barticle}[author]
\bauthor{\bsnm{Bregman},~\bfnm{Dennis~J}\binits{D.~J.}} \AND
  \bauthor{\bsnm{Langmuir},~\bfnm{Alexander~D}\binits{A.~D.}}
(\byear{1990}).
\btitle{Farr's law applied to AIDS projections}.
\bjournal{Jama}
\bvolume{263}
\bpages{1522--1525}.
\end{barticle}
\endbibitem

\bibitem[\protect\citeauthoryear{Canonne}{2017}]{cannone2017poisson}
\begin{btechreport}[author]
\bauthor{\bsnm{Canonne},~\bfnm{Clément}\binits{C.}}
(\byear{2017}).
\btitle{A short note on Poisson tail bounds}
\btype{Technical Report}.
\end{btechreport}
\endbibitem

\bibitem[\protect\citeauthoryear{Chakrabarti
  et~al.}{2008}]{chakrabarti2008epidemic}
\begin{barticle}[author]
\bauthor{\bsnm{Chakrabarti},~\bfnm{Deepayan}\binits{D.}},
  \bauthor{\bsnm{Wang},~\bfnm{Yang}\binits{Y.}},
  \bauthor{\bsnm{Wang},~\bfnm{Chenxi}\binits{C.}},
  \bauthor{\bsnm{Leskovec},~\bfnm{Jurij}\binits{J.}} \AND
  \bauthor{\bsnm{Faloutsos},~\bfnm{Christos}\binits{C.}}
(\byear{2008}).
\btitle{Epidemic thresholds in real networks}.
\bjournal{ACM Transactions on Information and System Security (TISSEC)}
\bvolume{10}
\bpages{1--26}.
\end{barticle}
\endbibitem

\bibitem[\protect\citeauthoryear{Chandrasekhar
  et~al.}{2020}]{chandrasekhar2020interacting}
\begin{barticle}[author]
\bauthor{\bsnm{Chandrasekhar},~\bfnm{Arun~G}\binits{A.~G.}},
  \bauthor{\bsnm{Goldsmith-Pinkham},~\bfnm{Paul~S}\binits{P.~S.}},
  \bauthor{\bsnm{Jackson},~\bfnm{Matthew~O}\binits{M.~O.}} \AND
  \bauthor{\bsnm{Thau},~\bfnm{Samuel}\binits{S.}}
(\byear{2020}).
\btitle{Interacting regional policies in containing a disease}.
\bjournal{Available at SSRN}.
\end{barticle}
\endbibitem

\bibitem[\protect\citeauthoryear{Chen et~al.}{2021}]{chen2021numerical}
\begin{barticle}[author]
\bauthor{\bsnm{Chen},~\bfnm{Xiaowei}\binits{X.}},
  \bauthor{\bsnm{Li},~\bfnm{Jing}\binits{J.}},
  \bauthor{\bsnm{Xiao},~\bfnm{Chen}\binits{C.}} \AND
  \bauthor{\bsnm{Yang},~\bfnm{Peilin}\binits{P.}}
(\byear{2021}).
\btitle{Numerical solution and parameter estimation for uncertain SIR model
  with application to COVID-19}.
\bjournal{Fuzzy Optimization and Decision Making}
\bvolume{20}
\bpages{189--208}.
\end{barticle}
\endbibitem

\bibitem[\protect\citeauthoryear{Chernozhukov, Kasaha and
  Schrimpf}{2020}]{chernozhukov2020causal}
\begin{barticle}[author]
\bauthor{\bsnm{Chernozhukov},~\bfnm{Victor}\binits{V.}},
  \bauthor{\bsnm{Kasaha},~\bfnm{Hiroyuki}\binits{H.}} \AND
  \bauthor{\bsnm{Schrimpf},~\bfnm{Paul}\binits{P.}}
(\byear{2020}).
\btitle{Causal impact of masks, policies, behavior on early COVID-19 pandemic
  in the US}.
\bjournal{arXiv preprint arXiv:2005.14168}.
\end{barticle}
\endbibitem

\bibitem[\protect\citeauthoryear{Dandekar and
  Barbastathis}{2020}]{dandekar2020quantifying}
\begin{barticle}[author]
\bauthor{\bsnm{Dandekar},~\bfnm{Raj}\binits{R.}} \AND
  \bauthor{\bsnm{Barbastathis},~\bfnm{George}\binits{G.}}
(\byear{2020}).
\btitle{Quantifying the effect of quarantine control in Covid-19 infectious
  spread using machine learning}.
\bjournal{medRxiv}.
\end{barticle}
\endbibitem

\bibitem[\protect\citeauthoryear{Darakjy et~al.}{2014}]{darakjy2014applying}
\begin{barticle}[author]
\bauthor{\bsnm{Darakjy},~\bfnm{Salima}\binits{S.}},
  \bauthor{\bsnm{Brady},~\bfnm{Joanne~E}\binits{J.~E.}},
  \bauthor{\bsnm{DiMaggio},~\bfnm{Charles~J}\binits{C.~J.}} \AND
  \bauthor{\bsnm{Li},~\bfnm{Guohua}\binits{G.}}
(\byear{2014}).
\btitle{Applying Farr’s Law to project the drug overdose mortality epidemic
  in the United States}.
\bjournal{Injury epidemiology}
\bvolume{1}
\bpages{31}.
\end{barticle}
\endbibitem

\bibitem[\protect\citeauthoryear{Daskalakis, Tzamos and
  Zampetakis}{2017}]{daskalakis2017ten}
\begin{binproceedings}[author]
\bauthor{\bsnm{Daskalakis},~\bfnm{Constantinos}\binits{C.}},
  \bauthor{\bsnm{Tzamos},~\bfnm{Christos}\binits{C.}} \AND
  \bauthor{\bsnm{Zampetakis},~\bfnm{Manolis}\binits{M.}}
(\byear{2017}).
\btitle{Ten steps of EM suffice for mixtures of two Gaussians}.
In \bbooktitle{Conference on Learning Theory}
\bpages{704--710}.
\bpublisher{PMLR}.
\end{binproceedings}
\endbibitem

\bibitem[\protect\citeauthoryear{Ding, Tarokh and Yang}{2018}]{ding2018model}
\begin{barticle}[author]
\bauthor{\bsnm{Ding},~\bfnm{Jie}\binits{J.}},
  \bauthor{\bsnm{Tarokh},~\bfnm{Vahid}\binits{V.}} \AND
  \bauthor{\bsnm{Yang},~\bfnm{Yuhong}\binits{Y.}}
(\byear{2018}).
\btitle{Model selection techniques: An overview}.
\bjournal{IEEE Signal Processing Magazine}
\bvolume{35}
\bpages{16--34}.
\end{barticle}
\endbibitem

\bibitem[\protect\citeauthoryear{Easley et~al.}{2010}]{easley2010networks}
\begin{bbook}[author]
\bauthor{\bsnm{Easley},~\bfnm{David}\binits{D.}},
  \bauthor{\bsnm{Kleinberg},~\bfnm{Jon}\binits{J.}} \betal{et~al.}
(\byear{2010}).
\btitle{Networks, crowds, and markets}
\bvolume{8}.
\bpublisher{Cambridge university press Cambridge}.
\end{bbook}
\endbibitem

\bibitem[\protect\citeauthoryear{Epstein et~al.}{2008}]{epstein2008coupled}
\begin{barticle}[author]
\bauthor{\bsnm{Epstein},~\bfnm{Joshua~M}\binits{J.~M.}},
  \bauthor{\bsnm{Parker},~\bfnm{Jon}\binits{J.}},
  \bauthor{\bsnm{Cummings},~\bfnm{Derek}\binits{D.}} \AND
  \bauthor{\bsnm{Hammond},~\bfnm{Ross~A}\binits{R.~A.}}
(\byear{2008}).
\btitle{Coupled contagion dynamics of fear and disease: mathematical and
  computational explorations}.
\bjournal{PLoS One}
\bvolume{3}
\bpages{e3955}.
\end{barticle}
\endbibitem

\bibitem[\protect\citeauthoryear{Farr}{1840}]{farr1840progress}
\begin{barticle}[author]
\bauthor{\bsnm{Farr},~\bfnm{Wf1840}\binits{W.}}
(\byear{1840}).
\btitle{Progress of epidemics}.
\bjournal{Second report of the Registrar General of England and Wales}
\bpages{16--20}.
\end{barticle}
\endbibitem

\bibitem[\protect\citeauthoryear{Frieden and
  Lee}{2020}]{frieden2020identifying}
\begin{barticle}[author]
\bauthor{\bsnm{Frieden},~\bfnm{Thomas~R}\binits{T.~R.}} \AND
  \bauthor{\bsnm{Lee},~\bfnm{Christopher~T}\binits{C.~T.}}
(\byear{2020}).
\btitle{Identifying and interrupting superspreading events—implications for
  control of severe acute respiratory syndrome coronavirus 2}.
\end{barticle}
\endbibitem

\bibitem[\protect\citeauthoryear{Ganesh, Massouli{\'e} and
  Towsley}{2005}]{ganesh2005effect}
\begin{binproceedings}[author]
\bauthor{\bsnm{Ganesh},~\bfnm{Ayalvadi}\binits{A.}},
  \bauthor{\bsnm{Massouli{\'e}},~\bfnm{Laurent}\binits{L.}} \AND
  \bauthor{\bsnm{Towsley},~\bfnm{Don}\binits{D.}}
(\byear{2005}).
\btitle{The effect of network topology on the spread of epidemics}.
In \bbooktitle{Proceedings IEEE 24th Annual Joint Conference of the IEEE
  Computer and Communications Societies.}
\bvolume{2}
\bpages{1455--1466}.
\bpublisher{IEEE}.
\end{binproceedings}
\endbibitem

\bibitem[\protect\citeauthoryear{Girvan et~al.}{2002}]{girvan2002simple}
\begin{barticle}[author]
\bauthor{\bsnm{Girvan},~\bfnm{Michelle}\binits{M.}},
  \bauthor{\bsnm{Callaway},~\bfnm{Duncan~S}\binits{D.~S.}},
  \bauthor{\bsnm{Newman},~\bfnm{Mark~EJ}\binits{M.~E.}} \AND
  \bauthor{\bsnm{Strogatz},~\bfnm{Steven~H}\binits{S.~H.}}
(\byear{2002}).
\btitle{Simple model of epidemics with pathogen mutation}.
\bjournal{Physical Review E}
\bvolume{65}
\bpages{031915}.
\end{barticle}
\endbibitem

\bibitem[\protect\citeauthoryear{Hespanha
  et~al.}{2021}]{hespanha2021forecasting}
\begin{barticle}[author]
\bauthor{\bsnm{Hespanha},~\bfnm{Jo{\~a}o~P}\binits{J.~P.}},
  \bauthor{\bsnm{Chinchilla},~\bfnm{Raphael}\binits{R.}},
  \bauthor{\bsnm{Costa},~\bfnm{Ramon~R}\binits{R.~R.}},
  \bauthor{\bsnm{Erdal},~\bfnm{Murat~K}\binits{M.~K.}} \AND
  \bauthor{\bsnm{Yang},~\bfnm{Guosong}\binits{G.}}
(\byear{2021}).
\btitle{Forecasting COVID-19 cases based on a parameter-varying stochastic SIR
  model}.
\bjournal{Annual Reviews in Control}.
\end{barticle}
\endbibitem

\bibitem[\protect\citeauthoryear{Hethcote}{2000}]{hethcote2000mathematics}
\begin{barticle}[author]
\bauthor{\bsnm{Hethcote},~\bfnm{Herbert~W}\binits{H.~W.}}
(\byear{2000}).
\btitle{The mathematics of infectious diseases}.
\bjournal{SIAM review}
\bvolume{42}
\bpages{599--653}.
\end{barticle}
\endbibitem

\bibitem[\protect\citeauthoryear{Holmdahl and Buckee}{2020}]{holmdahl2020wrong}
\begin{barticle}[author]
\bauthor{\bsnm{Holmdahl},~\bfnm{Inga}\binits{I.}} \AND
  \bauthor{\bsnm{Buckee},~\bfnm{Caroline}\binits{C.}}
(\byear{2020}).
\btitle{Wrong but useful—what covid-19 epidemiologic models can and cannot
  tell us}.
\bjournal{New England Journal of Medicine}.
\end{barticle}
\endbibitem

\bibitem[\protect\citeauthoryear{Jewell, Lewnard and
  Jewell}{2020}]{jewell2020caution}
\begin{bmisc}[author]
\bauthor{\bsnm{Jewell},~\bfnm{Nicholas~P}\binits{N.~P.}},
  \bauthor{\bsnm{Lewnard},~\bfnm{Joseph~A}\binits{J.~A.}} \AND
  \bauthor{\bsnm{Jewell},~\bfnm{Britta~L}\binits{B.~L.}}
(\byear{2020}).
\btitle{Caution warranted: using the Institute for Health Metrics and
  Evaluation model for predicting the course of the COVID-19 pandemic}.
\end{bmisc}
\endbibitem

\bibitem[\protect\citeauthoryear{Keeling and Eames}{2005}]{keeling2005networks}
\begin{barticle}[author]
\bauthor{\bsnm{Keeling},~\bfnm{Matt~J}\binits{M.~J.}} \AND
  \bauthor{\bsnm{Eames},~\bfnm{Ken~TD}\binits{K.~T.}}
(\byear{2005}).
\btitle{Networks and epidemic models}.
\bjournal{Journal of the Royal Society Interface}
\bvolume{2}
\bpages{295--307}.
\end{barticle}
\endbibitem

\bibitem[\protect\citeauthoryear{Kermack and
  McKendrick}{1927}]{kermack1927contribution}
\begin{barticle}[author]
\bauthor{\bsnm{Kermack},~\bfnm{William~Ogilvy}\binits{W.~O.}} \AND
  \bauthor{\bsnm{McKendrick},~\bfnm{Anderson~G}\binits{A.~G.}}
(\byear{1927}).
\btitle{A contribution to the mathematical theory of epidemics}.
\bjournal{Proceedings of the royal society of london. Series A, Containing
  papers of a mathematical and physical character}
\bvolume{115}
\bpages{700--721}.
\end{barticle}
\endbibitem

\bibitem[\protect\citeauthoryear{Le et~al.}{2020}]{le2020neural}
\begin{barticle}[author]
\bauthor{\bsnm{Le},~\bfnm{Matthew}\binits{M.}},
  \bauthor{\bsnm{Ibrahim},~\bfnm{Mark}\binits{M.}},
  \bauthor{\bsnm{Sagun},~\bfnm{Levent}\binits{L.}},
  \bauthor{\bsnm{Lacroix},~\bfnm{Timothee}\binits{T.}} \AND
  \bauthor{\bsnm{Nickel},~\bfnm{Maximilian}\binits{M.}}
(\byear{2020}).
\btitle{Neural Relational Autoregression for High-Resolution COVID-19
  Forecasting}.
\bjournal{Facebook AI Research}.
\end{barticle}
\endbibitem

\bibitem[\protect\citeauthoryear{Li et~al.}{2020}]{li2020forecasting}
\begin{barticle}[author]
\bauthor{\bsnm{Li},~\bfnm{Michael~L}\binits{M.~L.}},
  \bauthor{\bsnm{Bouardi},~\bfnm{Hamza~Tazi}\binits{H.~T.}},
  \bauthor{\bsnm{Lami},~\bfnm{Omar~Skali}\binits{O.~S.}},
  \bauthor{\bsnm{Trikalinos},~\bfnm{Thomas~A}\binits{T.~A.}},
  \bauthor{\bsnm{Trichakis},~\bfnm{Nikolaos~K}\binits{N.~K.}} \AND
  \bauthor{\bsnm{Bertsimas},~\bfnm{Dimitris}\binits{D.}}
(\byear{2020}).
\btitle{Forecasting Covid-19 and analyzing the effect of government
  interventions}.
\bjournal{medRxiv}.
\end{barticle}
\endbibitem

\bibitem[\protect\citeauthoryear{Ma}{2020}]{ma2020estimating}
\begin{barticle}[author]
\bauthor{\bsnm{Ma},~\bfnm{Junling}\binits{J.}}
(\byear{2020}).
\btitle{Estimating epidemic exponential growth rate and basic reproduction
  number}.
\bjournal{Infectious Disease Modelling}
\bvolume{5}
\bpages{129--141}.
\end{barticle}
\endbibitem

\bibitem[\protect\citeauthoryear{Macal and North}{2009}]{macal2009agent}
\begin{binproceedings}[author]
\bauthor{\bsnm{Macal},~\bfnm{Charles~M}\binits{C.~M.}} \AND
  \bauthor{\bsnm{North},~\bfnm{Michael~J}\binits{M.~J.}}
(\byear{2009}).
\btitle{Agent-based modeling and simulation}.
In \bbooktitle{Proceedings of the 2009 Winter Simulation Conference (WSC)}
\bpages{86--98}.
\bpublisher{IEEE}.
\end{binproceedings}
\endbibitem

\bibitem[\protect\citeauthoryear{Moon}{1996}]{moon1996expectation}
\begin{barticle}[author]
\bauthor{\bsnm{Moon},~\bfnm{Todd~K}\binits{T.~K.}}
(\byear{1996}).
\btitle{The expectation-maximization algorithm}.
\bjournal{IEEE Signal processing magazine}
\bvolume{13}
\bpages{47--60}.
\end{barticle}
\endbibitem

\bibitem[\protect\citeauthoryear{Murray}{2020}]{murray2020forecasting}
\begin{barticle}[author]
\bauthor{\bsnm{Murray},~\bfnm{Christopher~JL}\binits{C.~J.}}
(\byear{2020}).
\btitle{Forecasting COVID-19 impact on hospital bed-days, ICU-days,
  ventilator-days and deaths by US state in the next 4 months}.
\bjournal{medRxiv}.
\bdoi{10.1101/2020.03.27.20043752}
\end{barticle}
\endbibitem

\bibitem[\protect\citeauthoryear{Ray et~al.}{2020}]{ray2020ensemble}
\begin{barticle}[author]
\bauthor{\bsnm{Ray},~\bfnm{Evan~L}\binits{E.~L.}},
  \bauthor{\bsnm{Wattanachit},~\bfnm{Nutcha}\binits{N.}},
  \bauthor{\bsnm{Niemi},~\bfnm{Jarad}\binits{J.}},
  \bauthor{\bsnm{Kanji},~\bfnm{Abdul~Hannan}\binits{A.~H.}},
  \bauthor{\bsnm{House},~\bfnm{Katie}\binits{K.}},
  \bauthor{\bsnm{Cramer},~\bfnm{Estee~Y}\binits{E.~Y.}},
  \bauthor{\bsnm{Bracher},~\bfnm{Johannes}\binits{J.}},
  \bauthor{\bsnm{Zheng},~\bfnm{Andrew}\binits{A.}},
  \bauthor{\bsnm{Yamana},~\bfnm{Teresa~K}\binits{T.~K.}},
  \bauthor{\bsnm{Xiong},~\bfnm{Xinyue}\binits{X.}} \betal{et~al.}
(\byear{2020}).
\btitle{Ensemble forecasts of coronavirus disease 2019 (covid-19) in the us}.
\bjournal{MedRXiv}.
\end{barticle}
\endbibitem

\bibitem[\protect\citeauthoryear{Rockett et~al.}{2020}]{rockett2020revealing}
\begin{barticle}[author]
\bauthor{\bsnm{Rockett},~\bfnm{Rebecca~J}\binits{R.~J.}},
  \bauthor{\bsnm{Arnott},~\bfnm{Alicia}\binits{A.}},
  \bauthor{\bsnm{Lam},~\bfnm{Connie}\binits{C.}},
  \bauthor{\bsnm{Sadsad},~\bfnm{Rosemarie}\binits{R.}},
  \bauthor{\bsnm{Timms},~\bfnm{Verlaine}\binits{V.}},
  \bauthor{\bsnm{Gray},~\bfnm{Karen-Ann}\binits{K.-A.}},
  \bauthor{\bsnm{Eden},~\bfnm{John-Sebastian}\binits{J.-S.}},
  \bauthor{\bsnm{Chang},~\bfnm{Sheryl}\binits{S.}},
  \bauthor{\bsnm{Gall},~\bfnm{Mailie}\binits{M.}},
  \bauthor{\bsnm{Draper},~\bfnm{Jenny}\binits{J.}} \betal{et~al.}
(\byear{2020}).
\btitle{Revealing COVID-19 transmission in Australia by SARS-CoV-2 genome
  sequencing and agent-based modeling}.
\bjournal{Nature medicine}
\bvolume{26}
\bpages{1398--1404}.
\end{barticle}
\endbibitem

\bibitem[\protect\citeauthoryear{Ruhi and Hassibi}{2015}]{ruhi2015sirs}
\begin{binproceedings}[author]
\bauthor{\bsnm{Ruhi},~\bfnm{Navid~Azizan}\binits{N.~A.}} \AND
  \bauthor{\bsnm{Hassibi},~\bfnm{Babak}\binits{B.}}
(\byear{2015}).
\btitle{SIRS epidemics on complex networks: Concurrence of exact Markov chain
  and approximated models}.
In \bbooktitle{2015 54th IEEE Conference on Decision and Control (CDC)}
\bpages{2919--2926}.
\bpublisher{IEEE}.
\end{binproceedings}
\endbibitem

\bibitem[\protect\citeauthoryear{Santillana
  et~al.}{2018}]{santillana2018relatedness}
\begin{barticle}[author]
\bauthor{\bsnm{Santillana},~\bfnm{Mauricio}\binits{M.}},
  \bauthor{\bsnm{Tuite},~\bfnm{Ashleigh}\binits{A.}},
  \bauthor{\bsnm{Nasserie},~\bfnm{Tahmina}\binits{T.}},
  \bauthor{\bsnm{Fine},~\bfnm{Paul}\binits{P.}},
  \bauthor{\bsnm{Champredon},~\bfnm{David}\binits{D.}},
  \bauthor{\bsnm{Chindelevitch},~\bfnm{Leonid}\binits{L.}},
  \bauthor{\bsnm{Dushoff},~\bfnm{Jonathan}\binits{J.}} \AND
  \bauthor{\bsnm{Fisman},~\bfnm{David}\binits{D.}}
(\byear{2018}).
\btitle{Relatedness of the incidence decay with exponential adjustment (IDEA)
  model,“Farr's law” and SIR compartmental difference equation models}.
\bjournal{Infectious disease modelling}
\bvolume{3}
\bpages{1--12}.
\end{barticle}
\endbibitem

\bibitem[\protect\citeauthoryear{Schwarz}{1978}]{schwarz1978estimating}
\begin{barticle}[author]
\bauthor{\bsnm{Schwarz},~\bfnm{Gideon}\binits{G.}}
(\byear{1978}).
\btitle{Estimating the dimension of a model}.
\bjournal{The annals of statistics}
\bpages{461--464}.
\end{barticle}
\endbibitem

\bibitem[\protect\citeauthoryear{Shahid, Zameer and
  Muneeb}{2020}]{shahid2020predictions}
\begin{barticle}[author]
\bauthor{\bsnm{Shahid},~\bfnm{Farah}\binits{F.}},
  \bauthor{\bsnm{Zameer},~\bfnm{Aneela}\binits{A.}} \AND
  \bauthor{\bsnm{Muneeb},~\bfnm{Muhammad}\binits{M.}}
(\byear{2020}).
\btitle{Predictions for COVID-19 with deep learning models of LSTM, GRU and
  Bi-LSTM}.
\bjournal{Chaos, Solitons \& Fractals}
\bvolume{140}
\bpages{110212}.
\end{barticle}
\endbibitem

\bibitem[\protect\citeauthoryear{Virtanen et~al.}{2020}]{virtanen2020scipy}
\begin{barticle}[author]
\bauthor{\bsnm{Virtanen},~\bfnm{Pauli}\binits{P.}},
  \bauthor{\bsnm{Gommers},~\bfnm{Ralf}\binits{R.}},
  \bauthor{\bsnm{Oliphant},~\bfnm{Travis~E.}\binits{T.~E.}},
  \bauthor{\bsnm{Haberland},~\bfnm{Matt}\binits{M.}},
  \bauthor{\bsnm{Reddy},~\bfnm{Tyler}\binits{T.}},
  \bauthor{\bsnm{Cournapeau},~\bfnm{David}\binits{D.}},
  \bauthor{\bsnm{Burovski},~\bfnm{Evgeni}\binits{E.}},
  \bauthor{\bsnm{Peterson},~\bfnm{Pearu}\binits{P.}},
  \bauthor{\bsnm{Weckesser},~\bfnm{Warren}\binits{W.}},
  \bauthor{\bsnm{Bright},~\bfnm{Jonathan}\binits{J.}}, \bauthor{\bsnm{{van der
  Walt}},~\bfnm{St{\'e}fan~J.}\binits{S.~J.}},
  \bauthor{\bsnm{Brett},~\bfnm{Matthew}\binits{M.}},
  \bauthor{\bsnm{Wilson},~\bfnm{Joshua}\binits{J.}},
  \bauthor{\bsnm{Millman},~\bfnm{K.~Jarrod}\binits{K.~J.}},
  \bauthor{\bsnm{Mayorov},~\bfnm{Nikolay}\binits{N.}},
  \bauthor{\bsnm{Nelson},~\bfnm{Andrew R.~J.}\binits{A.~R.~J.}},
  \bauthor{\bsnm{Jones},~\bfnm{Eric}\binits{E.}},
  \bauthor{\bsnm{Kern},~\bfnm{Robert}\binits{R.}},
  \bauthor{\bsnm{Larson},~\bfnm{Eric}\binits{E.}},
  \bauthor{\bsnm{Carey},~\bfnm{C~J}\binits{C.~J.}},
  \bauthor{\bsnm{Polat},~\bfnm{{\. I}lhan}\binits{{\. I}.}},
  \bauthor{\bsnm{Feng},~\bfnm{Yu}\binits{Y.}},
  \bauthor{\bsnm{Moore},~\bfnm{Eric~W.}\binits{E.~W.}},
  \bauthor{\bsnm{{VanderPlas}},~\bfnm{Jake}\binits{J.}},
  \bauthor{\bsnm{Laxalde},~\bfnm{Denis}\binits{D.}},
  \bauthor{\bsnm{Perktold},~\bfnm{Josef}\binits{J.}},
  \bauthor{\bsnm{Cimrman},~\bfnm{Robert}\binits{R.}},
  \bauthor{\bsnm{Henriksen},~\bfnm{Ian}\binits{I.}},
  \bauthor{\bsnm{Quintero},~\bfnm{E.~A.}\binits{E.~A.}},
  \bauthor{\bsnm{Harris},~\bfnm{Charles~R.}\binits{C.~R.}},
  \bauthor{\bsnm{Archibald},~\bfnm{Anne~M.}\binits{A.~M.}},
  \bauthor{\bsnm{Ribeiro},~\bfnm{Ant{\^o}nio~H.}\binits{A.~H.}},
  \bauthor{\bsnm{Pedregosa},~\bfnm{Fabian}\binits{F.}}, \bauthor{\bsnm{{van
  Mulbregt}},~\bfnm{Paul}\binits{P.}} \AND \bauthor{\bsnm{{SciPy 1. 0
  Contributors}}}
(\byear{2020}).
\btitle{{{SciPy} 1.0: Fundamental Algorithms for Scientific Computing in
  Python}}.
\bjournal{Nature Methods}
\bvolume{17}
\bpages{261--272}.
\bdoi{10.1038/s41592-019-0686-2}
\end{barticle}
\endbibitem

\bibitem[\protect\citeauthoryear{Xue}{2017}]{xue2017law}
\begin{barticle}[author]
\bauthor{\bsnm{Xue},~\bfnm{Xiaofeng}\binits{X.}}
(\byear{2017}).
\btitle{Law of large numbers for the SIR model with random vertex weights on
  Erd{\H{o}}s--R{\'e}nyi graph}.
\bjournal{Physica A: Statistical Mechanics and its Applications}
\bvolume{486}
\bpages{434--445}.
\end{barticle}
\endbibitem

\bibitem[\protect\citeauthoryear{Zhang and Zhou}{2020}]{zhang2020non}
\begin{barticle}[author]
\bauthor{\bsnm{Zhang},~\bfnm{Anru~R}\binits{A.~R.}} \AND
  \bauthor{\bsnm{Zhou},~\bfnm{Yuchen}\binits{Y.}}
(\byear{2020}).
\btitle{On the non-asymptotic and sharp lower tail bounds of random variables}.
\bjournal{Stat}
\bvolume{9}
\bpages{e314}.
\end{barticle}
\endbibitem

\end{thebibliography}

\end{document}